\pdfoutput=1
\documentclass[11pt]{article}
\usepackage{fullpage}
\usepackage[margin=1in]{geometry}
\usepackage{epsf}
\usepackage{lastpage,algorithm,etoolbox}
\usepackage{fancyheadings,hyperref}
\usepackage{graphics}
\usepackage{graphicx}
\usepackage{psfrag}

\usepackage{amsmath, amsfonts, amsthm, amssymb, algorithm, graphicx}
\usepackage{mathtools}
\usepackage{comment}
\usepackage{enumitem}
\usepackage{pdfpages}
\usepackage[noend]{algpseudocode}
\usepackage{fancyhdr}
\usepackage{listings}
\usepackage{mathtools}
\newcommand{\C}{\mathbb{C}}
\newtoggle{acm}
\togglefalse{acm}


\DeclareMathOperator{\BigOm}{\mathcal{O}}
\newcommand{\BigOh}[1]{\BigOm\left({#1}\right)}

\DeclareMathOperator{\BigWm}{\Omega}
\newcommand{\BigOmega}[1]{\BigWm\left({#1}\right)}

\algnewcommand\algorithmicinput{\textbf{INPUT:}}
\algnewcommand\INPUT{\item[\algorithmicinput]}
\algnewcommand\algorithmicoutput{\textbf{OUTPUT:}}
\algnewcommand\OUTPUT{\item[\algorithmicoutput]}
 
\newcommand{\calO}{\mathcal{O}}

\newcommand{\calV}{\mathcal{V}}

\newcommand{\calM}{\mathcal{M}}
\newcommand{\calD}{\mathcal{D}}

\newcommand{\calE}{\mathcal{E}}
\newcommand{\calF}{\mathcal{F}}
\newcommand{\fraka}{\mathfrak{a}}

\newcommand{\frakq}{\mathfrak{q}}

\newcommand{\frakA}{\mathfrak{A}}

\newcommand{\itSigma}{\mathsf{\Sigma}}
\newcommand{\itP}{\mathsf{P}}
\newcommand{\itT}{\mathsf{T}}
\newcommand{\acc}{\mathsf{\epsilon}}

\newcommand{\itO}{\mathsf{O}}
\newcommand{\itv}{\mathsf{v}}

\newcommand{\itV}{\mathsf{V}}
\newcommand{\itVbar}{\overline{\mathsf{V}}}
\newcommand{\itOtil}{\widetilde{\mathsf{O}}}
\newcommand{\itVtil}{\widetilde{\mathsf{V}}}

\newcommand{\gapexp}{\frac{1 + \sqrt{\boldgap(2-\boldgap)}}{1-\boldgap}}
\newcommand{\matZ}{\mathbf{Z}}
\newcommand{\matz}{\mathbf{z}}
\newcommand{\matx}{\mathbf{x}}
\newcommand{\matLambda}{\mathbf{\Lambda}}

\newcommand{\matU}{\mathbf{U}}
\newcommand{\matu}{\mathbf{u}}
\newcommand{\matW}{\mathbf{W}}
\newcommand{\matM}{\mathbf{M}}
\newcommand{\matMabs}{\mathrm{abs}(\matM)}
\newcommand{\matWabs}{\mathrm{abs}(\matW)}

\newcommand{\matX}{\mathbf{X}}
\newcommand{\matWtil}{\widetilde{\mathbf{W}}}
\newcommand{\matO}{\mathbf{O}}

\newcommand{\calY}{\mathcal{Y}}
\newcommand{\calS}{\mathcal{S}}

\newcommand{\calA}{\mathcal{A}}
\newcommand{\calB}{\mathcal{B}}
\newcommand{\calI}{\mathcal{I}}

\newcommand{\calW}{\mathcal{W}}

\newcommand{\calP}{\mathcal{P}}

\newcommand{\calX}{\mathcal{X}}
\newcommand{\calG}{\mathcal{G}}

\newcommand{\Lip}{\mathrm{Lip}}

\newcommand{\boldeps}{\overline{\epsilon}}
\newcommand{\boldDelta}{\mathtt{\Delta}}
\newcommand{\boldlam}{\mathtt{\lambda}}

\newcommand{\Stief}{\mathrm{Stief}}

\newcommand{\calN}{\mathcal{N}}
\newcommand{\gap}{\mathrm{gap}}
\newcommand{\boldgap}{\mathtt{gap}}

\newcommand{\gaptil}{\widetilde{\mathrm{gap}}}

\DeclareMathAlphabet{\mathbfsf}{\encodingdefault}{\sfdefault}{bx}{n}
\newcommand{\Prit}{\mathbfsf{P}}
\newcommand{\Expit}{\mathsf{E}}

\newcommand{\Qit}{\mathsf{\mathbf{Q}}}

\newcommand{\KL}{\mathrm{KL}}
\newcommand{\Sym}{\mathbb{S}}
\newcommand{\Symd}{\mathbb{S}^{d}}
\newcommand{\Alg}{\mathsf{Alg}}

\newcommand{\rmd}{\mathrm{d}}
\newcommand{\tr}{\mathrm{tr}}

\newcommand{\median}{\mathrm{Median}}
\newcommand{\op}{\mathrm{op}}
\newcommand{\F}{\mathrm{F}}

\newcommand{\spec}{\mathrm{spec}}

\newcommand{\diag}{\mathrm{diag}}

\newcommand{\sphere}{\calS^{d-1}}

\newcommand{\calVhead}{\mathcal{V}_{\mathrm{head}}}
\newcommand{\calVtail}{\mathcal{V}_{\mathrm{tail}}}

\newcommand{\GOE}{\mathrm{GOE}}
\newcommand{\SG}{\mathrm{StdG}}
\newcommand{\boldtheta}{\boldsymbol{\theta}}

\newcommand{\vtil}{\widetilde{\mathsf{v}}}
\newcommand{\wtil}{\widetilde{\mathsf{v}}}
\newcommand{\lamtil}{\widetilde{\lambda}}
\newcommand{\lambdainv}{\boldlam^{-1}}

\newcommand{\itZ}{\mathsf{Z}}

\newcommand{\vj}{\mathsf{v}^{(j)}}
\newcommand{\vi}{\mathsf{v}^{(i)}}
\newcommand{\vk}{\mathsf{v}^{(k)}}
\newcommand{\vkT}{\mathsf{v}^{(k)\top}}
\newcommand{\vone}{\mathsf{v}^{(1)}}
\newcommand{\vkplus}{\mathsf{v}^{(k+1)}}

\newcommand{\unifsim}{\overset{\mathrm{unif}}{\sim}}

\newcommand{\wk}{\mathsf{w}^{(k)}}

\newcommand{\wone}{\mathsf{w}^{(1)}}
\newcommand{\wiminus}{\mathsf{w}^{(i-1)}}
\newcommand{\viminus}{\mathsf{w}^{(i-1)}}
\newcommand{\wi}{\mathsf{w}^{(i)}}

\newcommand{\Vhat}{\widehat{\mathsf{V}}}
\newcommand{\vhat}{\widehat{\mathsf{v}}}
\newcommand{\vtplus}{{\mathsf{v}^{(\itT+1)}}}
\newcommand{\vt}{{\mathsf{v}^{(\itT)}}}

\newcommand{\vtplusr}{{\mathsf{v}^{(\itT+r)}}}

\newcommand{\vitil}{\vtil^{(i)}}
\newcommand{\vktil}{\vtil^{(k)}}

\newcommand{\wjtil}{\wtil^{(j)}}

\newcommand{\vjplustil}{\vtil^{(j+1)}}
\newcommand{\viT}{\mathsf{v}^{(i)\top}}
\newcommand{\vtwo}{\mathsf{v}^{(2)}}

\newcommand{\wonetil}{\wtil^{(1)}}
\newcommand{\wiminustil}{\wtil^{(i-1)}}
\newcommand{\viminustil}{\vtil^{(i-1)}}
\newcommand{\witil}{\wtil^{(i)}}
\newcommand{\vonetil}{\vtil^{(1)}}

\newcommand{\opt}{\mathrm{opt}}

\newcommand{\Law}{\mathrm{Law}}
\newcommand{\Leb}{\mathrm{Lebesgue}}
\newcommand{\stielt}{\mathfrak{s}}

\newcommand{\alow}{a_{\mathrm{low}}}
\newcommand{\aup}{a_{\mathrm{up}}}
\newcommand{\Eup}{\calE_{\mathrm{up}}}
\newcommand{\Egood}{\calE_{\mathrm{good}}}

\newcommand{\Elow}{\calE_{\mathrm{low}}}

\renewcommand{\implies}{\text{ implies }}
\renewcommand{\iff}{\text{ iff }}

\newcommand{\err}{\mathrm{Err}}
\renewcommand{\Im}{\mathfrak{Im }}
\renewcommand{\Re}{\mathfrak{Re }}
\newcommand{\fraki}{\mathfrak{i}}

\newcommand{\im}{\mathrm{im }}

\renewcommand{\C}{\mathbb{C}}
\newcommand{\R}{\mathbb{R}}
\newcommand{\I}{\mathbb{I}}
\newcommand{\Exp}{\mathbb{E}}
\newcommand{\Q}{\mathbb{Q}}
\newcommand{\sign}{\mathrm{sign\ }}

\newcommand{\Var}{\mathrm{Var}}

\renewcommand{\Pr}{\mathbb{P}}

{
 \theoremstyle{plain}
      
}
\newtheorem{nono-theorem}{Theorem}[]
\theoremstyle{plain}
\newtheorem{thm}{Theorem}[section]
\newtheorem{claim}[thm]{Claim}
\newtheorem{lem}[thm]{Lemma}
\newtheorem{cor}[thm]{Corollary}

\DeclarePairedDelimiter\ceil{\lceil}{\rceil}

\newtheorem{prop}[thm]{Proposition}

\theoremstyle{definition}
\newtheorem{defn}{Definition}[section]

\newtheorem{exmp}{Example}[section]

\newtheorem{obs}{Observation}[section]
\title{\Large \bf Tight Query Complexity Lower Bounds for PCA via Finite Sample Deformed Wigner Law}

\author{ Max Simchowitz\thanks{UC Berkeley, CA. msimchow@berkeley.edu. } \and
Ahmed El Alaoui\thanks{UC Berkeley, CA. elalaoui@berkeley.edu. }
\and
Benjamin Recht\thanks{UC Berkeley, CA.  brecht@berkeley.edu.} }

\date{}
\begin{document}
\begin{titlepage}
\maketitle{}
\begin{abstract}
   %!TEX root = main.tex

We prove a \emph{query complexity} lower bound for approximating the top $r$ dimensional eigenspace of a matrix. We consider an oracle model where, given a symmetric matrix $\mathbf{M} \in \mathbb{R}^{d \times d}$, an algorithm $\mathsf{Alg}$ is allowed to make $\mathsf{T}$ exact queries of the form $\mathsf{w}^{(i)} = \mathbf{M} \mathsf{v}^{(i)}$ for $i$ in $\{1,...,\mathsf{T}\}$, where $\mathsf{v}^{(i)}$ is drawn from a distribution which depends arbitrarily on the past queries and measurements $\{\mathsf{v}^{(j)},\mathsf{w}^{(i)}\}_{1 \le j \le i-1}$. We show that for every $\mathtt{gap} \in (0,1/2]$, there exists a distribution over matrices $\mathbf{M}$ for which 1) $\mathrm{gap}_r(\mathbf{M}) = \Omega(\mathtt{gap})$ (where $\mathrm{gap}_r(\mathbf{M})$ is the normalized gap between the $r$ and $r+1$-st largest-magnitude eigenvector of $\mathbf{M}$), and 2) any algorithm $\mathsf{Alg}$ which takes fewer than $\mathrm{const} \times \frac{r \log d}{\sqrt{\mathtt{gap}}}$ queries fails (with overwhelming probability) to identity a matrix $\widehat{\mathsf{V}} \in \mathbb{R}^{d \times r}$ with orthonormal columns for which $\langle \widehat{\mathsf{V}}, \mathbf{M} \widehat{\mathsf{V}}\rangle \ge (1 - \mathrm{const}  \times \mathtt{gap})\sum_{i=1}^r \lambda_i(\mathbf{M})$. Our bound requires only that $d$ is a small polynomial in $1/\mathtt{gap}$ and $r$, and matches the upper bounds of Musco and Musco '15. Moreover, it establishes a strict separation between convex optimization and \emph{randomized}, ``strict-saddle'' non-convex optimization of which PCA is a canonical example: in the former, first-order methods can have dimension-free iteration complexity, whereas in PCA, the iteration complexity of gradient-based methods must necessarily grow with the dimension.
\vspace{.1cm}

Our argument proceeds via a reduction to estimating a rank-$r$ spike in a deformed Wigner model $\mathbf{M} =\mathbf{W} + \mathtt{\lambda} \mathbf{U} \mathbf{U}^\top$, where $\mathbf{W}$ is from the Gaussian Orthogonal Ensemble, $\mathbf{U}$ is uniform on the $d \times r$-Stieffel manifold and $\mathtt{\lambda} > 1$ governs the size of the perturbation. Surprisingly, this ubiquitous random matrix model witnesses the worst-case rate for eigenspace approximation, and the `accelerated' inverse square-root dependence on the gap in the rate follows as a consequence of the correspendence between the asymptotic eigengap and the size of the perturbation $\mathtt{\lambda}$, when $\mathtt{\lambda}$ is near the ``phase transition'' $\mathtt{\lambda} = 1$. To verify that $d$ need only be polynomial in $\mathtt{gap}^{-1}$ and $r$, we prove a finite sample convergence theorem for top eigenvalues of a deformed Wigner matrix, which may be of independent interest. We then lower bound the above estimation problem with a novel technique based on Fano-style data-processing inequalities with truncated likelihoods; the technique generalizes the Bayes-risk lower bound of Chen et al.\ '16, and we believe it is particularly suited to lower bounds in adaptive settings like the one considered in this paper.

\end{abstract}
\end{titlepage}

%!TEX root = main.tex
\section{Introduction}
Eigenvector approximation is widely regarded as a fundamental problem in machine learning~\cite{jolliffe2002principal}, numerical linear algebra~\cite{demmel1997applied}, optimization, and numerous graph-related learning problems~\cite{spielman2007spectral,page1999pagerank,ng2001spectral}. Interest in PCA has been driven further by the rush to understand non-convex optimization, as PCA has become the cannonical example of a benign, but not-quite-convex objective. For one, there is a striking resemblence between eigenvector approximation algorithms and first-order convex optimization procedures~\cite{shamir2015fast,garber2016faster,allen2016first}. Moreover, PCA is one of the simplest  `strict saddle' objectives: that is, a function whose first-order stationary points are either local minima, or saddle points at which the Hessian has a strictly negative eigenvalue. The strict saddle property extends to many popular nonconvex objectives, and enables efficient optimization by first order algorithms. Notably, Jin et al.~\cite{jin2017escape} proposed a gradient algorithm which finds an approximate local minimum of a strict saddle objective in a number of iterations which matches first-order methods for comparable convex problems, up to poly-logarithmic factors in the dimension. 

The aim of this paper is to understand the fundamental limits of \emph{randomized} first-order methods for such benign non-convex problems by establishing sharp query-complexity lower bounds for approximating the top eigenspace of a symmetric matrix. Specifically, we consider randomized, adaptive algorithms $\Alg$ which access an unknown symmetric matrix $\matM \in \R^{d\times d}$ via $\itT$ queries of the form $\{\vi = \matM\wi\}_{i \in [\itT]}$. Letting $\gap_r(\matM)$ denotes the (normalized) eigengap between the $r$- and $r+1$-st singular value (or eigenvalue-magnitude) of $\matM$, let $\matMabs := (\matM^{2})^{1/2}$, we prove the following:
\begin{nono-theorem}[Main Theorem]\label{mainthm} There are universal constants $c_1,c_2 > 0$ such that for every $r \ge 1$ and $\boldgap \in (0,1/2]$, then there exist $d_0 = \mathrm{poly}(\frac{1}{\boldgap},r)$ such that, for all $d \ge d_0$, there exists a distribution over symmetric matrices $\matM \in \R^{d \times d}$ for which $\gap_r(\matM) \ge \frac{\boldgap}{3}$, and if $\Alg$ makes $\itT \le \frac{c_1r \log d}{\sqrt{\boldgap}}$ queries, then with probability at least $1 - \exp( - d^{c_2})$, $\Alg$ cannot identify a matrix $\Vhat \in \R^{d \times r}$ with orthonormal columns for which $\langle \Vhat, \matMabs \Vhat \rangle \ge \left(1 - \frac{\boldgap}{12}\right)\sum_{i=1}^r \lambda_i(\matMabs)$.\footnote{ Note that $\matM$ and $\matMabs$ have the same singular values. We choose  to consider $\matMabs$ since $\matM$ will not be PSD in our construction, and we do not wish to penalize the learner for the negative eigendirections of $\matM$. Moreover, the ``hard distribution'' over $\matM$ in our lower bounds will always be conditioned on the event that $\lambda_{\ell}(\matM) = \sigma_{\ell}(\matM) = \lambda_{\ell}(\matMabs)$. If the reader's taste prefers, one can readily establish qualitatively similar results in terms of $\langle \vhat, \matM^2 \vhat\rangle$.}
\end{nono-theorem}
We emphasize that our lower bounds are information-theoretic, and do not place any computational or Krylov restrictions on how $\Alg$ generates its queries. Our bounds are tight, and are matched by the Block-Lanczos algorithm~\cite{musco2015randomized}. Finally, the presence of a logarithmic factor in the dimension establishes a strict separation between truly-convex and strict saddle objectives: whereas convex objectives admit first order algorithms whose query complexity is independent of the ambient dimension, strict saddle-objectives necessarily incur dimension-dependent terms, even for \emph{randomized} algorithms. 

%we note that our techniques can be applied to proving related lower bounds for a) the number of queries required to test between between $\matM = \matW$ and $\matM = \matW + \lambda \matU \matU^\top$ b) the number of rounds of adaptivity required if $\Alg$ can make $k$-batch queries; c) lower bound of $T \gtrsim -\log( 1- \boldgap) \cdot r \log d$ in the setting that $\boldgap$ approaches $1$. However, the proof of Theorem~\ref{mainthm} is long and quite intricate, and in the interest of brevity, we defer the proof of these ancillary results to a future publication. 

%!TEX root = main.tex
\paragraph{New Techniques}
Our lower bound proceeds by a reduction from eigenspace computation to estimating a planted rank-$r$ component in a deformed Wigner model $\matM = \matW + \boldlam \matU\matU^\top$, where $\matW$ is from the Gaussian Orthogonal Ensemble (GOE), $\matU$ is uniform on the $d \times r$ Stieffel manifold, and $\boldlam = \gapexp$ is a parameter ensuring that $\gap_r(\matM)$ concentrates around $\boldgap$. Note that as $\boldgap \to 0$, $\boldlam \to 1$ placing us near the ``phase transition'' $\boldlam = 1$~\cite{feral2007largest}.  To ensure that we can take $d = \mathrm{poly}(\frac{1}{\boldgap},r)$, we prove the first (to our knowledge) finite-sample convergence result for the top $r$ eigenvalues of a deformed Wigner matrix in the regime where $d$ is polynomial $r$ and $\boldgap^{-1}$. Along the way, we prove a variant of the Hanson-Wright inequality for the Stieffel manfiold, and a pointwise convergence result for the Stieltjes transform; these results are outlined in Section~\ref{sec:Wig_Spec}.

After formalizing the reduction, our proof hinges on showing that when $r = 1$ and $\matU = \matu \in \R^d$, then our ``information'' about $\matu$, quantified by the squared-norm of the projection of $\matu$ onto the span of the first $k$ query vectors, can grow at a rate of at most $\boldlam^{\BigOh{1}} = 1 + \BigOh{\boldgap^{1/2}}$ per round. We generalize to the rank-$r$ case by leveraging the information-theoretic arguments from the rank-one case, but with a far more careful recursion to obtain the right dependence on $r$ (see Section~\ref{sec:body_rk_r_thm} for details). For $r=1$, our basic strategy mirrors Price and Woodruff's \cite{price2013lower} sparse recovery lower bound, which sequentially controls the mutual information bewteen measurements of a sparse vector and a planted solution. However, since $\boldlam = 1 + \BigOh{\boldgap^{1/2}}$, we require novel techniques in order to not overshoot the slow growth rate of $\lambda^{\BigOh{1}}$ per round. Specifically, at each round $k \in [\itT]$, we apply a generalization of Fano's inequality which replaces the $\KL$-divergence with the expected $1+\eta$-powers of appropriate likelihood ratios. The inequality is based on the Bayes risk lower bounds of Chen et al.\ \cite{chen2016bayes}, who generalize Fano's inequality to arbitrary $f$-divergences, and show that their $\chi^2$-variant of Fano's inequality (i.e. $\eta = 1$) yields sharper lower bounds in many non-adaptive problems. In our case, we tune $\eta$ as a function of $\boldlam$ to get the correct rate.

Unfortunately, we cannot apply the bounds from Chen et al.\ \cite{chen2016bayes} out of the box. This is because if there is even a small probability that $\Alg$ takes one highly informative measurement, then the expected likelihood ratios will overestimate the average information gain. This is an artifact of the the fact that tails of likelihood ratios (unlike log-likelihoods) are ill-behaved. But this only becomes a problem in adaptive settings where measurements grow more informative over time. We circumvent this by proving a ``truncated'' variant of the bound in Chen et al.\ \cite{chen2016bayes}, which replaces the expected likelihood moments with an expectation restricted to the ``good event'' where $\Alg$ has yet to take an improbably-informative measurement. We prove this bound by generalizing $f$-divergences to arbitrary finite, non-normalized measures (e.g., measures obtained by restricting probability distributions to a given event), and establish that the data-processing inequality still holds in this general setting. Our information-theoretic tools are explained at length in Section~\ref{sec:info_th_rkone}. %
%!TEX root = main.tex

\paragraph{Related Work} It is hard to do justice to the vast body of work on eigenvector computation, matrix approximation, and first order methods for convex and strict saddle objectives. We shall instead focus on situating our work in the lower bounds literature. As described above, our proof casts eigenvector computation as a sequential estimation problem. These have been studied at length in the context of sparse recovery and active adaptive compressed sensing~\cite{arias2013fundamental,price2013lower,castro2017adaptive,castro2014adaptive}. Due to the noiseless oracle model, our setting is most similar to that of Price and Woodruff~\cite{price2013lower}, whereas other works~\cite{arias2013fundamental,castro2017adaptive,castro2014adaptive} study measurements contaminated with noise. 
%Our setting also exhibits similarities to the stochastic linear bandit problem~\cite{soare2014best}. 
More broadly, query complexity has received much recent attention in the context of communication-complexity~\cite{anshu2017lifting,nelson2017optimal}, in which lower bounds on query complexity imply corresponding bounds against communication via lifting theorems. Similar ideas also arise the study of learning under memory constraints~\cite{steinhardt2015memory,steinhardt2015minimax,shamir2014fundamental}. 

From an optimization perspective, our lower bound can be cast as a non-convex analogue of the seminal convex-optimization oracle lower bounds of Nemirovskii and Yudin~\cite{nemirovskii1983problem}. But whereas the latter  bounds match known upper bounds in terms of dependence on relevant parameters (accuracy, condition number, Lipschitz constant), Nemirovskii and Yudin consider worst-case initializations, and impose a strong Krylov space assumption. In the context of finite sums, Agarwal et al.\ \cite{agarwal2014lower} show that the Krylov space assumption can be removed, and Woodworth et al.\ \cite{woodworth2016tight} prove truly information-theoretic lower bounds by considering randomized algorithms as we do in this work (albeit with different techniques). Lower bounds have also been established in the stochastic convex optimization~\cite{agarwal2009information,jamieson2012query} where each gradient- or function-value oracle query is corrupted with i.i.d.\ noise, and Allen-Zhu et al.\ \cite{allen2016first} prove analogues of these bounds for streaming PCA. While these lower bounds are information-theoretic, and thus unconditional, they are incomparable to  the setting considered in this work, where we are allowed to make exact, noiseless queries. 
%\maxs{Ahmed - write a sentence or two about related random matrix theory}

 %Moving beyond the classical power method and Lanczos algorithm~\cite{demmel1997applied}, numerous works have studied the eigenvector computation  and in the case where the matrix $M$ can be factored as the outer product of matrices $AA^{\top}$, here $A$ only few non-zero entries~\cite{garber2016faster}. 

%!TEX root = main.tex

\section{Statement of Main Results}

Let $\|\cdot\|$ denote the $\ell_2$ norm on $\R^{d}$, and let $\calS^{d-1} := \{x\in \R^{d}:\|x\| = 1\}$ denote the unit sphere and $\Stief(d,r)$ denote the Stieffel manifold consisting of matrices $V \in \R^{d \times r}$ such that $V^\top V = I$.  Let $\Sym^{d\times d}$ denote the set of symmetric $d\times d$ matrices, and for $M \in \Sym^{d\times d}$, we let $\lambda_1(M) \ge \lambda_2(M) \ge \dots \ge \lambda_d(M)$ denote its eigenvalues in decreasing order, $v_1(M),v_2(M), \dots, v_{d}(M)$ denote the corresponding eigenvectors, let $\|M\|_\op$ and $\|M\|_F$ denote the operator and Frobenious norms, and $\mathrm{abs}(M) := (M^2)^{1/2}$. Finally, we define the eigengap of $M \in \Sym^{d\times d}$ as $\mathrm{gap}_r(M) := \frac{\sigma_r(M) - \sigma_{r+1}(M)}{\sigma_{r}(M)}$, where $\sigma_i(M) = \lambda_i(M^2)^{1/2}$  is the  $i$-th singular value of $M$. We will also use the notation $\gap(M) := \gap_1(M)$.  We now introduce a definition of our query model:
\begin{defn}[Query Model]\label{def:Query_def} An \emph{randomized adaptive query algorithm} $\Alg$ with \emph{query complexity} $\itT \in \mathbb{N}$ and accuracy $\acc \in (0,1)$ is an algorithm which, for rounds $i \in [\itT]$, queries an oracle with a vector $\vi$, and receives a noiseless response  $\wi = \matM \vi$. At the end $\itT$ rounds, the algorithm returns a matrix $\Vhat \in \Stief(d,r)$. The queries $\vi$ and output $\Vhat$ are allowed to be randomized and adaptive, in that $\vi$ is a function of $\{(\vone,\wone),\dots,(\viminus,\wiminus)\}$, as well as some random seed. %In the rank one case, our goal is to ensure  that $\langle \Vhat, \matMabs\Vhat \rangle  \ge (1 -  \acc \gap(\matM))\sum_{\ell=1}^r \sigma_{\ell}(\matM)$.

 %We say that $\Alg$ is \emph{deterministic} if, for all $i \in [T+1]$, $v^{(i)}$ is a deterministic function of $\{(v^{(1)},w^{(1)}),\dots,(v^{(i-1)},w^{(i-1)})\}$. We say that $\Alg$ is \emph{non-adaptive} if, for all $i \in [T]$ the distribution of $v^{(i)}$ is independent of the observations $\{(w^{(1)}),\dots,w^{(i-1)})\}$, (but $\widehat{v}$ may dependent on past observations.) 
\end{defn}
The goal of $\Alg$ is to return a $\Vhat$ satisfying
\begin{eqnarray*}
\langle \Vhat, \matMabs\Vhat \rangle  \ge (1 -  \acc \gap(\matM))\sum_{\ell=1}^r \sigma_{\ell}(\matM)
\end{eqnarray*}
for some small $\acc \in (0,1)$. In the rank-one case, $\Vhat$ is a vector $\vhat \in \sphere$, and the above condition reduces to  $\langle \vhat, \matMabs \vhat \rangle  \ge (1 - \acc \gap(\matM)))\|\matM\|_{\op}$. 
\begin{exmp}[Examples of Randomized Query Algorithms]
In the rank-one case, the Power Method and Lanczos algorithms~\cite{demmel1997applied} are both randomized, adaptive query methods. Even though the iterates $\vi$ of the Lanczos and power methods converge to the top eigenvector at different rates, they make identical queries: namely, they both identify $\matM$ on the Krylov space spanned by $\vone,\matM v^{(1)},\dots, \matM^{T-1}\vone$. Lanczos differs from the Power Method by choosing $\vhat$ to be the optimal vector in this Krylov space, rather than the last iterate. Observe that even in the rank-$r$ case, our query model still permits each single vector-query to be chosen adaptively. Hence, our lower bound applies to subspace iterations (e.g., the block Krylov method of Musco and Musco~\cite{musco2015randomized}), and to algorithms which use deflation~\cite{allen2016lazysvd}).
\end{exmp}%The only difference is that the Lanczos algorithm selects $\vhat$ in a more intelligent manner than the power method. Running the power method from a deterministic initialization would be a \emph{non-randomized} algorithm. Any non-randomized algorithm, even an adaptive one, must take $d$ queries in the worse case, since $d-1$ queries can only identify a matrix up to a $d-1$ dimensional subspace. Randomized, but non-adaptive algorithms need to take $\Omega(d)$ queries as well, as established formally in Li et al.~\cite{li2014sketching}. 
\begin{comment}
\begin{defn}[Eigenratio]For $\gamma \in [0,1)$, we define the set of matrices with positive leading eigenvector and bounded \emph{eigenratio} between its first and second eigenvalues:
\begin{eqnarray}
\calM_{\gamma}:= \left\{M \in \Sym^{d\times d} :  \lambda_1(M) = \|M\| >0,  \frac{|\lambda_j(M)|}{\lambda_1(M)} \le \gamma \quad \forall j\ge 2\right\}.
\end{eqnarray}
\end{defn}
\end{comment}
To state our results, we construct for every $\boldgap \in (0,1)$ a distribution over matrices $\matM$ under which have a $\gap(\matM) \gtrsim \boldgap$. To do so, we introduce the classical $\GOE$ or Wigner law \cite{anderson2010introduction}:
\begin{defn}[Gaussian Orthogonal Ensemble (GOE)] We say that $\matW \sim \GOE(d)$ if the entries $\{\matW_{ij}\}_{1 \le i \le j \le d}$ are independent, for $1 \le i < j \le d$, $\matW_{ij} \sim \calN(0,1/d)$, for $i \in [d]$, $\matW_{ii} \sim \calN(0,2/d)$, and for $1 \le j < i \le d$, $\matW_{ij} = \matW_{ji}$. 
\end{defn}
In the rank one case, we will then take our matrix to be $\matM  := \matW + \boldlam\matu\matu^\top$, where $\matW \sim \GOE(d)$, $\matu \unifsim \sphere$, and $\boldlam  > 1$ is a parameter to be chosen. A critical result gives a finite-sample analogue of a classical result in random-matrix theory, which states that $\lambda_{\max}(\matM) \approx \boldlam + \lambdainv$ with high probability. On the other hand, $\|\matW\|_{\op}$ concentrates around $2$, and thus by eigenvalue interlacing $\lambda_{\max}(\matM) - \lambda_2(\matM) \gtrapprox \boldlam + \lambdainv - 2$. Motivated by this, we define the asymptotic gap of $\matM$:
\begin{eqnarray}\label{eq:boldgapeq}
\boldgap = \boldgap(\boldlam) := \frac{\boldlam + \lambdainv - 2}{\boldlam + \lambdainv} = \frac{(\boldlam-1)^2}{\boldlam^2 + 1}~.
\end{eqnarray}
It is well know that, for a fixed $\boldlam$, $\gap_r(\matM) \overset{\mathrm{prob}}{\to} \boldgap$ as $d \to \infty$. We give a finite sample analogue:
\begin{prop}[Finite Sample Eigengap of Deformed Wigner]\label{prop:egood_rkr} Let $\matM = \matW + \lambda \matU \matU^\top$, where $\matW \sim \GOE(d)$and $\matU \sim \Stief(d,r)$ are independent. For $\gamma \in (0,1)$, define the event
\begin{eqnarray*}
\Egood(\gamma) &:=& \left\{\|\matW\|_{\op} + (1-\gamma)(\boldlam + \lambdainv - 2) \le \lambda_{r}(\matM)\right\} \\
&\bigcap&  \left\{\lambda_{1}(\matM) \le (1+\gamma)(\boldlam + \lambdainv) \right\}. 
\end{eqnarray*}
There exists exists a polynomially bounded function $\frakq(\boldgap^{-1},(1-\boldgap)^{-1},\gamma^{-1},r,\log(1/\delta))$ such that for $\gamma, \delta \in (0,1/10)$, and $d \ge \frakq(\boldgap^{-1},(1-\boldgap)^{-1},\gamma^{-1},r,\log(1/\delta))$, $\Pr[\Egood(\gamma)] \ge 1 - \delta$. Moreover, on $\Egood(\boldlam,\gamma)$, $\gap_r(\matM) = \gap_r(\matMabs) \ge \frac{1-\gamma}{1+\gamma} \cdot \boldgap$. 
\end{prop}
The explicit polynomial can be derived from a more precise statement, Theorem~\ref{thm:main_spec_thm}. We now state more precise version of Theorem~\ref{mainthm}: 
\begin{thm}\label{thm:main_tech_thm} Fix a $\boldgap \in (0,1)$ and any $d \ge \frakq(\boldgap^{-1},(1-\boldgap)^{-1},2,r,\log(2))$ where $\frakq$ is as in Proposition~\ref{prop:egood_rkr}, and  let $\boldlam = \gapexp$ be the solution to Equation~\ref{eq:boldgapeq}. Let $\matM = \matW + \lambda \matU\matU^\top$ where $\matU \unifsim \Stief(d,r)$ and $\matW \sim \GOE(d)$.  Then for any $\Alg$ satisfying Definition~\ref{def:Query_def}, we have
\begin{multline}
\Exp_{\matM}\left[\Prit_{\Alg}\left[ \langle \Vhat, \matMabs \Vhat \rangle \ge \left(1 - \frac{\boldgap}{12}\right)\sum_{\ell =1}^{r} \sigma_{\ell}(\matM)  \right] \big{|} \Egood(1/2) \right] \\ 
\le 2e \cdot\exp\left( - \frac{d}{78  \log(d)\boldgap^{3}} \cdot \left(\gapexp\right)^{-18(\frac{\itT}{r}+2)} \right). 
\end{multline}
Where $\Prit_{\Alg}$ is the probability taken with respect to the randomness of the algorithm. Note that on $\Egood(1/2)$, $\gap_r(\matM) \ge \boldgap/3$.
\end{thm}
Observe $\gapexp \lesssim 1 + \BigOh{\sqrt{\boldgap}}$ as $\boldgap$ is bounded away from $1$. Hence, if $\boldgap \le 1/2$, $d$ is a large enough polynomial in $\boldgap$,  then if $(1 + \BigOh{\sqrt{\boldgap}})^{\itT/r} \le d^{\BigOmega{1}} $, or equivalently, $ \itT \ll \sqrt{\boldgap}^{-1/2} r \log d$, we see that the probability that $\langle \Vhat, \matMabs \Vhat \rangle \ge \left(1 - \frac{\boldgap}{12}\right)\sum_{\ell =1}^{r} \sigma_{\ell}(\matM) $ is at most $e^{-d^{\BigOmega{1}}}$, proving Theorem~\ref{mainthm}. 

\iftoggle{acm}{In Appendix~\ref{sec:further_results} of the full paper~\cite{simchowitz18}}{In Appendix~\ref{sec:further_results},} we present two additional results that follow as easy modifications of our proofs: Theorem~\ref{thm:main_tech_thm_rk1} presents an improved $\boldgap$-dependence for $r = 1$, and generalizes to the setting where $\Alg$ is allowed $\itT$ rounds of adaptivity, and makes a batch of $B$ queries per round; Theorem~\ref{thm:main_tech_thm_biggap} presents a modification of Theorem~\ref{thm:main_tech_thm} which establishes a sharp lower bound of $\BigOmega{\frac{r\log d}{-\log ( 1 - \boldgap)}}$ in the ``easy'' regime where $\boldgap$ approaches one. Our techniques can be adapted to show sharp lower bounds for adaptively testing between $H_1: \matM = \matW + \boldlam \matU \matU^\top$ against $H_0: \matM = \matW$; we omit these arguments in the interest of brevity.

%!TEX root =main.tex
\section{Proof Roadmap}
\subsection{Notation}
In what follows, we shall use bold letters $\matu$, $\matU$, $\matM$ and $\matW$ to denote the random vectors and matrices which arise from the deformed Wigner law; blackboard font $\Pr$ and $\Exp$ will be used to denote laws governing these quantities. We will use standard typesetting (e.g. $u,M$) to denote fixed (non-random) quantities vectors, as well as problem dimension $d$ and rank $r$ of the plant $\matU$.

Quantities relating to $\Alg$ will be in serif font; these include the queries $\vi$, responses $\wi$, and outputs $\vhat$ and $\Vhat$. The law of these quantities under $\Alg$ will be denoted $\Prit$ in bold serif.

Mathematical operators like $\gap(\matM)$ are $\lambda_1(\matM)$ are denoted in Roman or standard font, and asymptotic quantities like $\boldgap$ in Courier.

\subsection{Reduction from Eigenvector Computation to Estimating $\matU$\label{main_redux_sec}}
In this section, we show that an algorithm which adaptively finds a near-optimal $\Vhat$ implies the existence of a \emph{deterministic} algorithm which plays a sequence of orthonormal queries $\vone,\dots,\vtplusr$ for which $\sum_{i=1}^{\itT + r} \|\matU\vi\|^2$ is large.  Our first step is to show that if $\Vhat$ is near-optimal, then $\Vhat$ has a large overlap with $\matU$, in the following sense:
\begin{lem}\label{lem:estimation_reduction_good} Given any $\Vhat \in \Stief(d,r)$, any $r' \in [r]$, and under the event $\Egood(\boldlam,1/2)$, if $\langle \Vhat,\matMabs \Vhat \rangle \ge \left(1 - \frac{(r + 1 - r') \boldgap}{6r}   \right) \cdot \sum_{\ell=1}^r \sigma_\ell(\matM)$, then $\lambda_{r'}(\Vhat^\top \matU\matU^\top \Vhat) \ge \frac{ \boldgap}{4}$.
\end{lem}
In the rank one case, with $r = r' = 1$, $\Vhat = \vhat$ and $\matU = \matu$, the above lemma just implies that a near optimal $\vhat$ satisfies $\langle \vhat, \matu \rangle^2  \gtrsim \boldgap$. In the more general case, we have that $\lambda_{r'}(\Vhat^\top \matU\matU^\top \Vhat) \gtrsim \boldgap$ means that the image of $\Vhat$ needs to have ``uniformly good" coverage of the planted matrix $\matU\matU^\top$. The proof of Lemma~\ref{lem:estimation_reduction_good}  begins with the Lowner-order inequality
\begin{eqnarray*}
\Vhat^{\top}\matMabs \Vhat  = \Vhat^{\top}\matW\Vhat + \lambda \Vhat \preceq \|\matW\| I_r + \lambda \Vhat^\top \matU\matU^\top \Vhat
\end{eqnarray*}
In the rank one-case, this reduces to 
\begin{eqnarray*}
\vhat^{\top}\matMabs \vhat \le \|\matW\| + \lambda \langle \vhat, \matu \rangle^2. 
\end{eqnarray*}
Hence, if we want $\vhat^{\top}\matMabs \vhat \ge \lambda_{\max}(\matM) - \boldgap/2$, we must have that, since $\lambda_{\max}(\matM) - \|\matW\|$ concentrates around $\boldgap$,
\begin{eqnarray*}
\langle \vhat, \matu \rangle^2 \ge \frac{1}{\lambda}(\vhat^{\top}\matMabs\vhat  - \|\matW\|) \ge \frac{1}{\lambda}(\lambda_{\max}(\matM)  - \|\matW\| - \boldgap/2) \approx \boldgap/2\lambda~.
\end{eqnarray*}
which gives the lower bound. For $r > 1$, the proof becomes more technical, and is deferred to Appendix~\ref{sec:estimation_red_proof}.

Next, we argue that the performance of the optimal $\Vhat$ is bounded by a quantity depending only on the query vectors. As a first simplification, we argue that we may assume without loss of generality that $\vone,\vtwo,\dots$ are orthonormal. 
\begin{obs}\label{Observation_2}
We may assume that the queries are orthonormal, so that:
\begin{eqnarray*}
\itV_k := \begin{bmatrix} \vone | \mathsf{v}^{(2)} | \dots | \vk \end{bmatrix} \in \Stief(d,k)~,
\end{eqnarray*}
and that, rather that returning responses $\wk = \matM \vk$, the oracle returns responses $\wk = (I-\itV_{k-1}\itV_{k-1}^\top) \matM\vk$, where we note that $\itV_{k-1}\itV_{k-1}^\top $ is the projection onto $\mathrm{span}(\vone,\mathsf{v}^{(2)},\dots,\mathsf{v}^{(k-1)})$. 
\end{obs}
The assumption that the queries are orthonormal are valid since we can always reconstruct $k$-queries $\vone,\dots,\vk$ from an associated orthonormal sequence obtained via the Gram-Schmidt procedure. The reason we can assume the responses are of the form $\wk = (I-\itV_{k-1}\itV_{k-1}^\top) \matM\vk$ is that since $\Alg$ queries $\vone,\dots,\mathsf{v}^{(k-1)}$, it knows $\matM\itV_{k-1}\itV_{k-1}^\top $, and thus, since $\matM$ and $\itV_{k-1}\itV_{k-1}^\top$ are symmetric, it also knows $\itV_{k-1}\itV_{k-1}^\top \matM $, and thus $\matM\vk$ can be reconstructed from $\wk = (I-\itV_{k-1}\itV_{k-1}^\top) \matM\vk$. 
 The next observation shows that it suffices to upper bound $\lambda_{r'}(\Vhat^\top \matU\matU^\top \Vhat)$ with $\lambda_{r'}(\itV_{\itT+r}^\top \matU\matU^\top \itV_{\itT+r})$
\begin{obs}\label{obs:second_obs}
We may assume without loss of generality that $\Alg$ makes $r$ queries $\vtplus, \dots, \vtplusr$ after outputing $\Vhat$, and that $\lambda_{r'}(\Vhat^\top \matU\matU^\top \Vhat) \le \lambda_{r'}(\itV_{\itT+r}^\top \matU\matU^\top \itV_{\itT+r})$.
\end{obs}
This is valid because we can always modify the algorithm so that the queries $\vtplus,\dots,\vtplusr$ ensures that 
\begin{eqnarray*}
\mathrm{range}(\Vhat) \subset \mathrm{span}(\vone,\dots,\vt,\vtplus,\dots,\vtplusr) = \mathrm{range}(\itV_{\itT+r})
\end{eqnarray*}
In this case, we have that for all $\ell \in [r]$ (in particular, $\ell = r'$),
\begin{eqnarray*}
 \Vhat\Vhat^\top \preceq \itV_{\itT+r}\itV_{\itT+r}^\top &\Longrightarrow& \matU^\top\Vhat\Vhat^\top\matU \preceq \matU^\top\itV_{\itT+r}\itV_{\itT+r}^\top\matU\\
 &\Longrightarrow& \lambda_{\ell}\left(\matU^\top\Vhat\Vhat^\top\matU\right) \le \lambda_{\ell}(\matU^\top\itV_{\itT+r}\itV_{\itT+r}^\top\matU)\\
 &\Longleftrightarrow& \lambda_{\ell}\left(\Vhat^\top\matU\matU^\top\Vhat\right) \le \lambda_{\ell}(\itV_{\itT+r}^\top\matU\matU^\top\itV_{\itT+r})~.
\end{eqnarray*}
 Lastly, suppose it is the case that $\Pr_{\matU,\matM}[\lambda_{\ell}(\itV_{\itT+r}^\top\matU\matU^\top\itV_{\itT+r}) \ge B] \le b$ for any determinstic algorithm $\Alg_{\mathrm{det}}$, and some bounds $B > 0$ and $b \in (0,1)$. Then for any randomized algorithm $\Alg$, Fubini's theorem implies 
\begin{align*}
&~\Pr_{\matU,\matM,\Alg}[\lambda_{\ell}(\itV_{\itT+r}^\top\matU\matU^\top\itV_{\itT+r}) \ge B] \\
=&~\Exp_{\Alg}\Pr_{\matU,\matM}\left[\lambda_{\ell}(\itV_{\itT+r}^\top\matU\matU^\top\itV_{\itT+r}) \ge B \big{|} \text{ seed of } \Alg\right] \le b.
\end{align*}
as well. Hence, 
\begin{obs}
We may assume that, for all $k \in [\itT + r]$, the query $\vkplus$ is \emph{deterministic} given the previous query-observation pairs $(\vi,\wi)_{1 \le i \le k}$.
\end{obs}

\subsection{Lower Bounding the Estimation Problem}

As discussed above, we need to present lower bounds for the problem of sequentially selecting measurements $\vone,\vtwo,\dots,\vtplusr$ for which the associated measurement matrix $\itV_{\itT+r}$ has a large overlap with the planted matrix $\matU$. Proving a lower bound for this sequential, statistical problem constitutes the main technical effort of this paper. We encode the entire history of $\Alg$ up to time $i$ as $\itZ_i := (\vone,\wone,\dots,\vi,\wi)_{1 \le j \le i} $; in particular, $\itZ_{\itT+r}$ describes the entire history of the algorithm. 

Next, for $U \in \Stief(d,r)$, we let $\Prit_U$ denote the law of $\itZ_{\itT+r}$ where $\matM = \matW + \boldlam \matU\matU^\top$ conditioned on $\{\matU = U\}$. In the rank-one case, we $\Prit_u$ denotes the law of $\itZ_{\itT + r}$ where $\matM = \matW + \boldlam \matu\matu^\top$ conditioned on $\{\matu = u\}$. We will also abuse notation slightly by letting $\Prit_{0}$ denote the law obtained by running $\Alg$ on $\matM = \matW$, i.e. with $\matU = \matu = 0$. In the rank-one case, we have the following theorem, whose proof is outlined in Section~\ref{sec:body_rk_one_thm}:

\begin{thm}\label{thm:rank_one_thm_est}  Let $\matM = \matW + \boldlam \matU \matU^\top$, where $\matW \sim \GOE(d)$, and $\matu \unifsim \sphere$. Then for all $\delta \in (0,1/e)$,
\iftoggle{acm}
{
	\begin{align*}
	&\Exp_{\matu}\Prit_\matu\left[\exists k \ge 1:  \matu^\top \itV_k \itV_k^\top \matu \ge 32\boldlam^{4k} \frac{\boldgap^{-1}(\log \delta^{-1} + \boldgap^{-1/2})}{d}\right] \\
	&\le  \delta.
	\end{align*}
}
{
	\begin{eqnarray*}
	\Exp_{\matu}\Prit_\matu\left[\exists k \ge 1:  \matu^\top \itV_k \itV_k^\top \matu \ge 32\boldlam^{4k} \frac{\boldgap^{-1}(\log \delta^{-1} + \boldgap^{-1/2})}{d}\right] \le  \delta.
	\end{eqnarray*}
}

\end{thm}
The above theorem essential states that the quantity $\matu^\top \itV_k \itV_k^\top \matu $ can grow at most geometrically at a rate of $\boldlam^{4k}$, with an initial value sufficiently large in terms of the probability $\delta$ and $\boldgap$. In Section~\ref{sec:body_rk_r_thm}, we prove an analogous bound, which gives a geometric control on $\lambda_{r'}(\matU^\top \itV_k\itV_k \matU)$: 
\begin{thm}\label{thm:rank_r_thm_est} Let $\matM = \matW + \boldlam \matU \matU^\top$, where $\matU \unifsim \Stief(d,r)$. Then for $d \ge \gap^{-1/2}$ and $\delta \in (0,1)$, and $r' \in [r]$
\iftoggle{acm}
{
	\begin{align*}
	&\Exp_{\matU}\Prit_\matU\left[\forall k \in [d]: \lambda_{r'}(\matU^\top \itV_k\itV_k \matU) \le \frac{26 r \boldlam^{9k/r'} \log(20d^{2})\log(e\delta^{-1})}{d\boldgap^{2}}\right] \\
	&\ge 1 - \delta.
	\end{align*}
}
{
	\begin{eqnarray*}
	\Exp_{\matU}\Prit_\matU\left[\forall k \in [d]: \lambda_{r'}(\matU^\top \itV_k\itV_k \matU) \le \frac{26 r \boldlam^{9k/r'} \log(20d^{2})\log(e\delta^{-1})}{d\boldgap^{2}}\right] \ge 1 - \delta.
	\end{eqnarray*}
}
\begin{comment}
\begin{eqnarray*}
\Exp_{\matU}\Prit_\matU\left[\forall k \in [d]: \lambda_{r'}(\matU^\top \itV_k\itV_k \matU) \le \frac{26 r \boldlam^{9k/r'} \log(20d^{2})}{d\boldgap^{2}}\log(e\delta^{-1})\right] \ge 1 - \delta.
\end{eqnarray*}
\end{comment}
\end{thm}
In Section~\ref{sec:main_tech_thm_proof}, we combine Theorem~\ref{thm:rank_r_thm_est}, Lemma~\ref{lem:estimation_reduction_good}, and Observation~\ref{Observation_2} to prove Theorem~\ref{thm:main_tech_thm}. The final rate is a consequence of the fact that $\boldlam = \gapexp$. As mentioned in the paragraph New Techniques, our main technical hammer for proving Theorems~\ref{thm:rank_one_thm_est} and~\ref{thm:rank_r_thm_est} is a novel data-processing lower bound (Proposition~\ref{prop:likelihood_info}) which applies to ``truncated" distributions; the techniques are explained in greater detail in Appendix~\ref{sec:info_th_tools}.

\subsection{Conditional Likelihoods from Orthogonal Queries}
	We conclude with one further simplification which yields a closed form for the conditional distributions of our queries. Observe that it suffices to observe the queries $\wi = (I-\itV_{i-1}\itV_{i-1}^\top)\matM\vi = \matM \vi - \itV_{i-1}(\matM \itV_{i-1})^{\top}\vi$, our algorithm already ``knows'' the matrix $\matM \itV_{i-1}$ from the previous queries. Hence,
	\begin{obs} We may assume that we observe queries $\wi = \itP_i \matM\vi$, where $\itP_i := I-\itV_{i-1}\itV_{i-1}^\top$.
	\end{obs}
	We now show that, with our modified measurements $\wi = \itP_i\matM\vi$, then the query-observation pairs $(\vi,\wi)$ in the rank-one case have Gaussian likelihoods conditional on $\itZ_i$ and $\matu$.
	\begin{lem}[Conditional Likelihoods] \label{ConditionalLemma} Let $\itP_{i} := I - \itV_i\itV_i^\top$ denote the orthogonal projection onto the orthogonal complement of $\mathrm{span}(\vone,\dots,\vi)$. Under $\Prit_{u}$ (the joint law of $\matM$ and $\itZ_T$ on $\{\matu = u\}$), we have
	\begin{eqnarray*}
	(\itP_{i-1})\matM\vi \big{|} \itZ_{i-1} &\sim& \mathcal{N}\left(\boldlam (u^\top \vi) \itP_{i-1} u,\frac{1}{d}\itSigma_i \right)\\
	&& \text{where} \quad \itSigma_{i} := \itP_{i-1}\left(I_d+ \vi\viT \right)\itP_{i-1}.
	\end{eqnarray*}

	\begin{comment}
	\begin{multline*}
	(\itP_{i-1})\matM\vi \big{|} \itZ_{i-1} \sim \mathcal{N}\left(\boldlam (u^\top \vi) \itP_{i-1} u,\frac{1}{d}\itSigma_i \right), \quad \text{where} \quad \itSigma_{i} := \itP_{i-1}\left(I_d+ \vi\viT \right)\itP_{i-1}.
	\end{multline*}
	\end{comment}

	In particular, $\wi$ is conditionally independent of $\wone,\dots,\wiminus$ given $\vone,\dots,\viminus$ and $\matu = u$.
	\end{lem}
	Lemma~\ref{ConditionalLemma} is proved in Appendix~\ref{CondLemProof}. We remark that $\itSigma_i$ is rank-deficient, with its kernel being equal to the span of $\{\vone,\dots,\viminus\}$. Nevertheless, because the mean vector $\boldlam (u^\top \vi) \itP_{i-1} u$ lies in the orthogonal complement of $\ker \itSigma_i$, computing $\itSigma_i^{-1} (\boldlam (u^\top \vi) \itP_{i-1} u)$ can be understood as $\itSigma_i^{\dagger} (\boldlam (u^\top \vi) \itP_{i-1} u)$, where $\dagger$ denotes the Moore-Penrose pseudo-inverse~\cite{horn2012matrix}.

%!TEX root = main.tex

\section{Proof of Theorem~\ref{thm:rank_one_thm_est}\label{sec:body_rk_one_thm} ($r=1$)}
In this section, we prove a lower bound for the rank-one planted perturbation. The arguments in this section will also serve as the bedrock for the rank $r$ case, and exemplify our proof strategy. Given any $\itV_k \in \Stief(d,k)$, we introduce the notation $\Phi(\itV_k;\matu) := \langle \itV_k,\matu \matu^\top \itV_k \rangle
$, which is just the square Euclidean norm of the projection of $\matu$ onto the span of $\vone,\dots,\vk$. $\Phi(\itV_k;\matu)$ will serve as a ``potential function'' which captures how much information the queries $\vone,\dots,\vk$ have collected about the planted solution $\matu$, in a sense made precise in Proposition~\ref{prop:likelihood_info} below. The core of our argument is the following proposition, whose proof is given in the following subsection:
\begin{prop}\label{prop:recur_prop_rkone} Let $(\tau_k)$  be a sequence such that $\tau_0 = 0$, and for $k \ge 1$, $\tau_{k} \ge 2k$. Then for all $\eta > 0$, one has 
\begin{align}\label{eq:recur_bound_rkone}
&\Exp_{\matu}\Prit_{\matu}[\{ \Phi(\itV_k;\matu) \le \frac{\tau_k}{d}\} \cap \{ \Phi(\itV_{k+1};\matu)> \frac{\tau_{k+1}}{d}\} ]\nonumber \\
\le~&\exp\left\{ \frac{\eta}{2(1+\eta)}  \left((1+\eta)\boldlam^2 \tau_k  - \left(\sqrt{\tau_{k+1}}-\sqrt{2k+2}\right)^2\right) \right\}~.
\end{align}
\end{prop}
The above proposition states that, given two thresholds $\tau_k,\tau_{k+1} \ge 0$, the probability that $d (\itV_{k+1};\matu)$ exceeds the threshold $\tau_{k+1}$ on the event that $d (\itV_{k};\matu)$ does not exceed the threshold $\tau_k$ is small. Hence, for a sequence of thresholds $0 = \tau_0 < \tau_1 < \dots $, we have

\begin{comment}
\begin{multline*}
\Exp_{\matu}\Prit_{\matu}[ \exists k \ge 0 : \Phi(\itV_{k+1};\matu)> \tau_{k+1}/d ]  \\
\le \sum_{k=0}^{\infty} \Exp_{\matu}\Prit_{\matu}[\{ \Phi(\itV_k;\matu) \le \tau_k/d\} \cap \{ \Phi(\itV_{k+1};\matu)> \tau_{k+1}/d\} ]. 
\end{multline*}
\end{comment}
\begin{align*}
&\Exp_{\matu}\Prit_{\matu}[ \exists k \ge 0 : \Phi(\itV_{k+1};\matu)> \tau_{k+1}/d ]  \\
\le~&\sum_{k=0}^{\infty} \Exp_{\matu}\Prit_{\matu}[\{ \Phi(\itV_k;\matu) \le \tau_k/d\} \cap \{ \Phi(\itV_{k+1};\matu)> \tau_{k+1}/d\} ]. 
\end{align*}

Theorem~\ref{thm:rank_one_thm_est} now follows by choosing the appropriate sequence $\tau_k(\delta)$, selecting $\eta$ appropriately, and verifying that the right hand side of the above display is at most $2\delta$. For intuition, setting $\eta = \boldlam - 1$, we see that once $\tau_k$ gets large, it is enough to choose $\tau_{k+1} = \boldlam^4 \tau_k$ ensure that the exponent in Equation~\eqref{eq:recur_bound_rkone} is a negative number of sufficiently large magnitude. The details are worked out in Appendix~\ref{sec:rk_one_thm_proof}. We now turn to the proof of Proposition~\ref{prop:recur_prop_rkone}.

\subsection{Proving Proposition~\ref{prop:recur_prop_rkone}}

To prove Proposition~\ref{prop:recur_prop_rkone},  we argue that if $\tau_k$ is much smaller than $\tau_{k+1}$, then under the event $\{ \Phi(\itV_k;\matu) \le \tau_k/d\}$, the algorithm  does not have enough information about $\matu$ to select a new query vector $\vkplus$ for which $\{ \Phi(\itV_{k+1};\matu)> \tau_{k+1}/d\}$. The following proposition is proved in Section~\ref{sec:info_th_rkone}, and arises as a special case of a more general information theoretic tools introduced in that section.
\begin{prop}\label{prop:info_th_rkone} Let $\calD$ be any distribution supported on $\calS^{d-1}$, and let $\eta > 0$. Then, 

\begin{comment}
\begin{align*}
&\Exp_{\matu \sim \calD}\Prit_{\matu}\left[\{ \Phi(\itV_k;\matu) \le \tau_k/d\} \cap \{ \Phi(\itV_{k+1};\matu)> \tau_{k+1}/d\} \right] \le \\
&\left(\Exp_{\matu \sim \calD}\Exp_{\itZ_{k} \sim \Prit_0}\left[\left(\frac{\rmd\Prit_{\matu}(\itZ_{k})}{\rmd \Prit_0(\itZ_{k})}\right)^{1+\eta} \I(\{ \Phi(\itV_k;\matu) \le \tau_k/d\})\right] \\
&\cdot \sup_{V \in \Stief(d,k+1)} \Pr_{\matu \sim \calD}[\Phi(V;\matu)> \tau_{k+1}/d]^\eta   \right)^{\frac{1}{1+\eta}}.
\end{align*}
\end{comment}
\begin{align*}
&\Exp_{\matu \sim \calD}\Prit_{\matu}\left[\{ \Phi(\itV_k;\matu) \le \tau_k/d\} \cap \{ \Phi(\itV_{k+1};\matu)> \tau_{k+1}/d\} \right] \le \\
&\left(\Exp_{\matu \sim \calD}\Exp_{\itZ_{k} \sim \Prit_0}\left[\left(\frac{\rmd\Prit_{\matu}(\itZ_{k})}{\rmd \Prit_0(\itZ_{k})}\right)^{1+\eta} \I(\{ \Phi(\itV_k;\matu) \le \tau_k/d\})\right]\right)^{\frac{1}{1+\eta}} \\
&\cdot \left(\sup_{V \in \Stief(d,k+1)} \Pr_{\matu \sim \calD}[\Phi(V;\matu)> \tau_{k+1}/d]^\eta   \right)^{\frac{1}{1+\eta}}.
\end{align*}
\end{prop}
As is typical for data-processing inequalities, the above proposition relates the probability of the event $\{ \Phi(\itV_{k+1};\matu)> \tau_{k+1}/d\}$ to an ``information'' term capturing the size of power of likelihood ratios $\left(\frac{\rmd\Prit_{\matu}^k}{\rmd \Prit_0^k}\right)^{1+\eta}$ restricted to the event $\{ \Phi(\itV_k;\matu) \le \tau_k/d\}$, and an ``entropy'' term, which captures how unlikely it would be to find a $V \in \Stief(d,k+1)$ such that $\{\Phi(V;\matu)> \tau_{k+1}/d\}$ by just randomly guessing. We remark that Proposition~\ref{prop:info_th_rkone} differs from many standard data-processing inequalities (e.g., Fano's inequality or the bounds in Chen et al.~\cite{chen2016bayes}) in two ways: first, we use an unorthodox information measure: $1 + \eta$-powers of likelihood ratios for $\eta$ close to zero.  This choice of divergence gives us granular control in the case when $\boldlam$ is close to one. As mentioned above, we will ultimately take $\eta$ by setting $\eta = \boldlam - 1$.  Second, we consider the restriction of the likelihood ratios to the ``low-information" event $\{ \Phi(\itV_k;\matu) \le \tau_k/d\}$. As mentioned above, this is necessary to deal with the ill-behaved tails of the likelihoods. In Appendix~\ref{sec:info_th_tools} we present additional general data-processing inequalities for truncated distributions. 

Deducing Proposition~\ref{prop:recur_prop_rkone} from Proposition~\ref{prop:info_th_rkone} now follows readily by bounding the ``entropy'' and ``information'' terms. We use concentration of measure on the sphere to bound the entropy term as follows (see Appendix~\ref{sec:small_ball} for proof):
\begin{lem}\label{lem:small_ball_prob} For any fixed $V \in \Stief(d,k+1)$ and $\tau_{k+1} \ge 2(k+1)$ , we have 
\begin{eqnarray*}
\Pr_{\matu}[ \matu^\top V^\top V \matu \ge \tau_{k+1}/d] \le \exp\left\{-\frac{1}{2}\left(\sqrt{\tau_{k+1}}-\sqrt{2(k+1)}\right)^2\right\}.
\end{eqnarray*}
\end{lem}
We now state \ref{prop:likelihood_info} which gives an upper bound on the information term. The proof is considerably more involed that that of Lemma~\ref{lem:small_ball_prob}, and so we present a sketch in Section~\ref{sec:likelihood_proof} below.
\begin{prop}\label{prop:likelihood_info}
For any $\tau_k \ge 0$ and any fixed $u \in \calS^{d-1}$, we have
\iftoggle{acm}
{
	\begin{multline}\label{Chi_Plus_1_Eq}
	\Exp_{\Prit_0}\left[\left(\frac{\mathrm{d}\Prit_u(\itZ_k)}{\mathrm{d}\Prit_0(\itZ_k)}\I(\Phi(\itV_k;u) \le \tau_k/d)\right)^{1+\eta} \right] \le e^{-\frac{\eta(1 + \eta)\boldlam^2 \tau_k}{2}}.
	\end{multline}
}
{
	\begin{eqnarray}\label{Chi_Plus_1_Eq}
	\Exp_{\Prit_0}\left[\left(\frac{\mathrm{d}\Prit_u(\itZ_k)}{\mathrm{d}\Prit_0(\itZ_k)}\I(\Phi(\itV_k;u) \le \tau_k/d)\right)^{1+\eta} \right] \le \exp\left(\frac{\eta(1 + \eta)}{2}\boldlam^2 \tau_k\right).
	\end{eqnarray}
}
\end{prop}
In particular, by taking an expectation over $\matu \sim \calS^{d-1}$, we have that 
\begin{eqnarray*}
\Exp_\matu \Exp_{\Prit_0}\left[\left(\frac{\mathrm{d}\Prit_\matu(\itZ_k)}{\mathrm{d}\Prit_0(\itZ_k)}\I(\Phi(\itV_k;\matu) \le \tau_k/d)\right)^{1+\eta} \right] \le \exp\left(\frac{\eta(1 + \eta)}{2}\boldlam^2 \tau_k\right).
\end{eqnarray*}
This motivates the choice of $\Phi(\itV_k;\matu)$ as an information-potential, since it gives us direct control over bounds of the likelihood ratios.  Propostion~\ref{prop:recur_prop_rkone} now follows immediately from stringing together Proposition~\ref{prop:info_th_rkone}, Proposition~\ref{prop:likelihood_info} for the ``information term'', and Lemma~\ref{lem:small_ball_prob} for the ``entropy term''.

%!TEX root =main.tex

\subsection{Proof of Proposition~\ref{prop:likelihood_info} \label{sec:likelihood_proof} (``Information Term'')}

	The difficulty in Proposition~\ref{prop:likelihood_info} is that truncating to the event $\I\left(\Phi(\itV_k;u)) \le \tau_k\right)$ introduces correlations between the conditional likelihoods that don't arise in the conditionally independent likelihoods of Lemma~\ref{ConditionalLemma}. Nevertheless, we use a careful peeling argument (Appendix~\ref{sec:gen_ub_llProof}) to upper bound the information term, an expected product of likelihoods, by a product of expected conditional likelihoods which we can compute. Formally, we have 
	\begin{prop}[Generic upper bound on likelihood ratios] \label{Generic_UB_LL}Fix an $u,s\in \calS^{d-1}$, and fix $r > 0$. Define the likelihood functions
	\begin{eqnarray}\label{g_func_def}
	g_i(\itVtil_i) &:=& \Exp_{\Prit_0}\left[\left(\frac{\rmd\Prit_u(\itZ_i |\itZ_{i-1})}{\rmd\Prit_0(\itZ_i | \itZ_{i-1})}\right)^r \I(\itV_i = \itVtil_i) \right].
	\end{eqnarray}
	Then for any subset $\calV_k \subset \Stief(d,k)$, we have
	\begin{eqnarray}\label{product_eq}
	\Exp_{\Prit_0}\left[\left(\frac{\rmd\Prit_u(\itZ_i |\itZ_{i-1})}{\rmd\Prit_0(\itZ_i | \itZ_{i-1})}\right)^r \I(\itV_k  \in \calV_k) \right] &\le& \sup_{\itVtil_{k} \in \calV_k}\prod_{i=1}^kg_i(\itVtil_{1:i})~,
	\end{eqnarray}
	where $\itVtil_{1:i}$ denotes the first $i$ columns of $\itVtil_k$.
	\end{prop}
	Here, we remark that the tilde-notation ($\itVtil_k, \vonetil,\vitil$,...) represents fixed vectors which the random quantities $\itV_k,\vone,\vitil$ etc.. For example, in the event $\{\itV_k = \itVtil_k\}$, $\itVtil_k$ is considered to be a deterministic matrix. 

	We can now invoke a  computation of the $1+\eta$-th moment of the likelihood ratios between two Gaussians, proved in Appendix~\ref{sec:power_div_proof}.
	\begin{lem}\label{lem:power_divergence_comp} Let $\Pr$ denote the distribution $\calN(\mu_1,\Sigma)$ and $\Q$ denote $\calN(\mu_2,\Sigma)$, where $\mu_1,\mu_2 \in (\ker \Sigma)^{\perp}$. Then 
	\begin{eqnarray}
	\Exp_{\Q}\left[\left(\frac{\rmd \Pr}{\rmd \Q}\right)^{1+\eta}\right] =  \exp\left(\frac{\eta(1 + \eta)}{2}(\mu_1-\mu_2)^{\top} \Sigma^{\dagger} (\mu_1-\mu_2)\right).
	\end{eqnarray}
	\end{lem}
	We are now in a position to prove Proposition~\ref{prop:likelihood_info}:
	\begin{proof}[Proof of Proposition~\ref{prop:likelihood_info}]
	Fix a $u \in \calS^{d-1}$, and we shall and apply Proposition~\ref{Generic_UB_LL} with $r_u = r_0 = 1+\eta$ and $r_s = 0$. In the language of Proposition~\ref{Generic_UB_LL} , we have
	\begin{eqnarray*}
	g_i(\itVtil_i ) &=& \Exp_{\Prit_0}\left[\left(\frac{\rmd\Prit_u(\itZ_i |\itZ_{i-1})}{\rmd\Prit_0(\itZ_i | \itZ_{i-1})}\right)^{1+\eta} \big{|} \itV_i = \itVtil_i \right]  \\
	&=& \Exp_{\Prit_0}\left[\left(\frac{\rmd\Prit_u(\wi |\itZ_{i-1})}{\rmd\Prit_0(\wi | \itZ_{i-1})}\right)^{1+\eta} \big{|} \itV_i = \itVtil_i \right] 
	\end{eqnarray*}
	Now, observe that, $\Prit_u(\wi |\itZ_{i-1})$ is the density of $\mathcal{N}(\boldlam \langle u, \vi \rangle \cdot \itP_{i-1}u, \frac{1}{d}\itSigma_i)$ and $\Prit_0(\wi |\itZ_{i-1})$ is the density of $\mathcal{N}(0, \frac{1}{d}\itSigma_i)$. Since $\itSigma_i = \itP_{i-1}\left(I_d+ \vi\viT \right)\itP_{i-1}$, we have $\itP_{i-1}\itSigma_i^{\dagger}\itP_{i-1} = \itP_{i-1} \preceq I$. Thus,
	\begin{eqnarray}\label{eq:itSigma_eq}
	u^{\top}\itP_{i-1}(\itSigma_i/d)^{\dagger}\itP_{i-1}u \le d\|u\|^2= d\quad \forall u \in \calS^{d-1}~.
	\end{eqnarray}
	Hence, by Lemma~\ref{lem:power_divergence_comp}, we have
	\begin{eqnarray*}
	g_i(\itV_i )  &=& \exp\left( \frac{\eta(1+\eta)\boldlam^2\langle u, \vi \rangle^2}{2}u^{\top}\itP_{i-1}(\itSigma_i/d)^{\dagger}\itP_{i-1}u\right)\\
	&\overset{\text{\eqref{eq:itSigma_eq}}}{\le}& \exp\left( \frac{\eta(1+\eta)\boldlam^2\cdot d\langle u, \vi \rangle^2}{2}\right).
	\end{eqnarray*}
	Hence, if $\calV_k := \{\itVtil_k \in \Stief(d;k): \Phi(\itVtil_k;u) \le \tau_k/d\}$, then Proposition~\ref{Generic_UB_LL} implies
	\begin{eqnarray*}
	&&\Exp_{\Prit_0}\left[\left(\frac{\rmd\Prit_u(Z_k}{ \rmd \Prit_0(Z_k) }\right)^{1+\eta}I(\itV_k \in \calV_k) \right] \\
	&\le& \sup_{\itVtil_k \in \calV_k}\prod_{i=1}^k \exp( \frac{\eta(1+\eta)\boldlam^2\cdot d\langle u, \itVtil_k[i] \rangle^2}{2})\\
	 &=& \sup_{\itVtil_k \in \calV_k} \exp( \frac{d\eta(1+\eta)\boldlam^2\Phi(\itVtil_k;u)}{2})\\
	&\le&  \exp( \frac{\eta(1+\eta)\boldlam^2 \tau_k}{2})~.
	\end{eqnarray*}
	\end{proof}
%!TEX root =main.tex
\subsection{Proof of Proposition~\ref{prop:info_th_rkone}\label{sec:info_th_rkone}}

We begin by introducing the general framework for Bayes risk lower bounds as presented in Chen et al.\ \cite{chen2016bayes}. We begin with an estimand parameter $\boldtheta$ drawn from some prior $\calP$ over a measureable space $(\Theta,\calG)$. To each fixed $\theta \in \Theta$ is associated a measure $\mu_{\theta}$ over a measureable space $(\calX,\calF)$, governing a random variable $\matx$. In our setting, consider the rank-one deformed Wigner $\matM = \matW + \lambda \matu\matu^\top$, a fixed $\Alg$ and round $k$. Then the estimand is $\boldtheta = \matu$, the measures $\{\mu_{\theta}\}_{\theta \in \Theta}$ is the measure correspond to the laws $\Prit_{\matu}(\cdot)$ over $\matx = \itZ_k$. 

We would like to use $\matx$ to make an action which tells us something useful about $\boldtheta$. Formally, consider a space of action $\calA$ and a space $\frakA$ of measurable action mappings $\fraka: \calX \mapsto \calA$, and an indicator function $\calI(\cdot,\cdot) : \fraka \times \Theta \mapsto \{0,1\}$ of a ``good event'' that we would want an algorithm to achieve. In our Wigner model, $\calA$ will denote the space  $\Stief(d,k)$, each $\fraka$ might denote a mapping from the playout history $\itZ_k$ to the measurements $\itV_k$, and the good event will be 
\begin{eqnarray*}
\I(\Phi(\itV_k;\boldtheta)) =  \I(\boldtheta^\top \itV_k \itV_k^\top \boldtheta > \tau)
\end{eqnarray*}
for some threshold $\tau$. 

As we want to show lower bounds, our goal will be to show that the quantity
	\begin{eqnarray}\label{eq:v_opt_def}
	V_{\opt} := \sup_{\fraka \in \calA}\Exp_{\boldtheta \sim \calP}\mu_{\boldtheta}[\{\calI(\fraka(\matx),\boldtheta) = 1\}]
	\end{eqnarray}
	cannot be too large. The key difference between this setup and typical information-theoretic lower bounds is that we will not require the measures $\mu_{\theta}$ to be normalized (i.e., probability measures), only that they have finite mass $\mu_{\theta}(\calX) < \infty$. Our motivation for this is that we will take $\mu_{\theta}$ to be truncated probability measures, or measures $\mu_\theta$ with $\mu_\theta(\calX) \le 1$ for which there exists a probability distribution $\overline{\mu}_{\theta}$ and an event $B_{\theta} \in \calF$ such that
	\begin{eqnarray} \label{eq:truncated_measure}
	\forall A \in \calF: \mu_{\theta}(A) =  \overline{\mu}_{\theta}(A \cap B_{\theta}).
	\end{eqnarray}
	To make this concrete, suppose in our above example that use $ \Prit_{\matu}$ as our unnormalized measures $\overline{\mu}_{\theta}$, and the sets $B_{u} := \{\Phi(\itV_{k-1};u) \le \tau_{k-1}\}$. Then, $\mu_{\theta}$ correspond to the subdistribution 
	\begin{eqnarray*}
	A \mapsto \Prit_\theta (A \cap \{\Phi(\itV_{k-1};\theta) \le \tau_{k-1}\}).
	\end{eqnarray*}
	Hence, we have
	\begin{comment}
	\begin{eqnarray}
	V_{\opt} &=& \sup_{\fraka \in \frakA}\Exp_{\matu \sim \calP}\Prit_{\theta}\left(\{\Phi(\fraka(\itZ_k);\matu) \le   \tau \} \cap \{\Phi(\itV_{k-1};\matu) \le \tau_{k-1}\} )\right)  \nonumber \\
	&\ge& \label{eq:V_opt_app} \Exp_{\matu \sim \calP}\Prit_{\matu}\left(\{\Phi(\itV_k;\matu)  > \tau \}\cap \{\Phi(\itV_{k-1};\matu) \le \tau_{k-1}\} \right),
	\end{eqnarray}
	\end{comment}
	%\begin{comment}
	\begin{eqnarray}
	V_{\opt}&=& \sup_{\fraka \in \frakA}\Exp_{\matu}\Prit_{\theta}\left(\{\Phi(\fraka(\itZ_k);\matu) \le   \tau \} \cap \{\Phi(\itV_{k-1};\matu) \le \tau_{k-1}\} )\right)  \nonumber \\
	&\ge& \label{eq:V_opt_app} \Exp_{\matu \sim \calP}\Prit_{\matu}\left(\{\Phi(\itV_k;\matu)  \le \tau \}\cap \{\Phi(\itV_{k-1};\matu) \le \tau_{k-1}\} \right),
	\end{eqnarray}
	%\end{comment}
	which is precisely the quantity we wish to control in Proposition~\ref{prop:info_th_rkone}. More generally, considering such truncated measures is desirable in adaptive settings when we may want to consider the probabililty than a sequential algorithm takes a certain action at stage $k$, on the event that it has taken certain actions prior to stage $k$. Our main theorem is as follows:

\begin{thm}[Bayes risk lower bound for sub-distributions]\label{Fano_sub_distr} Let $\calP$ be a prior distribution over $(\Theta,\calG)$, let $\nu$ and $\{\mu_{\theta}\}$ be a family of finite measures over $(\calX,\calF)$. Let $\calA$ denote an action space, let $\frakA$ denote the space of decision rules $\fraka$ from $\calX$ to $\calA$, and let $\calI: \calA \times \Theta \mapsto \{0,1\}$ be an indicator function. Let
	\begin{eqnarray}\label{eq:v_0_def}
	V_0 := \sup_{a \in \calA}\Pr_{\theta \sim \calP}[\{\calI(a,\theta) = 1\}]
	\end{eqnarray}
	denote the optimal value of the best action taken without observing $\matx$. If $f$ is non-negative, convex, $\nu(\calX) \le 1$,  $\sup_{\theta \in \Theta} \mu_\theta(\calX) \le 1$, and $\mu_{\theta}$ are absolutely continuous with respect to $\nu$\footnote{See e.g. Kallenberg \cite{kallenberg2006foundations} for a review of absolute continuity.} for all $\theta \in \Theta$, and either i) $x\mapsto xf(1/x)$ is non-increasing or ii) $\nu(\calX) = 1$, then
	\begin{eqnarray*}
	 \Exp_{\theta \sim \calP}\Exp_{\nu}[f(\frac{\rmd\mu_\theta}{\rmd \nu})] \ge V_0 f\left(\frac{V_\opt}{V_0}\right)~.
	\end{eqnarray*}
	\end{thm}

	In essence, the above theorem relates two quantities: on the right, a quantity comparing the optimal value $V_\opt$ to be the best ``data-oblivious''' value $V_0$, which depends only on the ''spreadness'' of the prior $\calP$ and not on the condition laws $\mu_{\theta}$.  The quantity $\Exp_{\nu}[f(\frac{\rmd\mu_\theta}{\rmd \nu})]$ on the left hand side is known as a $f$-divergence \cite{csiszar1972class} between $\mu_\theta$ and $\nu$, which measures the dissimilarity between the measures $\mu_{\theta}$ and $\nu$; we introduce them in full generality in Appendix~\ref{sec:gen_data_proc}. If there exists a measure $\nu$ for which $\Exp_{\theta \sim \calP}\Exp_{\nu}[f(\frac{\rmd\mu_\theta}{\rmd \nu})]$ is small, it means that the measures $\mu_{\theta}$ are in a sense similar on average, and hence the variable $\matx$ doesn't convey too much information about the estimand $\boldtheta$, and thus $V_\opt$ cannot be considerably larger than $V_0$.

	Theorem~\ref{Fano_sub_distr} is proven in Appendix~\ref{sec:gen_data_proc}, along with a more general bound, Theorem~\ref{thm:gen_fano}. We now conclude this subsection with the proof Proposition~\ref{prop:info_th_rkone}:
	\begin{proof}[Proposition~\ref{prop:info_th_rkone}] We apply Theorem~\ref{Fano_sub_distr} with $f(x) = x^{1+\eta}$ (which is non-negative on $(0,\infty)$, convex, and $xf(1/x) = x^{-\eta}$ non-increasing). For clarity, we will index our truncated laws $\mu_u(A)$ by $u \in \calS^{d-1}$. Now,  we take  $\mu_u(A) := \Prit_u (A \cap \{\Phi(\itV_{k-1};u) \le \tau_{k-1}\})$. We also take $\nu$ to be  the law of the law $\itZ_k$ under $\Prit_0$, the law of $\Alg$ under $\matM = \matW$, without the rank-one spike. Since, $ \Prit_\theta \ll  \Prit_0$ we see that $\mu_{\theta} \ll \nu$. Moreover, we have that
	\begin{eqnarray*}
	\frac{\rmd\mu_u}{\rmd \nu} = \frac{\rmd\Prit_u}{\rmd\Prit_0}\I(\Phi(\itV_{k-1};\theta) \le \tau_{k-1}).
	\end{eqnarray*}
	Lastly we take $\calP$ to be the uniform distribution on the sphere.
	Hence, the right hand side of Theorem~\ref{Fano_sub_distr} reads
	\begin{eqnarray*}
	 \Exp_{\matu \sim \sphere} \Exp_{\itZ_k \sim \Prit_0}\Big[\left(\frac{\rmd\Prit_{\matu}(\itZ_k)}{\rmd\Prit_0(\itZ_k)}\right)^{1+\eta}\I(\Phi(\itV_{k-1};\matu) \le \tau_{k-1})\Big].
	\end{eqnarray*}
	On the other hand, we now choose the action space $\calA = \Stief(d,k)$ and and the indicator $\calI(V_k,u) := \I( \Phi(\itV_k;u) > \tau)$. Using \eqref{eq:v_opt_def}, we have
	\begin{eqnarray*}
	&&V_0 f(V_\opt/V_0) ~=~ V_0^\eta V_{\opt}^{1+\eta} \\
	&\overset{\text{Eq.}~\eqref{eq:v_0_def}}{=}& V_\opt^{1+\eta}\left(\sup_{V_k \in \Stief(d,k)}\Pr_{\matu \sim \sphere}[\Phi(V_k;\matu) > \tau]\right)^\eta.
	\end{eqnarray*}
	Solving for $V_\opt$, we have
	\begin{eqnarray*}
	&&\left(\sup_{V_k \in \Stief(d,k)}\Pr_{\matu \sim \sphere}[\Phi(V_k;\matu) > \tau]\right)^{\frac{\eta }{1+\eta}} \\
	&\times& \left(\Exp_{\matu} \Exp_{\itZ_k \sim\Prit_0}\left[ \left(\frac{\rmd\Prit_{\matu}(\itZ_k)}{\rmd\Prit_0(\itZ_k)}\right)^{1+\eta}\I(\Phi(\itV_{k-1};\matu) \le \tau_{k-1})\right]\right)^{\frac{1}{1+\eta}} \\
	&\ge& V_{\opt}\\
	&\overset{\text{Eq.}~\ref{eq:V_opt_app}}{\ge}& \Exp_{\matu \sim \calP}\Prit_{\matu}\left(\{\Phi(\itV_k;\matu)  > \tau \}\cap \{\Phi(\itV_{k-1};\matu) \le \tau_{k-1}\} \right).
	\end{eqnarray*}
	This concludes the proof.
	\end{proof}
	
%!TEX root = main.tex

\section{Proof of Theorem~\ref{thm:rank_r_thm_est}\label{sec:body_rk_r_thm} ($r \ge 1$)}

Here we present a proof outline of Theorem~\ref{thm:rank_r_thm_est} which modifies the insights from the rank one case to get a recursion for the determinant $\det(\matU^\top \itV_k\itV_k \matU + \Delta I_r) $. We will still use the rank-one potential from the last section $\Phi(\cdot;\cdot)$, but will instead be interested in
\begin{eqnarray}
\Phi(\itV_k;\matU e) := e^\top \matU^\top \itV_k \itV_k^\top \matU e, \quad e \in \R^r ~,
\end{eqnarray}
which measures the amount of information gathered about $\matU$ in the direction of $e$. We will want to show that, with high probability, the following event holds for an appropriate choice of parameters:
\begin{eqnarray}
&&\calE(\lamtil,\Delta,k_{\max})~:=~ \left\{\forall e \in \R^r, k \in \{1,\dots,k_{\max}\}\right\} \nonumber \\
&\bigcap& \left \{d\Phi(\itV_k;\matU e) + \Delta \le \lamtil (d\Phi(\itV_{k-1};\matU e) + \Delta)\right\}~.~\label{eq:calErho_def}
\end{eqnarray}
In other words, $\calE(\lamtil,\Delta,k_{\max})$ corresponds to the event that, up to a translation by $\Delta$, the potentials $\Phi(\itV_k;\matU e)$ grows at most geometrically by a factor of $\lamtil$ in every direction. We should think of $\lamtil$ as being of order $\boldlam^{O(1)}$, which may be quite close to $1$. Hence, the translation $\Delta$ gives us additional slack which will be necessary for high-probability bounds. %We then have that on $\calE(\lamtil,\Delta,k_{\max})$
\begin{lem}\label{lem:Determinant_Growth_Lemma} On $\calE(\lamtil,\Delta,k_{\max})$, it holds that 
\begin{eqnarray*}
\det(d\matU^\top \itV_k\itV_k \matU + \Delta I_r) \le \lamtil^k \det(\Delta I_r) = \lamtil^k\Delta^{r}.
\end{eqnarray*}
\end{lem}
The proof of the above claim follows by first decomposing 
\begin{eqnarray*}
d\matU^\top \itV_k\itV_k \matU + \Delta I_r = d\matU^\top \itV_{k-1}\itV_{k-1} \matU + \Delta I_r + \matU^\top \vk\vkT \matU,
\end{eqnarray*}
and applying the Sherman-Morrison rank-one update formula to control the growth of the $\det(\matU^\top \itV_k\itV_k \matU + \Delta I_r)$ in each stage. In Section~\ref{sec:rank_r_thm_est_app_proof}, we prove Theorem~\ref{thm:rank_r_thm_est} by translating Lemma~\ref{lem:Determinant_Growth_Lemma} into a growth bound on $\lambda_{r'}(\matU^\top \itV_k\itV_k \matU)$, and control the probability of the event $\calE(\lamtil,\Delta,k_{\max})$ with the following proposition:
\begin{prop}\label{prop:CalE_prob}
Let $\rho \ge \boldlam^3 c_{d,r}$, fix $k_{\max} \ge 1$, and set $\Delta \ge \frac{\rho(2k_{\max}+2)}{(\rho - 1)^3}$. Then
\begin{eqnarray*}
\Pr[\calE(\rho^2,\Delta, k_{\max})] \ge 1 - (20 d/(\rho - 1))^{r+2}\exp\left\{ \frac{-\boldlam^3(\boldlam - 1)\Delta}{2} \right\}.
\end{eqnarray*}
\end{prop}

\subsection{Proof of Proposition~\ref{prop:CalE_prob}}
We will proeceed by arguing that an analogue of $\calE(\lamtil,\Delta,k_{\max})$ holds for a fixed $e \in \R^{r-1}$, and then extending to all of $\calS^{r-1}$ via a covering argument. Our first step is to prove an analogue of  Proposition~\ref{prop:recur_prop_rkone} for the potential $\Phi(\itV_k;\matU e)$. This ends up between very similar to the rank-one case, with the modification that we end up conditioning on the matrix $\matU (I - ee^\top)$, and consequently pay a slight penalty (see the factor $c_{d,r}$ below) for reducing the effective problem dimension from estimating a random vector in $\R^d$ to one in $\R^{d - r - 1}$. Precisely, we have the following:

\begin{prop}\label{prop:rk_k_info_th_prop}  Define the constant $c_{d,r} := \frac{d}{d-r-1}$. Fix an $\eta > 0$, $k \ge 0$, and let $\tau_k,\tau_{k+1} \ge 0$, with $\tau_0 = 0$. Then for any fixed $e \in \R^r$
\begin{equation}\label{eq:recur_bound_rkr}
\begin{aligned}
&\Exp_{\matU}\Prit^k_{\matU}[\{ \Phi(\itV_k;\matU e) \le \tau_k/d\} \cap \{ \Phi(\itV_{k+1};\matU e)> \tau_{k+1}/d\} ] \le \\
&\exp\left\{ \frac{\eta}{2}  \left(\boldlam^2 \tau_k  - \frac{\left(\sqrt{c_{d,r}\tau_{k+1}}-\sqrt{2k+2}\right)^2}{1+\eta}\right) \right\}. 
\end{aligned}
\end{equation}
\end{prop}
Proposition~\ref{prop:rk_k_info_th_prop} is proved in Appendix~\ref{sec:rk_k_info_th_prop}. We can now prove that a point-wise analogue of $\calE(\lamtil,\Delta,k_{\max})$ for each $e \in \calS^{r-1}$ holds for $\lamtil = \boldlam^3 c_{d,r}$:
\begin{lem}\label{lem:one_fixed_e} Let $\rho \ge \boldlam^3 c_{d,r}$, fix $k_{\max} \ge 1$, and set $\Delta \ge \frac{\rho(2k_{\max}+2)}{(\rho - 1)^3}$.  Then, for any fixed $k \in [k_{\max}] = \{1,\cdots,k_{\max}\}$, 
\iftoggle{acm}
{
	\begin{multline*}
	\Pr[ \exists k \in [k_{\max}] : d\Phi(\itV_k;\matU e) + \Delta \ge \rho (d\Phi(\itV_{k-1}; \matU e) + \Delta)] \\
	\le\frac{d^2}{\rho - 1}\exp\left\{ \frac{-\boldlam^3(\boldlam - 1)\Delta}{2}  \right\}.
	\end{multline*}
}
{
	\begin{eqnarray*}
	\Pr[ \exists k \in [k_{\max}] : d\Phi(\itV_k;\matU e) + \Delta \ge \rho (d\Phi(\itV_{k-1}; \matU e) + \Delta)] \le \frac{d^2}{\rho - 1}\exp\left\{ \frac{-\boldlam^3(\boldlam - 1)\Delta}{2}  \right\}.
	\end{eqnarray*}
}
\end{lem}
The lemma is established by first fixing a $k \in [k_{\max}]$, ``binning" $d\Phi(\itV_{k-1}; \matU e)$ into at most $\frac{d}{1-\rho}$ intervals $[\tau_{i-1},\tau_i]$, applying Proposition~\ref{prop:rk_k_info_th_prop} and then using union bound over over all $k_{\max} \le d$ time steps. To conclude the proof of Proposition~\ref{prop:CalE_prob}, we invoke a simple covering argument to extend to all $e \in \calS^{r-1}$; details are given in Section~\ref{sec:prop:CalEprob_details}.

%!TEX root = main.tex

\section{Spectrum of Deformed Wigner Model\label{sec:Wig_Spec}}

In this section, we establish that for $\boldlam > 1$, the top $r$ eigenvalues of $\matM = \matW + \boldlam \matU \matU^\top$ concentrate around $\boldlam + \lambdainv$, while the magnitude of the remaining eigenvalues lie below $2 + o_d(1)$. While results of this flavor are standard in the asymptotic regime in which $\boldlam$ and $r$ are held as fixed constants as $d \to \infty$~\cite{capitaine2009largest,benaych2011eigenvalues} our lower bounds require that $d$ can be taken to be polynomial in $r$, $\boldlam$, and $\boldgap^{-1}$. 

\begin{thm}\label{thm:main_spec_thm} There exists a universal constant $C \ge 0$ such that the following holds. Let $\matM = \matW + \boldlam \matU\matU^\top \in \Symd$, and let $\boldgap$ be as in~\eqref{eq:boldgapeq}. Let $\kappa \le 1/2$, $\epsilon \le \boldgap \cdot \min\{\frac12, \frac{1}{\boldlam^2 - 1}\}$, and $\delta >0$. Then for 
\begin{eqnarray}
d \ge  C \left(\frac{(r+\log(1/\delta)) }{ \boldgap \epsilon^2} + (\kappa\boldgap)^{-3} \log(1/\kappa\boldgap) \right),
 \end{eqnarray} 
 the event $\calE_{\matM}$ defined below holds with probability at least $1 - 9\delta$:
 \iftoggle{acm}
 {
	\begin{eqnarray*}
	\calE_{\matM} &:=& \left\{ \|\matW\|_{\op}  \le 2 + \kappa (\lambda + \lambdainv - 2) \right\} \\
	&\bigcap& \left\{[\lambda_{r}(\matM), \lambda_1(\matM)] \subset (\boldlam + \lambdainv)[1-\epsilon,1+\epsilon]\right\}.
	\end{eqnarray*}
}
{
	\begin{eqnarray*}
	\calE_{\matM} := \left\{ \|\matW\|_{\op}  \le 2 + \kappa (\lambda + \lambdainv - 2) \right\} \bigcap \left\{[\lambda_{r}(\matM), \lambda_1(\matM)] \subset (\boldlam + \lambdainv)[1-\epsilon,1+\epsilon]\right\}.
	\end{eqnarray*}
}

Moreover, on $\calE_\matM$, $\lambda_{r}(\matM) - \|\matW\|_{\op} \ge (1 - \frac{\epsilon}{\boldgap})(1 - \kappa) (\boldlam + \lambdainv) \boldgap  \ge (\boldlam + \lambdainv) \boldgap /4$.
\end{thm}

The proof begins with the standard observation that the eigenvalues of $\matM$ are precisely the zeros of the function $z \mapsto \det(zI - \matM ) =\det(zI - \matW + \boldlam \matU\matU^\top )$. In particular, if $z > \lambda_{\max}(\matW)$, then $zI - \matW$ is invertible, and by standard determinant identities, we have
\begin{eqnarray*}
\det(zI - \matM ) &=&\det(zI - \matW + \boldlam \matU\matU^\top)\\
 &=& \det(zI - \matW )\det(I - \boldlam \matU^\top(zI - \matW)^{-1}\matU)~. 
\end{eqnarray*}
In other words, $z > \lambda_{\max}(\matW)$ is in $\spec(\matM)$ if and only if $\det(I - \boldlam \matU^\top(zI - \matW)^{-1}\matU) = 0$. Given $\epsilon, \kappa$ as in Theorem~\ref{thm:main_spec_thm}, we show that for $z^* \le 2 + \kappa (\lambda + \lambdainv - 2) = 2 + o_d(1)$, and for the values
\begin{eqnarray*}
	\alow = (\boldlam + \lambdainv)(1-\epsilon) & \text{and} & \aup  = (\boldlam + \lambdainv)(1-\epsilon) ~,
\end{eqnarray*}
it simultaneously holds with high probability that $\|\matW\|_{\op} \le z^*$ and $z \mapsto \det(I - \boldlam \matU^\top(zI - \matW )^{-1}\matU)$ vanishes for $r$ distinct values of $z \in [\alow,\aup]$. This will imply that at least $r$ of the eigenvalues of $\matM$ lie in $[\alow,\aup]$. Note that, by eigenvalue interlacing, it also follows that the remaining eigenvalues of $\matM$ lie in $[\lambda_{\min}(\matW),\lambda_{\max}(\matW)] \subseteq [-\|\matW\|_{\op},\|\matW\|_{\op}]$. We will proceed by showing that the eigenvalues of the matrix $I_r - \boldlam \matU^\top(z I-\matW)^{-1}\matU$ are all negative when $ z^*< z < \alow$ and are all positive when $ z > \aup$.
This motivates the definition of the events $\calA(z^*) := \{\|\matW\|_\op \le z^* \}$, 
\begin{eqnarray*}
 \Elow(z) &:=& \left\{I_r - \boldlam \matU^\top(z I-\matW)^{-1}\matU \preceq  0 \right\},
 \mbox{ and } \\
 \Eup(z) &:=& \left\{I_r - \boldlam \matU^\top(z I-\matW)^{-1}\matU \succeq 0 \right\}.
\end{eqnarray*}
%\begin{eqnarray}
% \Elow(\alow) := \left\{\matU^\top(\alow I-\matW)^{-1}\matU \succeq \frac{1}{\lambda}I_k\right\}, 
% \mbox{ and } 
% \Eup(\aup) := \left\{\matU^\top(\aup I-\matW)^{-1}\matU \preceq \frac{1}{\lambda}I_k\right\}.
%\end{eqnarray}
Then, using a continuity argument, we derive a useful sufficient condition for $\det(I_r - \boldlam \matU^\top(zI_r - \matW )^{-1}\matU)$ to vanish at $r$ distinct points. 
%The motivation for $\Elow$ and $\Eup$ is as follows: as $a \to z^*$ from the left, then $\matU^\top(aI_d - \matW )^{-1}\matU$ explodes; as $a \to \infty$, $\matU^\top(aI_d - \matW )^{-1}\matU$ tends to zero. Hence, the events $\Elow(\alow)$ and $\Eup(\aup)$ create an ideal region where one of the eigenvalues of $\lambda \matU^\top(\aup I-\matW)^{-1}\matU$ are approximately $1$, and hence we have a hope that $\det(I - \lambda \matU^\top(aI - \matW)^{-1}\matU)$ will vanish $k$-times. 
\begin{prop}\label{Distinct_Eig_Prop} The exists a zero-measure event $\calN$ such that, on $\calN^c \cap \calA(z^*) \cap \Eup(\aup) \cap \Elow(\alow)$, the function $z \mapsto \det(I - \boldlam \matU^\top(zI - \matW )^{-1}\matU)$ vanishes at $r$ distinct points in $[\alow,\aup]$.
\end{prop}
We prove Proposition~\ref{Distinct_Eig_Prop} in Section~\ref{sec:distinct_eig_prop}. We are now left with controlling the probabilities of $\calA(z^*)$, $\Eup(\aup)$ and $\Elow(\alow)$. To control $\calA(z^*)$, we combine a non-asymptotic bound on the spectral norm of a Wigner matrix $\matW$ by Bandeira and van Handel~\cite{bandeira2016sharp} with a standard concentration inequality. Note that asymptotic results of the above statement can be found in references such as~\cite{anderson2010introduction}. Vershynin~\cite{vershynin2016high} gives bounds that are sharp up to constant factors. 
\begin{prop}[Bound on $\|W\|_{\op}$]\label{prop:norm_upper_bound} Let $d \ge 250$, and fix a $p \in (0,1)$. Then,
\iftoggle{acm}
{
	
	\begin{eqnarray}
	\Pr[\|\matW\|_\op > z^*] \le p,
	\end{eqnarray}
	where $z^* = z^*(p) := 2 + 21d^{-1/3}\log^{2/3}(d) + 2\sqrt{\log (1/p)/d}$.
}
{
	\begin{eqnarray}
	\Pr[\|\matW\|_\op > z^*] \le p \quad \text{ where } z^* = z^*(p) := 2 + 21d^{-1/3}\log^{2/3}(d) + 2\sqrt{\log (1/p)/d}
	\end{eqnarray}
}

\end{prop}
The above proposition is proved in Section~\ref{sec:norm_upper_bound_proof}. We must now control the probabilities of $\Elow$ and $\Eup$. 
Since $\matU$ is uniform on $\Stief(d,r)$ and independent of $\matW$, we expect by concentration that
\begin{eqnarray*}
\matU^\top(z I_d-\matW)^{-1}\matU &\approx& \Exp_{\matU}[\matU^\top(z I_d-\matW)^{-1}\matU] \\
&=& \frac{1}{d}\tr\big((zI_d - \matW)^{-1}\big) \cdot I_r = S_{\matW}(z) \cdot I_r~,
\end{eqnarray*}
where $ S_{\matW}(z) := \frac{1}{d}\tr\big((zI_d - \matW)^{-1}\big)$ the Stieltjes transform of the empirical spectral measure of the Wigner matrix $\matW$. As $d \to \infty$, it is well known~\cite{anderson2010introduction} that for all $z > 2$,
\iftoggle{acm}
{
	\begin{align}\label{eq:stieljes_convergence}
	 S_{\matW}(z) &:=\frac{1}{d}\tr\big((zI_d - \matW)^{-1}\big) ~\overset{prob.}{\longrightarrow}~ \stielt(z)\\
	 \text{where }\stielt(z) &:= \frac{z - \sqrt{z^2 - 4}}{2}\nonumber~.
	\end{align}
}
{
	\begin{eqnarray}\label{eq:stieljes_convergence}
	 S_{\matW}(z) &:=\frac{1}{d}\tr\big((zI_d - \matW)^{-1}\big) ~\overset{prob.}{\longrightarrow}~ \stielt(z),\quad\text{where }\stielt(z) := \frac{z - \sqrt{z^2 - 4}}{2}\nonumber~.
	\end{eqnarray}
}

Therefore we see that $\|\boldlam \matU^\top(zI - \matW )^{-1}\matU - \boldlam \stielt(z) I_r\|_{\op} = o_d(1)$. Finally, we see that the equation $\stielt(z) = \lambdainv$ is solved by $z = \lambda + \lambdainv$, and that $\stielt(\alow) > \lambdainv$ and $\stielt(\aup) < \lambdainv$. Thus, our goal will be to verify that, on $\calA(z^*)$, the following holds for $z \in \{\alow,\aup\}$ with high probability:
\iftoggle{acm}
{
	\begin{multline}\label{eqqq:eig:wts}
	\left|S_{\matW}(z) - \stielt(z)\right| + \|S_{\matW}(z)I_r - \matU^\top(z I_d-\matW)^{-1}\matU\|_{\op} \le |\stielt(z) - \lambdainv|~.
	\end{multline}
}
{
	\begin{eqnarray}\label{eqqq:eig:wts}
	\left|S_{\matW}(z) - \stielt(z)\right| + \|S_{\matW}(z)I_r - \matU^\top(z I_d-\matW)^{-1}\matU\|_{\op} \le |\stielt(z) - \lambdainv|~.
	\end{eqnarray}
}
Indeed, we see that for $z = \alow < \lambda + \lambdainv$, the above equation implies $\Elow$ by the triangle inequality, and similarly if $z = \aup > \lambda + \lambdainv$. To handle the error $\|S_{\matW}(z)I_r - \matU^\top(z I_d-\matW)^{-1}\matU\|_{\op}$, we fix $\matW$ and reason about the above quadratic form in $\matU$ using the following Hanson-Wright style inequality proved in Section~\ref{sec:hanson_wright_proof}:

\begin{prop}[Stieffel Hanson-Wright ]\label{prop:Hanson_wright_proposition} Let $A\in \R^{d \times d}$ be any fixed symmetric matrix, and let $\matU \unifsim \Stief(d,r)$ be uniform on the Stieffel manifold. Then for all $t \le d/4$,
\begin{eqnarray*}
\Pr\left[\left\|\matU^{\top}A\matU - \frac{\mathrm{tr}(A)}{d}\cdot I_r\right\|_{\op} > \frac{8\left(t^{1/2}\|A\|_{\F} + t\|A\|_{\op}\right)}{d(1 - 2\sqrt{t/d})} \right]  \le 3e^{-t+2.2r}
\end{eqnarray*}
\end{prop}
In particular, if we choose $A = (z I_d-\matW)^{-1}$ and condition on $\calA(z^*)$, we can bound $\|A\|_{\op} \le (z - z^*)^{-1}$, $\|A\|_{\F} \le \sqrt{d}(z-z^*)^{-1}$, and hence conclude that $\matU^\top(z I_d-\matW)^{-1}\matU = S_{\matW}(z) \cdot I_r + o_d(1)$. Next, the term $\left|S_{\matW}(z) - \stielt(z)\right|$ is upper bounded by Theorem~\ref{thm:stiel_thm} (proved below), which establishes a finite sample version of Equation~\eqref{eq:stieljes_convergence}.
\begin{thm}[Stieltjes transform]\label{thm:stiel_thm} Fix $p,\delta \in( 0,1)$ and let $z^*$ given in Proposition~\ref{prop:norm_upper_bound}. Fix an $a \in (2 + \frac{1}{31}(z^*-2),d)$, and assume that $\boldeps := (d(a-z^*)^2)^{-1/2}$ satisfies $\boldeps^2 < \min\{\frac{1}{16\sqrt{2}},\frac{a-2}{32}\}$, and $p^{1/3} < \boldeps/8$. Then there exists an event $\calE_S(a)$ with $\Pr[\calE_S(a)^c] \le 1 - \delta$ and on $\calE_S(a) \cap \calA(z^*)$,
\iftoggle{acm}
{
	\begin{align*}
	&\left|S_\matW(a) - \stielt(a)\right| \le  c_{\delta}\boldeps^2 + 8d^{3/2} p^{1/6} \\
	&\text{where } c_{\delta} := 4\sqrt{2} + 2\sqrt{\log(2/\delta)}.
	\end{align*}
}
{
	\begin{eqnarray*}
	\left|S_\matW(a) - \stielt(a)\right| \le  c_{\delta}\boldeps^2 + 8d^{3/2} p^{1/6} , \text{ where } c_{\delta} := 4\sqrt{2} + 2\sqrt{\log(2/\delta)}.
	\end{eqnarray*}
}
\end{thm}
Finally, Lemma~\ref{lem:stielt_lem} in the Appendix establishes a lower bound on $|\stielt(a) - \lambdainv|$ (note that this is deterministic). In Section~\ref{sec:main_spec_thm_proof}, we put the pieces together to show for our choice $\epsilon$, $\kappa$, and an appropriate $z^*$, Theorem~\ref{thm:stiel_thm} and Proposition~\ref{prop:Hanson_wright_proposition} imply that Equation~\eqref{eqqq:eig:wts} holds with high probability.

%!TEX root = main.tex

\subsection{Proof Roadmap for Theorem~\ref{thm:stiel_thm}}
	Here we prove Theorem~\ref{thm:stiel_thm} by establishing estimates of $S_\matW(a) := \frac{1}{d}\tr(aI - \matW) = \frac{1}{d}\sum_{i=1}^d \frac{1}{a - \lambda_i(\matW)}$ under the event $\calA(z^*) := \{\|\matW\| \le z^*\}$, for some $a$ bounded away from $z^*$. To do so, we shall need to overcome multiple technical roadblocks, and so we devote this subsection to give a roadmap. Our first challenge is that, even those $S_{\matW}(a)$ will end up concentrating around $\stielt(a)$, its expectation diverges for any $a \in \R$. Indeed, the eigenvalues of $\matW$ are distinct with probability $1$, and their marginals have a positive density with respect to the Lebesque measure, and so integrating the summand $\frac{1}{d(a - \lambda_{\max}(W))}$ in the neighborhood of $a$ will cause the expectation to diverge. 

	Luckily, the probability that $\lambda_{\max}(W)$ is close to $a$ is vanishly small in $d$, and so we will still be able to establish concentration by estimating $S_\matW(z)$, where $z = a+ b\frak i$ and $b > 0$ is very samll relative to $a$. This ensures that $\Exp[S_{\matW}(z)]$ will converge, and in fact we wil be able to both compute the latter expectation and show that $S_\matW(z)$ concentrates around it. Before establishing with these acts, we show that if $b$ is sufficiently small and $a$ is not to close to $z^*$, then $S_{\matW}(z) \approx S_\matW(a)$.
	\begin{lem}\label{lem:imag_closeness} On $\calA(z^*)$, $\left|\Re(S_{\matW}(a+ib)) - S_{\matW}(a)\right| \le \frac{b^2}{(a-z^*)^3}$ for any $a > z^*$. 
	\end{lem}
	The proof of the above lemma is deferred to Appendix~\ref{sec:imag_closeness_proof}. Our first step to control $S_\matW(z)$ is estimating its expectation: 
	\begin{prop}\label{prop:exp_comp} Define the determininstic quantity 
	\begin{eqnarray*}
	\err(z) &:=& \Exp[S_\matW(z)^2] + \frac{1}{d}\Exp[\tr(zI - \matX)^2] - \Exp[S_\matW(z)]^2~.
	\end{eqnarray*} Then, as long as $a^2 - 4 > b^2 + 4|\Re(\err)|$ and $b > |\Im(\err(z))|$, one has 
	\begin{eqnarray*}
	\left|\Re(\Exp[S_\matW(z)]) - \stielt(a)\right| \le  \sqrt{ |b^2 + 4\Re(\err)| + |(2ab + \Im(\err)|}~.
	\end{eqnarray*}
	\end{prop}
	The proposition, proved in Appendix~\ref{sec:exp_comp}, follows the standard arguments given in Section 2.4 of \cite{anderson2010introduction}. At a high level, we show that $\Exp[S_\matW(z)]$ satisfies a quadratic equation whose roots are approximately $\frac{a \pm \sqrt{a^2 -4 }}{2}$. We need to take care that we choose the correct root, which imposes the conditions $a^2 - 4 > b^2 + 4|\Re(\err)|$ and $|b > |\Im(\err(z))|$. In particular, the requirement $b > |\Im(\err(z))|$ will force us to take special care to show that $\Im(\err(z))$ is dominated by $b$.

	The remaining part of the proof requires us to establish two results: first, that $\Re(\err(z))$ and $\Im(\err(z))$ are sufficiently small, and second, that $S_\matW(z)$ concentrates around its expectation. Since $\matW$ has Gaussian entries, one may be tempted to argue both result by using the Lipschitz property of the map $\matW \mapsto S_\matW(z)$. Unfortunately, the Lipschitz constant of this map scales with $1/b$, which will become quite large if we take $b$ to be too small.

	Instead, we define a modified matrix $\widetilde{\matW}$ such that, $\widetilde{\matW} = \matW$ on $\calA(z^*)$, and the composition of maps $\matW \mapsto \widetilde{\matW} \mapsto S_{\widetilde{\matW}}(z)$ has a suitably large Lipschitz constant, even when $b$ is vanishly small function in $d$ (e.g. $d^{-10}$). Specifically, given the eigendecomposition $\matW = \matO\mathbf{\Lambda} \matO^\top$, define the matrix $\widetilde{\matW} := \matO^\top \diag(\min\{\mathbf{\Lambda}_{ii},z^*\}) \matO$ to be the matrix obtained by truncating the eigenvalues of $\matW$ to lie in $(-\infty,z^*]$. Observe that under $\calA(z^*)$, one has $\widetilde{\matW} = \matW$. Moreover, $\lambda_1(\widetilde{\mathbf{W}}) \le z^*$ almost surely.

	This latter observation is crucial in establilshing the following lemma, which allows us to approximate the real and imaginary parts of $\err(z)$ - a quantity defined in terms of the raw Wigner matrix $\matW$ - by variance-like quanities involving the Stieltjes transform of the modified matrix $\matWtil:$

	\begin{lem}\label{lemma:tildeW_err_bound} Suppose that $0 < b < \min\{1,a - z^*\}$, and define $p_{z^*} := \Pr[\calA(z^*)]$. Then,
	\begin{equation*}
	\begin{aligned}
	\left|\Re(\err(z)) - \Re(\Exp[(S_{\matWtil}(z) - \Exp[S_{\matWtil}(z)])^2)\right| &\le \frac{1}{d(a-z^*)^2} + \frac{4p_{z^*}}{b^2}\\
	\left|\Im(\err(z)) - \Im((\Exp[(S_{\matWtil}(z) - \Exp[S_{\matWtil}(z)])^2\right| &\le \frac{2b}{d(a-z^*)^2} + \frac{4p_{z^*}}{b^2}
	\end{aligned}
	\end{equation*}
	Moreover, on $\mathcal{A}(z^*)$, 
	\begin{eqnarray*}\label{eq:expectation_distance}
	|\Re(S_\matW(z)) - \Re(\Exp[S_\matW(z)])| \le |\Re(S_{\matWtil}(z)) - \Re(\Exp[S_{\matWtil}(z)])| + \frac{p_{z^*}}{b} 
	\end{eqnarray*}
	\end{lem}
	This leaves us with the two tasks before we can finally apply Proposition~\ref{prop:exp_comp} to get a high probability bound on $S_{\matW}(z)$. 
	First, we need to control $\err(z)$ by bounding the size of the variance-like terms $\Re(\Exp[(S_{\matWtil}(z) - \Exp[S_{\matWtil}(z)])^2)$ and $\Im(\Exp[(S_{\matWtil}(z) - \Exp[S_{\matWtil}(z)])^2)$. 
	Secondly, we shall need to argue that $S_{\widetilde{\matW}}(z)$ concentrates around $\Exp[S_{\matWtil}(z)]$. Both tasks amount to controlling the deviations of $S_{\matWtil}(z)$, which we can achieve by leveraging the fact that $S_{\matWtil}(z)$ is a Lipschitz function of the underlying standard gaussian matrix $\matX$ (recall that $\matX_{ij} \overset{i.i.d.}{\sim} \calN(0,1)$:, and that $\matW = \frac{1}{\sqrt{2d}}(\matX+\matX^\top)$.)
	\begin{lem}\label{lem:Lipschitz Computation}
	Let $z = a + b\fraki$, and define the map $\Psi: \matX \to S_{\matWtil}(z)$. Then if $a - z^* > |b|$,
	\begin{eqnarray*}
	\Lip(\Re(\Psi))^2 \le \frac{\sqrt{2}}{d^2(a-z^*)^4} & \text{and} & \Lip(\Im(\Psi))^2 \le \frac{4\sqrt{2}b^2}{d^2(a-z^*)^6} ~.
	\end{eqnarray*}
	\end{lem}
	To control the variance terms, we use the Gaussian Poincare inequality, which states that if $f: \R^{D} \to \R$ is an $L$ Lipschitz function, and let $\matx \in \R^{D}$ is a standard gaussian vector, then $\Var[f(\matx)] \le L^2$. This lets us control $\Re(\err(z))$ and $\Im(\err(z))$
	\begin{lem}\label{lem:err_control}
	Suppose that $b \le (a-z^*)/2$ and $d \ge (a-z^*)^{-2}$. Then,
	\begin{eqnarray*}
	|\Im(\err(z))| &\le&  \frac{8\sqrt{2} b}{d}\cdot \max\left\{\frac{1}{(a-z^*)^{2}},\frac{1}{d(a-z^*)^{5}}\right\} + \frac{4p_{z^*}}{b^2} \\
	|\Re(\err(z))| &\le& 4 \left( \frac{1}{d(a-z^*)^2} + p_{z^*}/b^2\right)
	\end{eqnarray*} 
	\end{lem}
	Appendix~\ref{sec:stiel_thm_proof}, we finally put together all these pieces to prove Theorem~\ref{thm:stiel_thm}. Proofs of all the supporting claims can be found in the following subsections of Appendix~\ref{sec:Stieltjes_sec}.

\bibliographystyle{plain}
\bibliography{PCA}
\clearpage

\appendix

\tableofcontents
\break

\part{Supporting Material for the Information-Theoretic Lower Bound}

%!TEX root = main.tex

\section{Proof of Theorem~\ref{thm:main_tech_thm} and Further Results \label{sec:further_results}}
In this section, we prove Theorem~\ref{thm:main_tech_thm}, and state additional results which follow as easy consequences of our framework.
\subsection{Batch-Queries and Improved $\boldgap$ dependence for $r = 1$}
We will begin by presenting an improved $\boldgap$ dependence in the $r = 1$ case. We shall actually prove a lower bound in a more general setting, where the actual is allowed to make $\itT$-adaptive rounds of batches of $B$ queries. When $d$ is a sufficiently large polynomial in $\gap$ and $B$, we show that we still need $\itT \gtrsim \sqrt{\boldgap}$
\begin{thm}\label{thm:main_tech_thm_rk1} Fix a $\boldgap \in (0,1)$ and any $d \ge \frakq(\boldgap^{-1},(1-\boldgap)^{-1},2,1,\log(2))$ where $\frakq$ is as in Proposition~\ref{prop:egood_rkr}, and  let $\boldlam = \frac{1 + \sqrt{\boldgap(1-\boldgap)}}{1-\boldgap}$ be the solution to Equation~\ref{eq:boldgapeq}. Then $\matM = \matW + \lambda \matU\matU^\top$ where $\matU \unifsim \Stief(d,r)$ and $\matW \sim \GOE(d)$,  Then for any $\Alg$ make $\itT$ adaptive rounds of batches of $B$ queries,
\begin{eqnarray}\label{eq:rk_one_result}
\Pr[\langle \vhat, \matMabs \vhat \rangle \ge \frac{1}{6}\boldgap] \le e\cdot\exp\left( - \frac{d \boldgap^{5/2} }{128B}\cdot \left(\gapexp\right)^{-4(\itT+1)}\right)
\end{eqnarray}
\end{thm}
The above theorem makes use of the following generalization of Theorem~\ref{thm:rank_r_thm_est}, which is sketeched in Section~\ref{sec:batch_est_proof}:
\begin{thm}\label{thm:rank_one_thm_est_B}  Let $\matM = \matW + \boldlam \matU \matU^\top$, where $\matW \sim \GOE(d)$, and $\matu \unifsim \sphere$. Suppose $\Alg$ is allowed to make $B$ queries per round adaptivity. Then for all $\delta \in (0,1/e)$,
\begin{eqnarray*}
\Pr_{\matW,\matu,\Alg}\left[\exists k \ge 1:  \matu^\top \itV_{Bk} \itV_{Bk}^\top \matu \ge \boldlam^{4k}  \cdot \frac{32 B\boldgap^{-1}(\log \delta^{-1} + \boldgap^{-1/2})}{d}\right] \le  \delta
\end{eqnarray*}
\end{thm}
\subsection{Sharp Lower Bounds in the Large-$\boldgap$ regime}
When the eigengap is bounded away from zero, the complexity of PCA is better paramterized in terms of the eigenration $1 - \gap_r(\matM) = \lambda_{r+1}(\matM)/ \lambda_{r}(\matM)$. Indeed, one can show in this regime that both Lanczos and the power methods converge at a rate of $\log_{(1 - \gap_r(\matM))^{-1}}(d) = \log d / -\log( 1- \gap_r(\matM))$. More over, for any $c \in (0,1)$, we see that $-\log( 1- \gap_r(\matM))$ and $- \log ( \frac{1}{c}(1 - \gap_r(\matM))$ blow up as $\gap_r(\matM) \to 1$, bu  $-\log( 1- c\gap_r(\matM))$ stays bounded.

In the large gap case, we can simplify use the crude bounds $\sigma_r(\matM) = \lambda_{r}(\matM) \ge \lambda $ and $\sigma_{r+1}(\matM) \le \|\matW\|_{\op}$, which are consequences of eigenvalue interlacing. In particular, when $d$ is a sufficiently large constant, our norm bound on $\|\matW\|_{\op}$ in Proposition~\ref{prop:norm_upper_bound} implies that for all $d$ large enough that $21d^{-1/3}\log^{2/3}d \le 1/2$, then with probabilty at least $1 - e^{-d^{1/2}/16}$, the event $ \{\|\matW\|_{\op} \le 3\}$ hholds. Hence, for $\boldlam \ge 6$, we see that 
\begin{eqnarray*}
\{\|\matW\|_{\op} \le 3\}\supset \left\{\lambda_1(\matM) \le 3 + \boldlam \le \frac{3\boldlam}{2}, \quad \lambda_r(\matM) - \|\matW\| \ge \boldlam/2, \quad 1 - \gap_r(\matM) \le 2/\boldlam\right\}
\end{eqnarray*}
Noting that in this regime, we have that $\boldgap \ge \frac{(6-1)^2}{6^2 + 1} = \frac{25}{37} \ge 1/2$. Hence, replacing $\boldgap$ by this lower bound, and  replacing $\frac{1 + \sqrt{\boldgap(1-\boldgap)}}{1-\boldgap}$ with $\boldlam$, we can state the following ``big-gap'' analogoue of Theorem~\ref{thm:main_tech_thm}

\begin{thm}\label{thm:main_tech_thm_biggap} Let $d$ be sufficiently large that $\Pr[\|\matW\|_{\op} \le 3] \ge 1/2$, and fix $\boldlam \ge 6$. Then if $\matM = \matW + \lambda \matU\matU^\top$ where $\matU \unifsim \Stief(d,r)$ and $\matW \sim \GOE(d)$, then for any $\Alg$ satisfying Definition~\ref{def:Query_def}, we have
\begin{multline}
\Exp_{\matM}\left[\Prit_{\Alg}\left[ \langle \Vhat, \matMabs \Vhat \rangle \ge\frac{23}{24}\sum_{\ell =1}^{r} \sigma_{\ell}(\matM)  \right] \big{|} \{\|\matW\|_{\op} \le 3\} \right] \le \\
2e \cdot\exp\left( - \frac{d}{8 \cdot 78  \log(d)} \cdot \boldlam^{-18(\frac{\itT}{r}+2)} \right) 
\end{multline}
Where $\Prit_{\Alg}$ is the probability taken with respect to the randomness of the algorithm. Moreover, on $\{\|\matW\|_{\op} \le 3$, we have $1 - \gap_r(\matM) \le 2/\boldlam$. 
\end{thm}

\subsection{Proof of Theorems~\ref{thm:main_tech_thm} and~\ref{thm:main_tech_thm_rk1}
\label{sec:main_tech_thm_proof}}
For $r' = \ceil{r/2}$, we have via Lemma~\ref{lem:estimation_reduction_good} and Observation~\ref{obs:second_obs} that 
\begin{eqnarray*}
\Pr[\langle \Vhat, \matMabs \Vhat \rangle \ge \left(1 - \frac{\boldgap}{12}\right)\sum_{\ell =1}^{r} \sigma_{\ell}(\matM)  | \Egood] \le \frac{1}{\Pr[\Egood]} \cdot \Pr_{\matW,\matU,\Alg}\left[\lambda_{r'}(\matU^\top \itV_{\itT+r}\itV_{\itT + r} \matU) \ge \boldgap/4\right] 
\end{eqnarray*}
We set $ \frac{26 r \boldlam^{9k/r'} \log(20d^{2})}{d\boldgap^{2}}\log(e\delta^{-1}) = \boldgap/4$. Setting $k = \itT + r$, we have for $d \ge 10$ that we can take
\begin{eqnarray*}
\log(e\delta^{-1}) = \frac{d}{12 \cdot 26 \log(20d^2)\boldgap^{3}} \boldlam^{-18(\itT/r+2)} \le \frac{d}{78  \log(d)\boldgap^{3}} \boldlam^{-18(\itT/r+2)}
\end{eqnarray*}
Using the formula $\boldlam = \gapexp$ concludes. In the rank one case, we can improve the dependence on the gap. Indeed,  setting $\boldlam^{4(\itT+1)}  \cdot \frac{32 B\boldgap^{-3/2}(\log e\delta^{-1})}{d} = \boldgap/4$ and solving for $\delta$, basic manipulations reveal   
\begin{eqnarray*}\label{eq:rk_one_result}
\Pr[\langle \vhat, \matMabs \vhat \rangle \ge \frac{1}{6}\boldgap] \le e\cdot\exp\left( - \frac{d \boldgap^{5/2} }{128B}\cdot \boldlam^{-4(\itT+1)}\right)
\end{eqnarray*}

\subsection{Proof of Theorem~\ref{thm:rank_one_thm_est_B}\label{sec:batch_est_proof}}
When we can make batches of $B$ queries, we consider the matrices $\itV_{Bk} \in \Stief(d,Bk)$ associated with the $k$-th round of adaptivity , and note $\Alg$ must decide on its $B$ queries $\itv^{(Bk+1)},\dots,\itv^{B(k+1)}$ at the end of round $k$.The analogoue of Proposition~\ref{prop:recur_prop_rkone} becomes that, for $\tau_{k+1} \ge \sqrt{2B(k+1)}$, 
\begin{multline}\label{eq:recur_bound_rkone_B}
\Exp_{\matu}\Prit_{\matu}[\{ \Phi(\itV_{Bk};\matu) \le \frac{\tau_k}{d}\} \cap \{ \Phi(\itV_{B(k+1)};\matu)> \frac{\tau_{k+1}}{d}\} ] \le \\
\exp\left\{ \frac{\eta}{2}  \left((1+\eta)\boldlam^2 \tau_k  - \left(\sqrt{\tau_{k+1}}-\sqrt{2B(k+1)}\right)^2\right) \right\}~,
\end{multline}
where the only change was that we had to inflate the entropy term from an upper bound on $\sup_{V \in \Stief(d,k+1)} \Pr_{u \sim \calD}[\Phi(V;u)> \tau_{k+1}/d]$ to an upper bound on $\sup_{V \in \Stief(d,B(k+1))} \Pr_{u \sim \calD}[\Phi(V;u)> \tau_{k+1}/d]$, which forces us to replace $\sqrt{2k + 2}$ with $\sqrt{2B(k+1)}$. We now examine Appendix~\ref{sec:rk_one_thm_proof}, which  proves Theorem~\ref{thm:rank_one_thm_est} from Proposition~\ref{prop:recur_prop_rkone}. There it is shown that, if $\tau_0(\delta) = 32\boldgap^{-1}(\log \delta^{-1} + \boldgap^{-1/2})$, then the sequence $\tau_k := \lambda^{4k} \tau_0(\delta)$ satisfies
	\begin{eqnarray*}
	\Pr_{\matu \sim \calS^{d-1}}\Prit_{\matu}\left[\exists k \ge 1:  d\matu^\top \itV_k \itV_k^\top \matu \ge \tau_k \right] \le  \delta \quad (B = 1)~.
	\end{eqnarray*}
	Following the algebra in that section, one can check that if we define the sequence $\widetilde{\tau}_k = B \tau_k$, and  replace Proposition~\ref{prop:recur_prop_rkone} with Equation~\eqref{eq:recur_bound_rkone_B}, we get 
	\begin{eqnarray*}
	\Pr_{\matu \sim \calS^{d-1}}\Prit_{\matu}\left[\exists k \ge 1:  d\matu^\top \itV_{Bk} \itV_{Bk}^\top \matu \ge \widetilde{\tau}_k \right] \le  \delta \quad (B \ge 1) ~.
	\end{eqnarray*}
	Concluding then follows from pulling in the definition$\widetilde{\tau}_k = B \tau_k = B \lambda^{4k} \tau_0 =  32B \lambda^{4k}\boldgap^{-1}(\log \delta^{-1} + \boldgap^{-1/2})$, and rearranging.

\begin{comment}
\subsection{Proof of Proposition~\ref{truncChi_detec}}
	For ease of notation, let $\overline{\Prit}_A =  \overline{\Prit}[\cdot;\{A_{\theta}\}]$. Note that $\overline{\Prit} - \overline{\Prit}_A  \ge 0$, which impies that
	\begin{eqnarray}
	\int |\rmd\overline{\Prit} - \rmd\overline{\Prit}_A| = \int \rmd\overline{\Prit} - \rmd\overline{\Prit}_A = 1 - \overline{\Prit}_A(\calX) := 1-p,
	\end{eqnarray}
	so by the triangle inequality
	\begin{eqnarray*}
	&& \|\Qit - \overline{\Prit}\|_{TV} =\frac{1}{2}\int |\rmd\Qit(x) - \rmd\overline{\Prit}(x)|\\
	&\le& \frac{1}{2}\int |\rmd\Qit(x) - \rmd\Prit_A(x)| + \frac{1}{2} \int |\rmd\overline{\Prit} - \rmd\overline{\Prit}_A| = \frac{1}{2}\int |\rmd \Qit(x) - \rmd\Prit_A(x)| + \frac{1-p}{2},
	\end{eqnarray*}
	Next, since $\Qit$ is a probability measure,
	\begin{eqnarray*}
	\int |\rmd\Qit(x) - \rmd\Prit_A(x)| &=& \Exp_{\Qit} |\frac{\rmd\overline{\Prit}_A}{\rmd\Qit} - 1| \le \sqrt{\Exp_{\Qit} |\frac{\rmd\overline{\Prit}_A}{\rmd\Qit} - 1|^2}\\
	&=& \sqrt{\Exp_{\Qit} |\frac{\rmd\overline{\Prit}_A}{\rmd\Qit}|^2  + 1 - 2\overline{\Prit}_{A}(\calX)} = \sqrt{\Exp_{\Qit} |\frac{d\overline{\Pr}_A}{\rmd\Q}|^2  + 1 - 2p}\\
	&=& \sqrt{\Exp_{\Qit} |\frac{\rmd\overline{\Prit}_A}{\rmd\Qit}|^2  - 1 + 2(1-p) } \le  \sqrt{\Exp_{\Qit} |\frac{\rmd\overline{\Prit}_A}{\rmd\Qit}|^2  - 1} + \sqrt{2(1-p)}.
	\end{eqnarray*}
	Putting pieces together yields the proof.
\end{comment}

\section{Reduction from Optimization to Estimation of $\matu$}
%!TEX root = main.tex

\subsection{Proof of Lemma~\ref{lem:estimation_reduction_good} \label{sec:estimation_red_proof}}
We begin with a technical lemma that holds for a general $\matMabs$, and then proceed to simplify using the definition of $\Egood$:
\begin{lem}\label{lem:estimation_reduction} Suppose that $\matU \in \Stief(d,r)$, and set $\gaptil_r(\matMabs) := \lambda_r(\matMabs)  - \|\matW\|_{\op} > 0$. Then for any $\Vhat \in \Stief(d,r)$, one has 
\begin{eqnarray*}
\langle \Vhat,\matMabs \Vhat \rangle  &\ge& \left(1 - \frac{r'\gaptil_r(\matMabs) }{2r\lambda_1(\matMabs)}\right) \cdot \sum_{i=1}^r \lambda_i(\matMabs) \quad \text{implies} \\
\sigma_{r + 1- r'}(\Vhat^\top \matU\matU^\top \Vhat) &\ge& \frac{\gaptil_r(\matMabs)}{2\boldlam} 
\end{eqnarray*}
\end{lem}
The above lemma is proved in the subsection below. To conclude, note that $\Egood(\gamma)$, we have that $\lambda_{\ell}(\matMabs) = \sigma_\ell(\matMabs) = \lambda_\ell(\matM)$ for $\ell \in [r]$, Moreover,  $\gaptil_r(\matMabs) := \lambda_r(\matM)  - \|\matW\|_{\op} \ge \frac{1}{2}(\boldlam + \lambdainv - 2)$, and $\lambda_1(\matMabs) = \lambda_1(\matM) \le \frac{3}{2}(\boldlam + \lambdainv)$. Hence, by the above lemma, we have that as long as
\begin{eqnarray*}
\langle \Vhat,\matMabs \Vhat \rangle \ge \left(1 - \frac{r' \boldgap}{6r}\right) \cdot \sum_{i=1}^r \lambda_i(\matMabs) \ge \left(1 - \frac{r'\gaptil_r(\matMabs) }{2r\lambda_1(\matMabs)}\right) \cdot \sum_{i=1}^r \lambda_i(\matMabs)
\end{eqnarray*}
Then,
\begin{eqnarray*}
\lambda_{r + 1- r'}(\Vhat^\top \matU\matU^\top \Vhat) &\ge& \frac{\gaptil_r(\matMabs)}{2\boldlam} \ge \boldgap/4
\end{eqnarray*}
Finally, we change the variables via $r' \leftarrow r + 1 - r'$. 

\subsection{Proof of Lemma~\ref{lem:estimation_reduction}\label{sec:est_lem_proof}}
Recall the definition 
\begin{eqnarray*}
\Phi(\itV;\matU) := \langle \Vhat, \matU \matU^\top \Vhat \rangle = \tr{\Vhat^\top \matU \matU^\top \Vhat}
\end{eqnarray*}
Given $r' \in [r]$, define the matrix $\itVtil \in \Stief(d,r')$ by
\begin{eqnarray*}
\itVtil := \Vhat \cdot \itOtil, \text{ where } \itOtil := \arg\inf \{ \Phi(\Vhat \cdot \itO;\matU) : \itO \in \Stief(r',r')\}
\end{eqnarray*}
Note then that $\itOtil$ corresponds to the eigenspace of the bottom $r'$ eigenvectors of the matrix $\Vhat^\top \matU \matU^\top \Vhat \succeq 0$ , and thus, we see that
\begin{eqnarray*}
\Phi(\itVtil;\matU) = \tr( \itOtil^\top\Vhat^\top \matU \matU^\top \Vhat\itOtil ) \le r'\lambda_{r + 1 - r'}(\Vhat^\top \matU \matU^\top \Vhat)
\end{eqnarray*}
We will now establish a lower bound on $\Phi(\itVtil;\matU)$. First, we observe that there exists an $\itVbar \in \Stief(d,r-r')$ such that $\Vhat\Vhat^\top = \itVtil\itVtil^\top + \itVbar\itVbar^\top$. Since $\matMabs \preceq \matWabs + \matU \matU^\top$, we have
\begin{eqnarray*}
\langle \Vhat, \matMabs \Vhat \rangle &=& \tr(\matMabs \Vhat\Vhat^\top)\\
&=&\tr(\matMabs\itVbar\itVbar^\top ) +  \tr(\matMabs\itVtil\itVtil^\top) + \\
&\le& \sum_{i=1}^{r - r'} \lambda_i(\matMabs) + \tr(\matMabs\itVtil\itVtil^\top ) \\
&\le& \sum_{i=1}^{r - r'} \lambda_i(\matMabs) + \tr(\matWabs\itVtil\itVtil^\top )  + \tr(\boldlam \matU \matU^\top \itVtil\itVtil^\top )\\
&\le& \sum_{i=1}^{r - r'} \lambda_i(\matMabs) + r'\|\matW\|_{\op}  + \boldlam  \Phi(\itVtil;\matU)\\
\end{eqnarray*}
In particular, if $\langle \Vhat, \matMabs \Vhat \rangle \ge (1 - \eta)\sum_{i=1}^{r} \lambda_i(\matMabs) $, then we must have that
\begin{eqnarray*}
\boldlam  \Phi(\itVtil;\matU) &\ge& (1 - \eta)\sum_{i=1}^{r} \lambda_i(\matMabs) - \sum_{i=1}^{r - r'} \lambda_i(\matMabs) + r'\|\matW\|_{\op}\\
&\ge& - \eta\sum_{i=1}^{r} \lambda_i(\matMabs) + \sum_{r=r-r'+1}^r \{\lambda_i(\matMabs)  - \|\matW\|_{\op}\}\\
&\ge& - \eta\sum_{i=1}^{r} \lambda_i(\matMabs) + r' \{\lambda_r(\matMabs)  - \|\matW\|_{\op}\}\\
&\ge& - \eta r\lambda_1(\matMabs) + r'  \gaptil_r(\matMabs)\\
&=& r' \cdot\left(\gap_r(\matMabs) -\eta  \cdot \frac{r\lambda_1(\matMabs)}{r'}\right)
\end{eqnarray*}
where $\gaptil_r(\matMabs) := \lambda_r(\matMabs)  - \|\matW\|_{\op}$. In particular, if we select $\eta = \gaptil_r(\matMabs)\cdot \frac{r'}{2r\lambda_1(\matMabs)}$, then we have that
\begin{eqnarray*}
 r' \boldlam \lambda_{r - r' + 1}(\Vhat^\top \matU \matU^\top \Vhat))  &\ge& r' \cdot\left(\gaptil_r(\matMabs) -\eta  \cdot \frac{r\lambda_1(\matMabs)}{r'}\right)\\
 &=& r' \cdot\left(\gaptil_r(\matMabs)\right)/2
\end{eqnarray*}
Rearranging proves the lemma.

\subsection{Proof of Lemma~\ref{lem:estimation_reduction}\label{sec:est_lem_proof}}
Recall the definition 
\begin{eqnarray*}
\Phi(\itV;\matU) := \langle \Vhat, \matU \matU^\top \Vhat \rangle = \tr{\Vhat^\top \matU \matU^\top \Vhat}
\end{eqnarray*}
Given $r' \in [r]$, define the matrix $\itVtil \in \Stief(d,r')$ by
\begin{eqnarray*}
\itVtil := \Vhat \cdot \itOtil, \text{ where } \itOtil := \arg\inf \{ \Phi(\Vhat \cdot \itO;\matU) : \itO \in \Stief(r',r')\}
\end{eqnarray*}
Note then that $\itOtil$ corresponds to the eigenspace of the bottom $r'$ eigenvectors of the matrix $\Vhat^\top \matU \matU^\top \Vhat \succeq 0$ , and thus, we see that
\begin{eqnarray*}
\Phi(\itVtil;\matU) = \tr( \itOtil^\top\Vhat^\top \matU \matU^\top \Vhat\itOtil ) \le r'\lambda_{r + 1 - r'}(\Vhat^\top \matU \matU^\top \Vhat)
\end{eqnarray*}
We will now establish a lower bound on $\Phi(\itVtil;\matU)$. First, we observe that there exists an $\itVbar \in \Stief(d,r-r')$ such that $\Vhat\Vhat^\top = \itVtil\itVtil^\top + \itVbar\itVbar^\top$. Since $\matMabs \preceq \matWabs + \matU \matU^\top$, we have
\begin{eqnarray*}
\langle \Vhat, \matMabs \Vhat \rangle &=& \tr(\matMabs \Vhat\Vhat^\top)\\
&=&\tr(\matMabs\itVbar\itVbar^\top ) +  \tr(\matMabs\itVtil\itVtil^\top) + \\
&\le& \sum_{i=1}^{r - r'} \lambda_i(\matMabs) + \tr(\matMabs\itVtil\itVtil^\top ) \\
&\le& \sum_{i=1}^{r - r'} \lambda_i(\matMabs) + \tr(\matWabs\itVtil\itVtil^\top )  + \tr(\boldlam \matU \matU^\top \itVtil\itVtil^\top )\\
&\le& \sum_{i=1}^{r - r'} \lambda_i(\matMabs) + r'\|\matW\|_{\op}  + \boldlam  \Phi(\itVtil;\matU)~.
\end{eqnarray*}
In particular, if $\langle \Vhat, \matMabs \Vhat \rangle \ge (1 - \eta)\sum_{i=1}^{r} \lambda_i(\matMabs) $, then we must have that
\begin{eqnarray*}
\boldlam  \Phi(\itVtil;\matU) &\ge& (1 - \eta)\sum_{i=1}^{r} \lambda_i(\matMabs) - \sum_{i=1}^{r - r'} \lambda_i(\matMabs) + r'\|\matW\|_{\op}\\
&\ge& - \eta\sum_{i=1}^{r} \lambda_i(\matMabs) + \sum_{r=r-r'+1}^r \{\lambda_i(\matMabs)  - \|\matW\|_{\op}\}\\
&\ge& - \eta\sum_{i=1}^{r} \lambda_i(\matMabs) + r' \{\lambda_r(\matMabs)  - \|\matW\|_{\op}\}\\
&\ge& - \eta r\lambda_1(\matMabs) + r'  \gaptil_r(\matMabs)\\
&=& r' \cdot\left(\gap_r(\matMabs) -\eta  \cdot \frac{r\lambda_1(\matMabs)}{r'}\right)~,
\end{eqnarray*}
where $\gaptil_r(\matMabs) := \lambda_r(\matMabs)  - \|\matW\|_{\op}$. In particular, if we select $\eta = \gaptil_r(\matMabs)\cdot \frac{r'}{2r\lambda_1(\matMabs)}$, then we have that
\begin{eqnarray*}
 r' \boldlam \cdot \lambda_{r - r' + 1}(\Vhat^\top \matU \matU^\top \Vhat))  &\ge& r' \cdot\left(\gaptil_r(\matMabs) -\eta  \cdot \frac{r\lambda_1(\matMabs)}{r'}\right)\\
 &=& r' \cdot\left(\gaptil_r(\matMabs)\right)/2~.
\end{eqnarray*}
Rearranging proves the lemma.

%!TEX root = main.tex

\section{Controlling the ``Entropy'' and ``Information'' Terms} 

\subsection{Proof of Lemma~\ref{lem:small_ball_prob} (``Entropy Term'')\label{sec:small_ball}}

Note that$\|V\|_{\op} = 1$, and hence the map $\matu \mapsto \|V\matu\|_2$ is $1$-Lipschitz. By spherical isoperimetry, this implies 
\begin{align*}
\Pr_{\matu \sim \calS^{d-1}}[\sqrt{d}\|V\matu\|_2 \ge \median(\sqrt{d}\|V\matu\|_2) + t] \le e^{-\frac{1}{2}t^2}.
\end{align*}

Note that by rotational invariance of $\matu$, we have $\median(\sqrt{d}\|V\matu\|_2) = \median_{\matu \sim \calS^{d-1}}\left[\sqrt{d\sum_{i=1}^k \matu_i^2 }\right] $. Hence by inequality that
\begin{align*}
\Pr(\|V\matu\|_2 \ge \sqrt{2k}) = \Pr(\|V\matu\|_2^2 \ge 2k) \le \frac{\Exp[\|V\matu\|_2^2}{2(k+1)} = \frac{1}{2}.
\end{align*}
Thus, $\median(\sqrt{d}\|V\matu\|_2) \le \sqrt{2(k+1)}$. Hence, by Markov's inequality and the fact that $\median(\sqrt{d}\|V\matu\|_2) \le \sqrt{2(k+1)}\le \tau_{k+1}$,
\begin{align*}
\Pr_{\matu \sim \sphere}[d\|V\matu\|_2^2 \ge \tau_{k+1}] &\le& \Pr_{\matu \sim \sphere}[\sqrt{d}\|V\matu\|_2 \ge \sqrt{\tau_{k+1}} ]\\
&\le& \Pr_{\theta \sim \sphere}[\sqrt{d}\|V\matu\|_2 \ge \median(\sqrt{d}\|V\matu\|_2) + (\sqrt{\tau_{k+1}} - \median(\sqrt{d}\|V\matu\|_2) ] \\
&\le& \exp( - \frac{1}{2}(\sqrt{\tau_{k+1}} - \median(\sqrt{d}\|V\matu\|_2)^2)\\
&\le& \exp( - \frac{1}{2}(\sqrt{\tau_{k+1}} - \sqrt{2(k+1)})^2).
\end{align*}

\subsection{Proof Lemma~\ref{ConditionalLemma}\label{CondLemProof}}
	Recall the definition $\Sigma_i := P_{i-1}(I_d+v^{(i)}v^{(i)\top})P_{i-1}$, and that $\mathcal{F}_{i-1}$ is the $\sigma$-algebra generated by $\vone,\wonetil,\dots,\viminustil,\wiminustil$. Since our algorithm is deterministic, $\vi$ is $\calF_{i-1}$ measurable. It then suffices to show that 
	\begin{align}
	\wjtil = P_{j-1}W\vj \big{|}\calF_{i} \sim \mathcal{N}(0,\Sigma_i).
	\end{align}
	Recall that $\Sigma_i$ is degenerate, so we understand $\mathcal{N}(0,\frac{1}{d}\Sigma_i)$ as a normal distribution absolutely continuous with respect to the Lebesque measure supported on $(\ker P_{i-1})^{\perp}$. Note that $\widetilde{w}^{(i)}$ is conditionally independent of $\wonetil,\dots,\wiminustil$ given $\vone,\dots,\viminus,\vi$. Consequently, the conditional distribution of $\witil$ given $\mathcal{F}_i$ can be computed as if the queries $\vone,\dots,\vi$ were fixed in advanced. 

	Hence, throughout, we shall assume that $\vone,\dots,\vi$ are deterministic, and consider the joint distribution of $\wonetil,\dots,\wiminustil,\witil$. We will show that $\witil$ is independent of $\wonetil,\dots,\wiminustil$, and that its marginal is $\mathcal{N}(0,\frac{1}{d}\Sigma_i)$. Since the map $W \mapsto P_{j-1}W\vj$ is linear maps, $\wonetil,\dots,\witil$ are jointly Gaussian with mean zero. Thus, it suffices to show that 1) the (marginal) covariance of $\witil$ is $\Sigma_i$ and, 2) the covariance between $\witil$ and $\wjtil$ for $j \ne i$ is $0$. The covariances are computed as
	\begin{align}\label{CovEq}
	\Exp\left[\widetilde{w}^{(j)} \widetilde{w}^{(j)\top}\right] = \Exp\left[(P_{i-1}Wv^{(i)})(P_{j-1}Wv^{(j)})^{\top}\right] = P_{i-1} \Exp\left[Wv^{(i)}v^{(j)\top}W\right] P_{j-1},
	\end{align}
	and we compute the inner term with the following lemma.
	\begin{lem}\label{GaussComp}
	For any $v^{(i)},v^{(j)}$, one has
	\begin{align}\label{CovEq2}
	\Exp\left[Wv^{(i)}v^{(j)\top}W\right] = v^{(j)}v^{(i)\top} + \langle v^{(i)},v^{(j)} \rangle I .
	\end{align}
	\end{lem} 
	For $\vi = \vj$, Equations~\eqref{CovEq} and~\eqref{CovEq2} immediately imply $\witil$ has covariance $\Sigma_i$. Moreover, for $j < i$, we have
	\begin{align}
	P_{i-1} \Exp\left[Wv^{(i)}v^{(j)\top}W\right] P_{j-1} \overset{(i)}{=} P_{i-1}v^{(i)}v^{(j)\top}P_{j-1} \overset{(ii)}{=} 0
	\end{align}
	where $(i)$ holds from Lemma~\ref{GaussComp} and the fact that $\langle \vi,\vj \rangle = 0$ (since $\vone,\dots,\vi$ are assumed to be orthogonal), and $(ii)$ holds since $\vj \in \ker(P_{i-1})$, as $P_{i-1}$ projects onto the complement of $\{\vone,\dots,\viminus\}$.

	\begin{proof}[Proof of Lemma~\ref{GaussComp}] 
	For $a \in \{1,\cdots,d\}$,
	\begin{align*}
	& \Exp[Wv^{(i)}v^{(j)\top}W]_{aa} = \Exp[W_{aa}^2]v_a^{(i)}v^{(j)}_a + \sum_{p \ne a}\Exp[W_{ap}^2]v_a^{(i)}v^{(j)}_a \\
	&= 2v_a^{(i)}v^{(j)}_a + \sum_{p \ne a}v_p^{(i)}v^{(j)}_p = v_a^{(i)}v^{(j)}_a + \sum_{p}v_p^{(i)}v^{(j)}_p = v_a^{(i)}v^{(j)}_a + \langle v^{(i)}, v^{(j)}\rangle .
	\end{align*}
	Whereas for $a \ne b$, 
	\begin{align}
	\Exp[Wv^{(i)}v^{(j)\top}W]_{ab} &= \Exp\left[\sum_{p,q}W_{ap}W_{bq}v_p^{(i)}v^{(j)}_q\right] = \sum_{p,q}\Exp[W_{ap}W_{bq}]v_p^{(i)}v^{(j)}_q.
	\end{align}
	$W_{ap}$ and $W_{bq}$ are independent unless $(a,p) = (b,q)$ or $(a,p) = (q,b)$. If $a \ne b$,  then this means the only term in the above sum which is non zero is $p = b$ and $q = a$, which yields $v_a^{(j)}v_b^{(i)}$.
	\end{proof}

\subsection{Proof of Lemma~\ref{Generic_UB_LL}\label{sec:gen_ub_llProof}}
	To make the argument, we will need a bit of notation. First $\itV_{i:j} \in \Stief(j-i,d)$ to be the matrix whose columns are the vector $\vi,\dots,\vj$.  In particular, $\itV_k = \itV_{1:k}$.We will let $\itVtil_{j+1:k}$ be a stand-in for an arbitrary matrix whose columns are $\vjplustil,\dots,\vktil$. Given two matrices $\itV_{1:j},\itV_{j+1:k}$, let $\itV_{1:j} \oplus \itV_{j+1:k}$ denote the matrix whose columns are such the concatenation of $\itV_{1:j}$ and $\itVtil_{j+1:k}$. We will introduce the ``head set'' of all sequences of orthogonal matrix $\itV_{1:j}$ which can be extended to matrices in $\calV_k$
	\begin{align}
	\calVhead^{1:j} := \{\itV_{1:j} : \exists  \itVtil_{j+1:i} \text{ such that } \itV_{1:j} \oplus \itVtil_{j+1:i} \in \calV_k\}.
	\end{align}
	and the ``tail'' set of all possible ways to complete the matrix $\itV_{1:j}$ such that 
	\begin{align}
	\calVtail^{1:j}(\itV_{1:j}) := \{\itVtil_{j+1:k} : \itV_{1:j} \oplus \itVtil_{j+1:k} \in \calV_k\}.
	\end{align}
	Observe that $\calVtail^{j+1:k}(\itV_{1:j]})$ depends only on $\itZ_{j-1}$, since the vectors $\itV_{1:j}$ are $\itZ_{j-1}$-measurable (recall that the $j$-th measurement is decided upon at the end of the $i$-th round. 

	Recall the definition:
	\begin{align*}
	g_i(\itVtil_i) &:= \Exp_{\Prit_0}\left[\left(\frac{\rmd\Prit_u(\itZ_i |\itZ_{i-1})}{\rmd\Prit_0(\itZ_i | \itZ_{i-1})}\right)^r \I(\itV_i = \itVtil_i) \right].
	\end{align*}

	 We now define partial supremum over the products of the terms $g_i(\itV_{1:j}\oplus \itVtil_{j+1:i})$ for $i \ge j$ as follows 
	 \begin{align*}
	 G_{j}(\itV_{1:j}) = \sup_{\itVtil_{j+1:k} \in \calVtail^{j+1:k}(\itV_{1:j]})} \prod_{i=j+1}^k g_i( \itV_{1:j} \oplus \itVtil_{j+1:i}).
	 \end{align*}
	 adopting the convention $G_{k}(v^{(1:k)}) = 1$ (since it's an empty product). We observe also that
	 \begin{align*}
	  G_{0} =  \sup_{\itVtil_k \in  \calV^k} \prod_{i=1}^k g_i(\itV_k)
	 \end{align*}
	 does not take any arguments. Finally, define the shorthand for the likelihood ratio terms
	 \begin{align*}
	 P_j(\itZ_j) := \left(\frac{\rmd\Prit_u(\itZ_j)}{ \rmd \Prit_0(\itZ_{j}) }\right)^r \quad \text{and } \quad P_j(\itZ_j | \itZ_{j-1}) = \left(\frac{\rmd\Prit_u(\itZ_j | \itZ_{j-1})}{ \rmd \Prit_0(\itZ_j | \itZ_{j-1}) }\right)^r.
	 \end{align*}
	 adopting the convention that $P_0(Z_0) = 1$. We observe then that, for all $j \in \{1,\dots,k\}$, 
	 \begin{align}\label{gj_alt_eq}
	 P_j(\itZ_j) = P_{j}(\itZ_j|\itZ_{j-1})P_{j-1}(\itZ_{j-1}) \quad\text{and}\quad \Exp_{\Prit_0}[P_{j}(\itZ_j|\itZ_{j-1}) | \itZ_{j-1}] = g_j(\itV_{1:j}).
	 \end{align}
	 Hence, with out notation, at $i = k$ we have
	 \begin{align*}
	 \Exp_{\Prit_0}\left[\left(\frac{\rmd\Prit_u(\itZ_k)}{ \rmd \Prit_0(\itZ_k) }\right)^r \I(\itV_{k} \in \calV_k)\right] = \Exp_{\Prit_0}\left[G_{k}(\itV_{1:k}) \cdot P_k(\itZ_k) \cdot \I(\itV_{1:k} \in \calVhead^{1:k}))\right],  
	 \end{align*}
	 since we took  $G_{k}(\itV_{1:k})  := 1$. Moreover, since $P_0(\itZ_0) = 1$, we have
	 \begin{align} 	 \Exp_{\Prit_0}\left[G_{0} \cdot P_0(\itZ_0)\right] = \sup_{\itV_k'} \prod_{i=1}^k g_i(\itV_k').
	 \end{align}
	 Hence, it suffices to show that 
	 \begin{align*}
	  \Exp_{\Prit_0}\left[G_{k}(\itV_{1:k}) \cdot P_k(\itZ_k) \cdot \I(\itV_{1:k} \in \calVhead^{1:k}))\right] \le \Exp_{\Prit_0}\left[G_0 \cdot P_0(\itZ_0)\right].
	 \end{align*}
	  The above display is a direct consequence of applying the following claim inductively:
	 \begin{claim}\label{llik}
	 It holds that
	 \begin{multline*}
		\Exp_{\Prit_0}\left[G_{j}(\itV_{1:j}) \cdot P_j(\itZ_j) \cdot \I( \itV_{1:j} \in \calVhead^{1:j}))\right] \le  \Exp_{\Prit_0}\left[\I( \itV_{1:j-1} \in \calVhead^{1:j-1}) G_{j-1}(\itV_{1:j-1})  P_{j-1}(\itZ_{j-1}) \right].
		\end{multline*}
	 \end{claim}
	To prove the above claim, we have that for any $j \in [k]$
		\begin{align*}%\label{Main_Inductive_Step_Chi}
		&\Exp_{\Prit_0}\left[G_{j}(\itV_{1:j}) \cdot P_j(\itZ_j) \cdot  \I(\itV_{1:j} \in \calVhead^{1:j}))\right] \\
		&=\Exp_{\Prit_0}\left[\Exp\left[G_{j}(\itV_{1:j}) \cdot P_j(\itZ_j) \cdot \I( \itV_{1:j} \in \calVhead^{1:j})) \big{|} \itZ_{j-1} \right]\right] \\
		&=\Exp_{\Prit_0}\left[\Exp\left[G_{j}(\itV_{1:j}) \cdot P_{j-1}(\itZ_{j-1}) \cdot P_{j}(\itZ_j|\itZ_{j-1}) \cdot   \I( \itV_{1:j} \in \calVhead^{1:j})) \big{|} \itZ_{j-1} \right]\right] \\
		&\overset{(i)}{=}\Exp_{\Prit_0}\left[P_{j-1}(\itZ_{j-1}) \cdot  \Exp\left[G_{j}(\itV_{1:j}) P_{j}(\itZ_j|\itZ_{j-1})  \I(\itV_{1:j}) \in \calVhead^{1:j})) \big{|} Z_{j-1} \right]\right] \\
		&\overset{(ii)}{=}\Exp_{\Prit_0}\left[P_{j-1}(\itZ_{j-1}) \cdot G_{j}(\itV_{1:j}) \I( \itV_{1:j} \in \calVhead^{1:j}))  \Exp\left[ P_{j}(\itZ_j|\itZ_{j-1})   \big{|} \itZ_{j-1} \right]\right] \\ \\
		&\overset{(iii)}{=}\Exp_{\Prit_0}\left[ P_{j-1}(\itZ_{j-1}) \cdot G_{j}(\itV_{1:j}) \I( \itV_{1:j} \in \calVhead^{1:j}))  \cdot g_j(\itV_{1:j})   \right] .
		\end{align*}
		Here, $(i)$ follows since $P_{j-1}(\itZ_{j-1})$ is $\itZ_{j-1}$-measurable, $(ii)$ follows since $ G_{j}(\itV_{j}) $ and $\I(\itV_j \in \calVhead^{1:j})$ are deterministic functions of $\itV_j$, which by assumption, are deterministic functions of $\itZ_{j-1}$. Finally, $(iii)$ follows from the definition of $g_j$, as noted in Equation~\eqref{gj_alt_eq}. To conclude, we now need only show that
		\begin{align}\label{InnerLikelihoodClaimEq}
		G_{j}(\itV_{1:j}) \I( \itV_{1:j} \in \calVhead^{1:j}))  \cdot g_j(\itV_{1:j})  \le \I( \itV_{1:j-1} \in \calVhead^{1:j-1})) G_{j-1}(\itV_{1:j-1}).
		\end{align}
		Note that, since $\itV_{1:j} = [\vone|\dots|\vj] \in \calVhead^{1:j}$, then for for any $\itVtil_{j+1:k} = [\vjplustil|\dots|\vktil] \in \calVtail^{j+1:k}(\itV_{1:j})$, we have that $\vone,\dots,\vj,\widetilde{\mathsf{v}}^{(j+1)},\dots,\vktil \in \calV^k$. In particular, this implies that 
		\begin{align}
		\vj \oplus \itVtil_{j+1:k} = [\vj|\widetilde{\mathsf{v}}^{(j+1)}|\dots|\vktil] \in \calVtail^{j:k}(\itV_{j-2}).
		\end{align}
		Hence, we have
		\begin{align*}
		G_{j}(\itV_{1:j}) \I( \itV_{1:j}) \in \calVhead^{1:j}) g_j(\itV)  &= \I( \itV_{1:j}) \in \calVhead^{1:j}) \cdot g_j(\itV_{1:j})\sup_{\itVtil_{j+1:k} \in \calVtail^{j+1:k}(\itV_{1:j})} \prod_{i = j+1}^k g_i(\itV_{1:j} \oplus \itVtil_{j+1:i})\\
		 &\le \I( \itV_{1:j-1} \in \calVhead^{1:j-1}) \cdot \sup_{\itVtil_{j:k} \in \calVtail^{j:k}(\itV_{j-1})} \prod_{i = j}^k g_i(\itV_{1:j-1} \oplus \itVtil_{j:i} )\\
		 &= \I( \itV_{1:j-1} \in \calVhead^{1:j-1})\cdot G_{j-1}(\itV_{1:j-1}).
		\end{align*}

\subsection{Proof of Lemma~\ref{lem:power_divergence_comp}\label{sec:power_div_proof}}

	 By a translation, we may assume without loss of generality that $\mu_1 = \mu$ and $\mu_2 =  0$.
	\begin{align*}
	\Exp_{\Q}\left[\left(\frac{\rmd \Pr}{\rmd \Q}\right)^{1+\eta}\right]&=  \Exp_{\matz \sim \mathcal{N}(0,\Sigma)}e^{- \frac{1+\eta}{2}\left\{(\matz-\mu)^{\top}\Sigma^{\dagger}( \matz - \mu) - \matz^{\top}\Sigma^{\dagger} \matz\right\}}\\
	&=  \Exp_{\matz \sim \mathcal{N}(0,\Sigma)}e^{- \frac{1+\eta}{2}\left\{(\mu^{\top}\Sigma^{\dagger} \mu - 2\mu^{\top}\Sigma^{\dagger}\matz\right\}}\\
	&=  e^{-\frac{1+\eta}{2}\mu^{\top}\Sigma \mu}\Exp_{\matz \sim \mathcal{N}(0,\Sigma)}e^{- 2\mu^{\top}\Sigma^{\dagger} \matz\}}\\
	&=  e^{-\frac{1+\eta}{2}\mu^{\top}\Sigma \mu}\Exp_{\matx \sim \mathcal{N}(0,I)}e^{(1+\eta)\mu^{\top}\Sigma^{\dagger} \Sigma^{1/2} \matx}\\
	&=  e^{-\frac{1+\eta}{2}\mu^{\top}\Sigma^{\dagger} \mu}\exp(\frac{(1+\eta)^2}{2} \mu^{\top} \Sigma^{\dagger} \Sigma \Sigma^{\dagger} \mu)\\
	&=  e^{-\frac{1+\eta}{2}\mu^{\top}\Sigma^{\dagger} \mu}\exp(\frac{(1+\eta)^2}{2} \mu^{\top} \Sigma^{\dagger} \mu)\\
	&= \exp(\frac{\eta(1 + \eta)}{2}\mu^{\top} \Sigma^{\dagger} \mu).
	\end{align*}
%!TEX root = main.tex
\section{Proof of Theorem~\ref{thm:rank_one_thm_est}\label{sec:rk_one_thm_proof} }

	Define $\tau_0(\delta) = 32\boldgap^{-1}(\log \delta^{-1} + \boldgap^{-1/2})$, and set $\tau_k = \lambda^{4k} \tau_0(\delta)$. It then suffices to show that, with probability at least $1 - \delta$, 
	\begin{eqnarray*}
	\Pr_{\matu \sim \calS^{d-1}}\Prit_{\matu}\left[\exists k \ge 1:  d\matu^\top \itV_k \itV_k^\top \matu \ge \tau_k \right] \le  \delta.
	\end{eqnarray*}
	Throughout, we will use the following technical lemmas to simplify our expressions:
	\begin{lem}\label{lem:boldgapsize} $\log(\boldlam) \ge \boldgap^{1/2}$, and $\boldlam^2/(\boldlam-1)^2 \le \boldgap^{-1}$
	\end{lem}
	\begin{proof}
	Since $\log(x) \ge \frac{x-1}{x}$ for $x \ge 0$, we have $\log(\boldlam) \ge \frac{\boldlam - 1}{\boldlam} = \sqrt{\frac{(\boldlam-1)^2}{\boldlam^2}} > \sqrt{\frac{(\boldlam-1)^2}{\boldlam^2}} = \sqrt{\boldgap}$. The second point follows since $\boldlam^2/(\boldlam-1)^2 \le (\boldlam^2+1)/(\boldlam-1)^2 = \boldgap^{-1}$.
	\end{proof}
	 
	\begin{lem}\label{lem:technical_recursion_max}
	\begin{eqnarray*}
	\max_{k \ge 0}  \boldlam^{-4k}(k+1) \le 1 + \frac{1}{4e \log (\boldlam)} & \text{and} &  \max_{k \ge 1} \boldlam^{-4k}\log(1+k)  \le \frac{1}{4e\log \boldlam}.
	\end{eqnarray*}
	\end{lem}
	Lemma~\ref{lem:technical_recursion_max} is proved in the following subsection. Observe then that, for $\delta \in (0,1/e)$, we have
	\begin{eqnarray}
	\tau_0(\delta) \ge  &\overset{\text{Lemma~\ref{lem:boldgapsize}}}{\ge}& \frac{32\boldlam^2}{(\boldlam-1)^2}(\log(1/\delta) + \log^{-1} (\boldlam)) \label{tau_bound} \\
	&\ge& \frac{32\boldlam^2}{(\boldlam-1)^2}(1 + \frac{1}{4e\log (\boldlam)}) \nonumber.
	\end{eqnarray}
	Hence, by  first inequality in the above Lemma~\ref{lem:technical_recursion_max}, we have
	\begin{eqnarray*}
	2\sqrt{(2k + 2)/\tau_{k+1}} &=& 2^{3/2}\sqrt{\boldlam^{-4k}(k+1)/\tau_{1}} \\
	&\le& 2^{3/2}\sqrt{ (1 + \frac{1}{4e \log (\boldlam)})/\tau_1}  \\
	&\le& 2^{3/2}\sqrt{(\boldlam-1)^2/32}   = \frac{1}{2}(1-1/\boldlam).
	\end{eqnarray*}
	Taking $\eta = \boldlam - 1$, the above inequality $(i)$ in the following display:
	\begin{eqnarray*}
	\exp\left\{ \frac{\eta}{2}  \left(\boldlam^2 \tau_k  - \frac{\left(\sqrt{\tau_{k+1}}-\sqrt{2k+2}\right)^2}{1+\eta}\right) \right\} &=& \exp\left\{ \frac{\boldlam - 1}{2\boldlam}  \left(\boldlam^3 \tau_k  - \left(\sqrt{\tau_{k+1}}-\sqrt{2k+2}\right)^2\right) \right\} \\
	&\ge& \exp\left\{ \frac{\boldlam - 1}{2\boldlam}  ((\boldlam^3 \tau_k - \tau_{k+1} + 2\sqrt{(2k + 2)/\tau_{k+1}})\tau_{k+1} )  \right\} \\
	&\overset{(i)}{\ge}& \exp\left\{ \frac{\boldlam - 1}{2\boldlam}  ((\boldlam^3 \tau_k - \tau_{k+1} + \frac{1}{2}(1 - 1/\boldlam)\tau_{k+1} )  \right\} \\
	&\overset{}{=}& \exp\left\{ -\frac{\boldlam - 1}{2\boldlam}  (\tau_{k+1} - \tau_{k+1}/\boldlam   + \frac{1}{2}(1 - 1/\boldlam)\tau_{k+1}  \right\} \\
	&=& \exp\left\{ -\frac{(\boldlam - 1)^2}{4\boldlam^2} \tau_{k+1}  \right\} \\
	&=& \exp\left\{ -\frac{ \tau_{1}\boldlam^{4k}(\boldlam - 1)^2}{4\boldlam^2}\right\}.
	\end{eqnarray*}
	Moreover, we have
	\begin{eqnarray}
	\frac{\tau_{1}\boldlam^{4k}(\boldlam - 1)^2}{4\boldlam^2} &\overset{\text{Equation~\eqref{tau_bound}}}{\ge}&  8\lambda^{4k}(\log(1/\delta) + \log^{-1} (\boldlam)) \nonumber\\
	&\overset{\text{Equation~\eqref{tau_bound}}}{\ge}&  8\lambda^{4k}(\log(1/\delta) + \log^{-1} (\boldlam))\nonumber\\
	&\ge&  8\log(1/\delta) + \log(k+1) \cdot \frac{8 \lambda^{4k}}{\log(1/\lambda)\log(k+1)}\nonumber\\
	&\ge&  8\log(1/\delta) + \log(k+1) \cdot \frac{8 \lambda^{4k}}{\log(1/\lambda)\log(k+1)} \nonumber\\
	&\overset{\text{Lemma~\ref{lem:technical_recursion_max}}}{\ge}&  8\log(1/\delta) + \log(k+1) \cdot 64 e \ge 2\log(1/\delta) + 2\log(k+1) \label{eq:last_line_rnk_one}.
	\end{eqnarray}
	Recalling the bound 
	\begin{eqnarray}\label{eq:one_stage_growth_bound2}
	\Pr[\{ \Phi(\itV_k; \matU) \ge \tau_k\} \cap  \{ \Phi(\itV_{k-1}; \matU) \le \tau_{k-1}\}]  \nonumber \le \exp\left\{ \frac{\eta}{2}  \left(\boldlam^2 \tau_k  - \frac{\left(\sqrt{\tau_{k+1}}-\sqrt{2k+2}\right)^2}{1+\eta}\right)\right\}~,
	\end{eqnarray}
	and putting thing together, we conclude
	\begin{eqnarray*}
	 \Pr[\exists k \ge 1: \Phi(\itV_k; \matU) \ge \tau_k] &\le& \sum_{k \ge 1} \Pr[\{ \Phi(\itV_k; \matU) \ge \tau_k\} \cap  \{ \Phi(\itV_{k-1}; \matU) \le \tau_{k-1}\}] \\
	 &\le& \sum_{k \ge 0} \exp( - \frac{\tau_{1}\boldlam^{4k}(\boldlam - 1)^2}{4\boldlam^2}) \\
	 &\overset{\text{Equation~\eqref{eq:last_line_rnk_one}}}{\le}& \sum_{k \ge 0} \frac{\delta^2}{(k+1)^2} = \delta^2 \pi^2/6 \le 2\delta^2 \overset{\delta \le e^{-1}}{\le} \delta.
	\end{eqnarray*}

	\subsection{Proof of Lemma~\ref{lem:technical_recursion_max}}

	For the first inequality, we have 
	\begin{eqnarray*}
	\max_{k \ge 0} (k+1)/ \boldlam^{4k} &=& \max_{k \ge 0} (k+1) \exp(- 4k\log(1/\boldlam))\\
	&=& \max_{k \ge 0} (\frac{k}{4\log(\boldlam)}+1) \exp(- k) \\
	&\le& 1 + (4\log(\boldlam))^{-1}\max_{k \ge 0} (k \exp(- k)) = 1 + \frac{1}{4e \log (\boldlam)} .
	\end{eqnarray*}
	For the second inequality, we have that 
	\begin{eqnarray*}
	\max_{k \ge 0} \log(1+k)/ \boldlam^{4k} &\,e& \max_{k \ge 0} k\exp(- 4k\log(1/\boldlam))\\
	&=&  \max_{k \ge 0} \frac{k}{4 \log \boldlam} \exp(- k))\\
	&=&\frac{1}{4e\log \boldlam}.
	\end{eqnarray*}
%!TEX root = main.tex
\section{Supporting Results for the Rank-K Case}
	\subsection{Details for Proof of Theorem~\ref{thm:rank_r_thm_est}\label{sec:rank_r_thm_est_app_proof}}
	 	Recall that we choose $\rho \ge \boldlam^3 c_{d,r}$,  $\Delta \ge \frac{\rho(2k_{\max}+2)}{(\rho - 1)^3}$, and define the event:
		 \begin{eqnarray}\label{eq:calErho_def}
		\calE(\lamtil,\Delta,k_{\max}) := \left\{\forall e \in \R^r, k \in [1,\dots,k_{\max}], \quad d\Phi(\itV_k;\matU e) + \Delta \le \lamtil(d\Phi(\itV_{k-1};\matU e) + \Delta)\right\}.
		\end{eqnarray}

		 On $\calE(\rho^2,\Delta, k_{\max})$, we have that for all $k \in [k_{\max}]$,
		 \begin{eqnarray*}
		\rho^{2k} \Delta^r \overset{\text{Lemma~\ref{lem:Determinant_Growth_Lemma}}}{\ge} \det(d\matU^\top \itV_k\itV_k \matU + \Delta I_r) &=& \prod_{i=1}^r \lambda_{i}(d\matU^\top \itV_k\itV_k \matU + \Delta I_r) \\
		&\ge&  (\lambda_{r'}(d\matU^\top \itV_k\itV_k \matU + \Delta I_r))^{r'} \cdot \lambda_{\min}(\matU^\top \itV_k\itV_k \matU + \Delta I_r)^{r - r'} \\
		&\ge&  (\lambda_{r'}(\matU^\top \itV_k\itV_k \matU))^{r'} \cdot \Delta^{r-r'}. 
		\end{eqnarray*}
		where the last step uses $\matU^\top \itV_k\itV_k \matU \succeq \Delta I_r, \matU^\top \itV_k\itV_k \matU \succeq 0 $. Rearranging, we have that
		\begin{eqnarray}
		d\lambda_{r'}(\matU^\top \itV_k\itV_k \matU) = \lambda_{r'}(d\matU^\top \itV_k\itV_k \matU) \le \Delta \cdot \rho^{\frac{2k}{r'}}.
		\end{eqnarray}
		We now apply Proposition~\ref{prop:CalE_prob} with the following $\rho$ and $\Delta$. Because we assume $c_{d,r} = \frac{d - (r-1)}{d}  \le \boldlam$, we can take $\rho = \boldlam^4$. If we choose
		\begin{eqnarray*}
		\Delta := \frac{2}{\boldlam^4\sqrt{\boldgap}}\left(  \frac{k_{\max}}{\boldgap} + \log(1/\delta) + (r+2)\log(20d/\boldlam^4\gap^{1/2})\right).
		\end{eqnarray*}
		Then, using the fact $\rho -1 = (\boldlam - 1)(\boldlam + 1)(\boldlam^2 + 1) = \boldgap(\boldlam + 1)(\boldlam^2 + 1)^{3/2} \ge \max\{4,\boldlam^4\} \boldgap^{1/2} $, we find that 
		\begin{eqnarray*}
		\frac{\rho(2k_{\max}+2)}{(\rho - 1)^3} \le   \frac{4k_{\max}\boldlam^4 }{\boldgap^{3/2}\boldlam^{4}4} \le \Delta,
		\end{eqnarray*}
		and hence statisifies the conditions of Proposition~\ref{prop:CalE_prob}.  Moreover, 
		\begin{eqnarray*}
		(20 d/(\rho - 1))^{r+2}\exp\left\{ \frac{-\boldlam^3(\boldlam - 1)\Delta}{2}  \right\} &\le& \exp\left\{ \frac{-\boldlam^4 \boldgap^{-1/2}\Delta}{2} + (r+2)\log(\frac{20d}{\boldlam^4 \boldgap^{1/2}} )\right\} \\
		&\le& \exp\{-\log(1/\delta)\} = \delta.
		\end{eqnarray*}
		Putting things together, we see that for $k_{\max} \ge 1$, the following is at least $1-\delta$:
		\begin{eqnarray*}
		\Pr\left[\forall k \in k_{\max}: d\lambda_{r'}(\matU^\top \itV_k\itV_k \matU) \le \frac{2\boldlam^{8k/r'}}{\boldlam^4\sqrt{\boldgap}}\left(  \frac{k_{\max}}{\boldgap} + \log \delta^{-1} + (r+2)\log(20d/\boldlam^4\gap^{1/2})\right)  \right].
		\end{eqnarray*}
		By union bounding over all $k_{\max} \in [d]$ and some elementary manipulations (noting $\boldlam^4 \le 1$),
		\begin{eqnarray*}
		\Pr\left[\forall k \in [d]: \lambda_{r'}(\matU^\top \itV_k\itV_k \matU) \le \frac{2\cdot \boldlam^{8k/r'}}{d\sqrt{\boldgap}}\left(  \frac{k}{\boldgap} + \log \delta^{-1} + 4r\log(20d/\boldlam^4\gap^{1/2})\right)  \right] \ge 1 - \delta.
		\end{eqnarray*}
		Finally, we simplify by noting that, from Lemma~\ref{lem:boldgapsize} and \ref{lem:technical_recursion_max}, we have $k/\boldgap = r' \cdot k/r'/\boldgap \le 2r'\boldgap^{-3/2}\lambda^{k/r'}$, and by assumption that $d \ge \gap^{-1/2}$, we have  $4r\log(20d/\boldlam^4\gap^{1/2}) \le 8r \log (20d)$. Hence, $\left(  \frac{k}{\boldgap} + \log \delta^{-1} + 4r\log(20d/\boldlam^4\gap^{1/2})\right) \le (8r \log(20d) + 2r'\boldgap^{-3/2}\lambda^{k/r'} + \log(e\delta^{-1}) \le \log(e\delta^{-1})( 1 + 8r \log(20d) + 2r'\boldgap^{-3/2}\lambda^{k/r'}) \le \log(e\delta^{-1})( 13r \log(20d)\boldgap^{-3/2}\lambda^{k/r'})$.

\subsection{Details for Proof of Propostion~\ref{prop:CalE_prob} \label{sec:prop:CalEprob_details}}
 	To prove Proposition~\ref{prop:CalE_prob}, we invoke a covering argument:
	\begin{claim}\label{claim:rk_k_covering claim}
	Let $\calN$ is an $\epsilon$-net of $\calS^{r-1}$, and let $e \in \calS^{r-1}$. Then there exists an $e' \in \calN$ such that $|d\Phi(\itV_k;\matU e) - d\Phi(\itV_k;\matU e')| \le  2d\epsilon$.
	\end{claim}
	\begin{proof}
		 Since $\calN$ is a $\epsilon$ -net we may choose an $e' \in \calN$ satisfies $\| e' -e\|_2 \le \epsilon$; and hence, the nuclear norm difference of the outer products satisfy $ \|ee^\top - ee^{'\top}\|_{*} \le \|e(e-e')^\top\|_{*} + \|(e-e')e^{'\top}\|_* = 2\epsilon$. Thus, 

		\begin{eqnarray*}
		 |d\Phi(\itV_k; \matU e') - d\Phi(\itV_k; \matU e')| 
		&\le& d| \langle \matU^{\top} \itV_k \itV_k^\top \matU, ee^\top - ee^{'\top}|\\
		&\overset{\text{matrix Holder}}{\le}&  d \|\matU^{\top} \itV_k \itV_k^\top \matU\|_{\op}  \|ee^\top - ee^{'\top}\|_{*}\\
		&\le& 2d \epsilon  \|\matU^{\top} \itV_k \itV_k^\top \matU\|_{\op}  \le 2 d\epsilon,
		\end{eqnarray*}
		where the last step uses  $\|\matU^{\top} \itV_k \itV_k^\top \matU\|_{\op}  \le 1$.
		\end{proof}
	This implies that if $\calN$ is an $ \epsilon = \frac{\Delta (\rho^2 -\rho)}{2d(1+\rho)}$-net of $\calS^{r-1}$, then for all $e \in \calS^{r}$, there is an $e'$ for which
	\begin{eqnarray*}
	d\Phi(\itV_k;\matU e) + \Delta \ge \rho^2 (d\Phi(\itV_k; \matU e) + \Delta) &\implies& d\Phi(\itV_k;\matU e') + 2d\epsilon + \Delta \ge \rho^2 (d\Phi(\itV_k; \matU e') + \Delta) - 2d\rho \rho \epsilon\\
	&\implies& d\Phi(\itV_k;\matU e')  + \Delta \ge \rho^2 (d\Phi(\itV_k; \matU e') + \Delta) - (\rho +1) 2d\epsilon\\
	&\implies& d\Phi(\itV_k;\matU e')  + \Delta \ge \rho^2 (d\Phi(\itV_k; \matU e') + \Delta) - (\rho^2 - \rho)\Delta \\
	&\implies& d\Phi(\itV_k;\matU e')  + \Delta \ge \rho (d\Phi(\itV_k; \matU e') + \Delta)~.
	\end{eqnarray*}
	Hence, using the estimate $|\calN| \le (10d(1+\rho)/\Delta(\rho^2 - \rho))^r \le (20 d/(\rho - 1))^r$ for $\Delta, \rho \ge 1$, we have
	\begin{eqnarray*}
	\Pr[ \calE(  \rho^2, \Delta,k_{\max})^c] &\le& |\calN|\sup_{e \in \calN} \Pr[\exists k \in [k_{\max}]: d\Phi(\itV_k;\matU e) + \Delta \ge \rho (d\Phi(\itV_k; \matU e) + \Delta) ]\\
	&\overset{\text{Lemma~\ref{lem:one_fixed_e}}}{\le}& |\calN| \frac{d^2 }{\rho - 1}\exp\left\{ \frac{-\boldlam^3(\boldlam - 1)\Delta}{2}  \right\}\\
	&\le& (20 d/(\rho - 1))^{r+2}\exp\left\{ \frac{-\boldlam^3(\boldlam - 1)\Delta}{2}  \right\},
	\end{eqnarray*}
	as needed.

\subsection{Proof of Lemma~\ref{lem:one_fixed_e}}

	Fix $\rho \ge \boldlam^3 c_{d,r}$, and set $\tau_0 = 0$, and let $\tau_{i+1} = \rho \tau_{i} + (\rho - 1) \Delta$, so that $\tau_{i+1} + \Delta \le \rho(\tau_{i} + \Delta)$. Finally, let $M = \inf \{i : \tau_i \ge d\}$, and observe that we can bound $M \le \frac{d}{(\rho - 1)\Delta} \le \frac{d}{\rho -1}$. Moreover, note that $\Delta \ge \frac{\rho(2k_{\max}+2)}{(\rho - 1)^3}$ and $\rho \ge 1$ implies that
	\begin{eqnarray}\label{eq:Delta_eq}
	\Delta \ge \frac{\rho^2(2k_{\max}+2)}{(\rho^2 - 1)(\rho - 1)^2} \iff \frac{(\rho-1)^2}{\rho^2} \ge \frac{2k_{\max}+2}{\Delta(\rho - 1)^2} \iff 1-\sqrt{(2k_{\max}+2)/\Delta(\rho^2-1)}.
	\end{eqnarray}
	\begin{eqnarray*}
	&&\Pr[ \Phi(\itV_k;\matU e) + \Delta \ge \rho (\Phi(\itV_k; \matU e) + \Delta)] \\
	&\le& \sum_{i=1}^{M} \Pr[ \Phi(\itV_k;\matU e) + \Delta \ge \rho (\tau_{i-1} + \Delta) \cap  \Phi(\itV_k;\matU e) + \Delta \in [\tau_{i-1},\tau_i]\})]\\
	&\overset{\tau_{i} + \Delta \le \rho(\tau_{i-1} + \Delta)}{\le}& \sum_{i=1}^{M} \Pr[ \Phi(\itV_k;\matU e) + \Delta \ge \rho^2 (\tau_{i} + \Delta) \cap  \Phi(\itV_k;\matU e) + \Delta \in [\tau_{i-1},\tau_i]\})]\\
	&\le& \sum_{i=1}^{M} \Pr[ \Phi(\itV_k;\matU e) + \Delta \ge \rho^2 (\tau_i + \Delta) \cap  \Phi(\itV_k;\matU e) \le \tau_i - \Delta\})]\\
	&\le & M\max_{1\le i \le M}\exp\left\{ \frac{\boldlam - 1}{2}  \left( \boldlam^3 (\tau_i - \Delta)  - \frac{\left(\sqrt{ \frac{\rho^2}{c_{d,r}}(\tau_i + \Delta) - \Delta)}-\sqrt{2k+2}\right)^2}{1+\eta}\right) \right\}\\
	&\le& M\max_{1\le i \le M}\exp\left\{ \frac{\boldlam - 1}{2}  \left( \boldlam^3(\tau_i - \Delta)  -  (\frac{\rho^2}{c_{d,r}}(\tau_i + \Delta) - \Delta)(1-\sqrt{(2k+2)/(\rho^2(\tau_i + \Delta) - \Delta)}\right)\right\}\\
	&\le& M\max_{1\le i \le M}\exp\left\{ \frac{\boldlam - 1}{2}  \left( \boldlam^3(\tau_i - \Delta)  -  (\frac{\rho^2}{c_{d,r}}(\tau_i + \Delta) - \Delta)(1-\sqrt{(2k+2)/\Delta(\rho^2-1)}\right)\right\}\\
	&\overset{\text{Equation}~\eqref{eq:Delta_eq}}{\le}& M\max_{1\le i \le M}\exp\left\{ \frac{\boldlam - 1}{2}  \left( \boldlam^3(\tau_i - \Delta)  -  \frac{\rho}{c_{d,r}}(\tau_i + \Delta) - \Delta)\right)\right\}\\
	&\overset{\rho \ge \boldlam^{3}c_{d,r}}{\le}& M\max_{1\le i \le M}\exp\left\{ \frac{\boldlam - 1}{2}  \left( \boldlam^3(\tau_i - \Delta)  -  \boldlam^3((\tau_i + \Delta) - \Delta)\right)\right\}\\
	&=& M\exp\left\{ \frac{\boldlam - 1}{2}  \left( (1-2\boldlam^3)\Delta)\right)\right\} \le M\exp\left\{ \frac{(\boldlam - 1)\boldlam^3}{2}  \left( (1-2\boldlam^3)\Delta)\right)\right\}.
	\end{eqnarray*}
	Union bounding over $k_{\max} \le d$ proves the lemma.

\subsection{Proof of Proposition~\ref{prop:rk_k_info_th_prop}\label{sec:rk_k_info_th_prop}}
Let $P_e := I - ee^\top$ be the projection onto the orthongal complement of $e$. It suffices to show that for a fixed $e \in \R^r$, and any fixed $U \in \Stief(d,r)$, then the conditional probability  $\Exp_{\matU}[\Prit^k_{\matU}[\{ \Phi(\itV_k;\matU e) \le \tau_k/d\} \cap \{ \Phi(\itV_{k+1};\matU e)> \tau_{k+1}/d\} \big{|} \{(\matU-U)P_e = 0\}] $ is also upper bounded by the right hand side of the display in Propostion~\ref{prop:rk_k_info_th_prop}. Observe that on the event $\{(\matU-U)P_e = 0\}$, $\matU(I-ee^\top)$ is fixed, but $\matU e$ is distributed uniformly on the of unit vectors orthogonal to the image of $\matU(I - ee^\top)$; let's denote this set $\calS_{\matU,e}$. Hence, on $\{(\matU-U)P_e = 0\}$, $\matM - \boldlam \matU(I-ee^\top)$ has the same distribution as $\matW + \boldlam\widetilde{\matu}$, where $\matu \overset{unif}{\sim} \calS_{\matU,e}$. Hence, an algorithm which achieves $\Exp_{\matU}[\Prit^k_{\matU}[\{ \Phi(\itV_k;\matU e) \le \tau_k/d\} \cap \{ \Phi(\itV_{k+1};\matU e)> \tau_{k+1}/d\} \big{|} \{(\matU-U)P_e = 0\}] \ge p$ implies the existence of an algorithm which achieves $\Exp_{\matu \sim \calS_{\matU,e}}[\Prit^k_{\matu}[\{ \Phi(\itV_k;\matu) \le \tau_k/d\} \cap \{ \Phi(\itV_{k+1};\matu)> \tau_{k+1}/d\} ] \ge p$, so it suffices to bound this latter probability. By Proposition~\ref{prop:info_th_rkone} with $\calD$ being the uniform distribution on $\calS_{\matU,e}$, we have 
\begin{multline}\label{eq:calSmatue_eq}
\Exp_{\matu \sim \calS_{\matU,e}}[\Prit^k_{\matu}[\{ \Phi(\itV_k;\matu) \le \tau_k/d\} \cap \{ \Phi(\itV_{k+1};\matu)> \tau_{k+1}/d\}] \le\\
 \left(\Exp_{u \sim \calS_{\matU,e}}\Expit_{0}\left[\left(\frac{\rmd\Prit_{u}^k}{\rmd \Prit_0^k}\right)^{1+\eta} \I(\{ \Phi(\itV_k;u) \le \tau_k/d\}\})\right] \cdot \sup_{V \in \Stief(d,k+1)} \Pr_{u \sim \calS_{\matU,e}}[\Phi(V;u)> \tau_{k+1}/d]^\eta   \right)^{\frac{1}{1+\eta}}.
\end{multline}
By Proposition~\ref{prop:likelihood_info}, the information term $\Exp_{u \sim \calS_{\matU,e}}\Expit_{0}\left[\left(\frac{\rmd\Prit_{u}^k}{\rmd \Prit_0^k}\right)^{1+\eta} \I(\{ \Phi(\itV_k;u) \le \tau_k/d\}\})\right] $ is at most $ \exp(\frac{\eta(1+\eta)}{2}\boldlam^2 \tau_k$. On the other hand, since $\calS_{\matU,e}$ is isomorphic to the $d-r-1$ sphere, Lemma~\ref{lem:small_ball_prob} implies
\begin{eqnarray*}
\sup_{V \in \Stief(d,k+1)} \Pr_{u \sim \calS_{\matU,e}\}}[ u^\top V^\top V u \ge \tau_{k+1}/(d-r-1)] \le \exp\left\{-\frac{1}{2}\left(\sqrt{\tau_{k+1}}-\sqrt{2(k+1)}\right)^2\right\},
\end{eqnarray*}
where we have had to replace $\tau_{k+1}/d$ by $\tau_{k+1}/(d-r-1) = c_{d,r} \cdot \tau_{k+1}/d$ to account for the change in dimension. Putting together these two estimates into Equation~\eqref{eq:calSmatue_eq} gives us the first display in Propostion~\ref{prop:rk_k_info_th_prop}.

\subsection{Proof of Lemma~\ref{lem:Determinant_Growth_Lemma}}

	Note that $\calE(\lamtil,\Delta,k_{\max})$ can be rexpressed as saying that, for all $k \in [1,\dots,k_{\max}]$, one has the
	\begin{eqnarray*}
	\matU^\top \itV_k \itV_k^\top \matU  + \Delta I_r &\preceq& \lamtil (\matU^\top \itV_{k-1} \itV_{k-1}^\top \matU  + \Delta I_r) \iff \\
	\matU^\top \itV_{k-1} \itV_{k-1}^\top \matU  + \Delta I_r  + \matU^\top \vk \vkT \matU &\preceq& \lamtil (\matU^\top \itV_{k-1} \itV_{k-1}^\top \matU  + \Delta I_r) \iff \\
	\matU^\top \vk \vkT \matU &\preceq& (\lamtil -1 )(\matU^\top \itV_{k-1} \itV_{k-1}^\top \matU  + \Delta I_r). 
	\end{eqnarray*}
	We now invoke a claim from linear algebra:
	\begin{claim}\label{claim:Lower_Claim} Let $u \in \R^r$, $M \succ 0$, and $t > 0$. Then $uu^\top \preceq tM \iff u^\top M^{-1} u \le t$. 
	\end{claim}
	\begin{proof}
	$uu^\top \prec tM \iff M^{-1/2}uu^\top M^{-1/2} \le tI \iff \|M^{-1/2}uu^\top M^{-1/2}M\|_2 \le t$. Since $\|M^{-1/2}uu^\top M^{-1/2}\|_2 = u^\top M^{-1} u$ as $u$ is a vector, the claim follows.
	\end{proof}
	Applying the above claim under $\calE(\lamtil,\Delta,k_{\max})$ with $u = \matU^\top \vk$, $M = \matU^\top \itV_{k-1} \itV_{k-1}^\top \matU  + \Delta I_r$, and $t = \lamtil - 1$, we have that 
	\begin{eqnarray*}
	(\matU^\top \vk)^\top (\matU^\top \itV_{k-1} \itV_{k-1}^\top \matU  + \Delta I_r)^{-1} (\matU^\top \vk) \le \lamtil - 1.
	\end{eqnarray*}
	Hence,
	\begin{eqnarray*}
	&&\det( \matU^\top \itV_k \itV_k^\top \matU  + \Delta I_r)\\
	 &=& \det( \matU^\top \itV_{k-1} \itV_{k-1}^\top \matU  + \Delta I_r + \matU^\top \vk \vkT \matU )\\
	 &=& \det( \matU^\top \itV_{k-1} \itV_{k-1}^\top \matU  + \Delta I_r) \\
	 &=& \det( \matU^\top \itV_{k-1} \itV_{k-1}^\top \matU  + \Delta I_r) \cdot \det( 1 + (\matU^\top \vk)^\top (\matU^\top \itV_{k-1} \itV_{k-1}^\top \matU  + \Delta I_r)^{-1} (\matU^\top \vk) )\\
	 &\le& \det( \matU^\top \itV_{k-1} \itV_{k-1}^\top \matU  + \Delta I_r) \cdot \det( 1 + (\lamtil - 1)) = \lamtil\det( \matU^\top \itV_{k-1} \itV_{k-1}^\top \matU  + \Delta I_r).
	\end{eqnarray*}

\begin{comment}

Hence, if we pick $\delta_{k+1} \in (0,1)$, $\eta = 1 - \boldlam$, and have $\tau_{k+1}/c_{d,r} \ge \boldlam^4 \tau_{k} + \frac{2}{\boldlam - 1} \left(\boldlam \log(1/\delta_{k+1}) + \sqrt{2k+2}\right)$, then the probability in the above display is bounded by at most $\delta$. 
A straightforward modification of the computations in the proof of theorem [] to accounrt for the constrant $c_{d,r}$ yields that, if $\tau1_(\delta) =  :=\frac{\boldlam^2}{(\boldlam - 1)^2}(32\log (1/\delta) + \frac{8}{e} \max\{1,\log(\boldlam)^{-1}\}$, then for any $e \in \R^r$,
\begin{eqnarray}\label{eq:one_e_equation}
\Pr[\exists k \ge 1:  d\Phi(\itV_k; \matU e) \ge c_{d,r}^k\boldlam^{4k}\tau_1(\delta_*)] \le \frac{\delta\pi^2}{6}
\end{eqnarray}
\end{comment}

%!TEX root = main.tex

\section{Information-Theoretic Tools\label{sec:info_th_tools}}

In this section, we prove Theorem~\ref{Fano_sub_distr}, a generalization of the data-processing style lower bounds for statistical estimation (e.g. Fano's inequality). Our techniques differ from the existing art in consideing un-normalized, finite measures, rather than normalized probability distributions. 

Theorem~\ref{Fano_sub_distr} is derived as a special case of Theorem~\ref{thm:gen_fano}, which extends generalized Data-Processing style Lower Bound (Theorem 2 in~\cite{chen2016bayes}) to the case where the measures of interest are not necessarily normalized. Along the way, we generalize the notion of $f$-divergences (~\cite{csiszar1972class} to the non-normalized setting, and establish that many key properties - notably the Data Processing inequality - still hold. 

%!TEX root = main.tex

\subsection{General Data-Processing Lower Bounds for Unnormalized Measures\label{sec:gen_data_proc}}

	In this section, we will prove a more general result, Theorem~\ref{thm:gen_fano}, from which we will derived Theorem~\ref{Fano_sub_distr} as a consequence.
	In order to state and prove our more general theorem, and then specialize to our case of interest, we need to introduce the object of ``f-divergences'' to measure the similarity between two measures. f-divergences between \emph{probability distributions} have a long history in information theory, coding theory~\cite{csiszar1972class,guntuboyina2011lower,liese2012phi}, and statistical lower bounds, but we define them here is a slightly more general fashion so as to be ammenable to describe distances between non-normalized measures\footnote {Typically one requires the divergence function $f$ to satisfy $f(1) = 0$, but we shall not need this normalization}:

\begin{defn}\label{f_div_def} For a finite, non-negative measure $\mu$ and finite positive measure $\nu$ over the class $(\calX,\calF)$, and a convex $f:(0,\infty) \to \R$, we define the (generalized) $f$-divergence between $\mu$ and $\nu$ as
\begin{eqnarray}
D_{f}(\mu,\nu) := \int_{x \in \calX: \rmd\nu(x) > 0} f\left(\frac{\rmd\mu}{\rmd\nu}\right)\rmd\nu + \mu\left(\{\rmd\nu = 0\}\right)\cdot f'(\infty)
\end{eqnarray}
with the notation $f'(\infty):= \lim_{t \to \infty} f(t)/t$, and $0 \cdot f'(\infty) = 0$.
\end{defn}
Here, $\rmd\nu$ and $\rmd\mu$ are understood as Radon-Nikodyn derivates (see, e.g.~\cite{kallenberg2006foundations}). Note that the case whenre $\mu$ is absolutely continuous with respect to $\nu$ (written $\mu \ll \nu$), we can disregard the term $\mu\left(\{\rmd\nu = 0\}\right)\cdot f'(\infty)$. Critically, this generalization of $f$-divergences preserves the ``Data-Processing Inequality'' which holds in the normalized case:
\begin{thm}[Generalized Data Processing Inequality]\label{thm:data_proc} Let $\mu,\nu$ be non-negative measures on a space $(\calX,\calF)$, and let $f:(0,\infty) \to \R$ be a convex function. Then, given a measure space $(\calY,\calF_{\calY})$ and a measurable map~\footnote{more generally, a Markov Transition Kernel} $\Gamma: \calX \to \calY$, 
	\begin{eqnarray}
	D_f(\mu,\nu) \ge D_f(\mu \Gamma^{-1},\nu \Gamma^{-1}),
	\end{eqnarray}
	where $\mu \Gamma^{-1}$ denotes the pull back measure $\forall B \in \calF_{\calY}: (\mu \Gamma^{-1})(B) = \mu (\Gamma^{-1}(B))$. 
\end{thm}
We will be particularly interested in the case where $\Gamma$ is just the indicator function of an event $\calY = \{0,1\}$ is just a binary space, and $\Gamma(\matx) = \I(\matx \in A)$ is an indicator function. In this case, the above data processing inequality immediately yields the following corollary:
\begin{cor}[Binary Data-Processing]\label{phi_cor} Let $\mu,\nu$ be non-negative measures on a space $(\calX,\calF)$. Then for all $A \in \calF$,
\begin{eqnarray}
D_f(\mu,\nu) \ge \phi_f(\mu(A),\nu(A);\mu(\calX),\nu(\calX)),
\end{eqnarray}
where, for $a \in [0,p]$, $b \in [0,q]$,
\begin{eqnarray}
\phi_f(a,b;p,q) := bf\left(\frac{a}{b}\right) + (q-b)\cdot f\left(\frac{p-a}{q-b}\right), 
\end{eqnarray}
ss the $f$ divergence between the meansure measures on $\{0,1\}$ which place mass $a$ (resp.\ $b$) on $1$, and $p-a$ (resp.\ $q-b$) on $0$, and $b = 0$ or $b = q$ is understood by taking the limits $b \to 0^+$ and $b \to q^-$.
\end{cor}
	We are now in a position to state our main theorem:
	
	\begin{thm}\label{thm:gen_fano}Consider the setting of Theorem~\ref{Fano_sub_distr}, but where 
	 $f$ is an arbitrary convex function on $(0,\infty)$, and both $\nu$ and $\{\mu_\theta\}$ are an arbitrary value of finite measure satisfying $\Exp_{\theta \sim \calP}\mu_{\theta}(\calX) > 0$. Then, either one of the two hold
	\begin{eqnarray}\label{eq:Gen_Fano_Eq}
	 V_{\opt} < \Exp_{\theta \sim \calP}\mu_{\theta}(\calX) \cdot V_0 & \text{or} & \Exp_{\theta \sim \calP} D_f(\mu_{\theta},\nu) \ge \phi_f(V_{\opt},\nu(\calX)\cdot V_0;(\Exp_{\theta \sim \calP}\mu_{\theta}(\calX)),\nu(\calX)).
	\end{eqnarray}
	\end{thm}
	Before proving the above theorem, we can derive Theorem~\ref{Fano_sub_distr} as a special case. For ease of notation, we introduce the shorthand $|\mu| = \mu(\calX)$.
	\begin{proof}[Proof of Theorem~\ref{Fano_sub_distr}]
	Since $f$ is convex, $D_f$ consitutes a valid $f$-divergence in the sense of Definition~\ref{f_div_def}. In Theorem~\ref{Fano_sub_distr}, we have that $\Exp_{\theta \sim \calP} |\mu_\theta| \le 1$, so that $\Exp_{\theta \sim \calP}|\mu_{\theta}| \cdot V_0 \le V_0 \le V_\opt$, so we have 
	\begin{eqnarray*}
	\Exp_{\theta \sim \calP} D_f(\mu_{\theta},\nu) \ge \phi_f(V_{\opt},|\nu|\cdot V_0;\Exp_{\theta \sim \calP}\mu_{\theta}|,|\nu|).
	\end{eqnarray*}
	Since $\mu_{\theta} \ll \nu$, we have
	\begin{eqnarray*}
	\Exp_{\theta \sim \calP} D_f(\mu_{\theta},\nu) = \Exp_{\theta \sim \calP} \Exp_{\nu} f(\frac{\rmd \mu_\theta}{\rmd\nu}).
	\end{eqnarray*}
	moreover, recalling that $f(x)$ is nonegative
	\begin{eqnarray*}
	\phi_f(V_{\opt},|\nu|\cdot V_0;\Exp_{\theta \sim \calP}\mu_{\theta}|,|\nu|) &=& |\nu| V_0 f(\frac{V_\opt}{|\nu| V_0}) + (1 - |\nu |V_0)\cdot f(\frac{|\Exp_{\theta \sim \calP}\mu_{\theta}| - V_\opt}{|\nu|(1 - V_0}\\
	&\ge& |\nu| V_0 f(\frac{V_\opt}{|\nu| V_0}).
	\end{eqnarray*}
	If $|\nu| = 1$, then the above is $|\nu|V_0 f(\frac{V_\opt}{|\nu| V_0})$. If, on the other hand, $x \mapsto xf(1/x)$ is non-increasing on $(0,\infty)$, so is $x \mapsto x f(p/x)$ for any fixed $p > 0$. Hence, if $|\nu| \le 1$, then $|\nu|V_0 f(\frac{V_\opt}{|\nu| V_0}) \ge V_0 f(\frac{V_\opt}{ V_0})$. In either case, we conclude that
	\begin{eqnarray*}
	\Exp_{\theta \sim \calP} \Exp_{\nu} f(\frac{\rmd \mu_\theta}{\rmd\nu}) \ge V_0 f(\frac{V_\opt}{ V_0})~,
	\end{eqnarray*}
	as needed.
	\end{proof}{}
	Before proving Theorem~\ref{thm:gen_fano}, we need one last regularity lemma, proved in Section~\ref{sec:phi_reg} regarding the function $\phi(a,b;p,q)$; this mirrors Lemma [] in~\cite{chen2016bayes}:
	\begin{lem}\label{lem:_phi_regularity} For $a \in [0,p]$ and $b \in [0,q]$, the mapping $(a,b) \mapsto \phi(a,b;p,q)$ is convex, and thus continuous on $(a,b)$. Hence, for a fixed $b \in [0,q]$, $a \mapsto \phi_f(a,b;p,q)$ is minimized when $a = (p/q)b$, and is therefore nondecreasing for $(q/p)a \ge b$. As a function of $b$ for fixed $a \in [0,p]$, $b \mapsto \phi_f(a,b;p,q)$ is minimized when $b = a(q/p)$ and is therefore nonincreasing for $(p/q)b \le a$
	\end{lem}
	We are now in place to prove Theorem~\ref{thm:gen_fano}:

	\begin{proof}[Proof of Theorem~\ref{thm:gen_fano}]
	We follow along the lines of the proofs of Lemma 3 and Theorem 2~\cite{chen2016bayes}, but first we introduce some notation.  Further $\calP \otimes \nu$ denote the product measure between $\calP$ and $\nu$, and let $\calP * \{\mu_{\boldtheta}\}$ denote the coupled measure with density $\rm(\calP * \{\mu_{\boldtheta})\}(\theta,x) = \rmd\calP(\theta)\cdot \rm\mu_{\theta}(x)$. Also, given a measure $\eta$ on $(\Theta,\calG) \times (\calX, \calF)$, define the event 
	\begin{eqnarray}
	A_{\fraka} := \{(\theta,x): \calI(\fraka(x),\theta) = 0\}.
	\end{eqnarray}
	Defining  the total masses $p := \Exp_{\boldtheta \sim \calP}|\mu_{\boldtheta}| $ and $q := |\nu| $, 
	\begin{eqnarray*}
	\Exp_{\theta \sim \calP} D_f(\mu_{\boldtheta},\nu) &=& \int f\left(\frac{\rmd\mu_{\boldtheta}}{\rmd\nu}\right)d(\calP \otimes \nu)\\
	&=& \int f\left(\frac{\rmd\calP \cdot \rmd\mu_{\boldtheta}}{d\calP \cdot \rmd\nu}\right)\rmd(\calP \otimes \nu)\\
	&=& D_f\left(\calP * \{\mu_{\boldtheta}\},\calP \otimes \nu\right)\\
	&\overset{(i)}{\ge}& \phi_f\left((\calP * \{\mu_{\boldtheta}\})(A_\fraka),(\calP \otimes \nu)(A_\fraka);|\calP * \{\mu_{\boldtheta}\}|,|\calP \otimes \nu|\right)\\
	&\overset{(i)}{=}& \phi_f\left((\calP * \{\mu_{\boldtheta}\})(A_\fraka),(\calP \otimes \nu)(A_\fraka);p,q\right),
	\end{eqnarray*}
	where $(i)$ follows from the Binary Data-Processing Inequality (Corollary~\ref{phi_cor}), and for $(ii)$ used definitions $p = |\calP * \{\mu_{\boldtheta}\}|$ and $q = |\calP \otimes \nu| = |\calP||\nu| = |\nu|$. To wrap up, suppose that $V_{\opt} > p V_0$. We first note that $(\calP \otimes \nu)(A_\fraka) \le  |(\calP \otimes \nu)|\cdot V_0= q V_0$, since $\matx$ and $\boldtheta$ are independent under $\calP \otimes \nu$. Moreover, for any $\epsilon > 0$, there exists a decision rule $\fraka$ for which
	\begin{eqnarray*}
	p = |\calP * \{\mu_{\boldtheta}\}| \ge V_{\fraka}(\calP * \{\mu_{\boldtheta}\}) > V_{\opt} - \epsilon.
	\end{eqnarray*}
	Taking $\epsilon$ small enough $V_\opt - \epsilon > pV_0$, we have that 
	\begin{eqnarray*}
	(\calP * \{\mu_{\boldtheta}\})(A_\fraka) > V_\opt - \epsilon > pV_0 = \frac{p}{q}(qV_0) \ge \frac{p}{q}\cdot (\calP \otimes \nu)(A_\fraka).
	\end{eqnarray*}
	By the second part of Lemma~\ref{lem:_phi_regularity}, applied first to the '$b$' argument and then to the `$a$' argument, we have 
	\begin{eqnarray*}
	\phi_f((\calP * \{\mu_{\theta}\})(A_\fraka);(\{\calP \otimes \nu\})(A_\fraka);p,q) &\ge& \phi_f((\calP * \{\mu_{\theta}\})(A_\fraka),qV_0;p,q)\\
	 &\ge& \phi_f(V^\opt - \epsilon,qV_0;p,q).
	\end{eqnarray*} 
	Since $\phi_f(a,b;p,q)$ is convex (Lemma~\ref{lem:_phi_regularity}), and therefore continuous, in its `$a$' argument for $a \in [0,p]$, and since $V^* \le p$, taking $\epsilon \to 0$ concludes. 
	\end{proof}

\subsection{Proofs of Data Processing Inequalities and Associated Lemmas }

In this subsection, we prove the Data-Processing Inequality (Theorem~\ref{thm:data_proc}), Binary Data-Processing Inequality (Corollary~\ref{phi_cor}), and the regularity lemma regarding $\phi$ (Lemma~\ref{lem:_phi_regularity}). Before we begin, we first argue that our generalization of $f$ divergences to satisfies two useful regularity properties which hold in the normalized case, and two useful properties which relate an un-normalized divergence to a normalized one:
\begin{lem}\label{lem:fDivProperties} Let $\mu,\nu$ be two finite positive measures on a space $(\calX,\calF)$, with $|\mu|  = \mu(\calX)$ and $|\nu| = \nu(\calX)$. Then $D_f(\mu,\nu)$ satisfies the following properties:
\begin{enumerate} 
\item \textbf{Convexity:} $D_{f}(\mu,\nu)$ is jointly convex in $\mu$ and $\nu$ over the convex set $(\mu,\nu): \mu \ll \nu$
\item \textbf{Distance-Like:}  $D_{f}(\mu,\nu) \ge |\nu|f(|\mu|/|\nu|)$, which is attained when $\rmd\mu(x)/ \rmd\nu(x) = |\mu|/|\nu|$
\item \textbf{Normalization:} Define $f(x;p,q) = qf(\frac{p}{q}x)$. Then,
\begin{eqnarray}
D_{f}(\mu,\nu) = D_{f(x;|\mu|,|\nu|)}(\mu/|\mu|,\nu/|\nu|).
\end{eqnarray}
\item \textbf{Linearity} $D_{\beta f+\alpha}(\mu,\nu) = \alpha|\nu| + \beta D_{f}(\mu,\nu)$
\end{enumerate}
\end{lem}
This lemma is proved in subsection~\ref{Sec:fDivProp}. We can now prove the generalized Data-Processing inequality (Theorem~\ref{thm:data_proc}):
\begin{proof}[Proof of Theorem~\ref{thm:data_proc}] In the case when $f$ is convex and $f(1) = 0$ and $\mu,\nu$ are both probability distributions, Theorem~\ref{thm:data_proc} just re-iterates the classica data-processing inequality (Theorem 3.1 in Liese~\cite{liese2012phi}). The inequality can be extended to an uncentered, convex $f$ where $f(1)$ is not necessarily zero, and normalized $\mu,\nu$, invoking Part 4 of Lemma~\ref{lem:fDivProperties}:
	\begin{eqnarray*}
	D_{f}(\mu,\nu) &\overset{\text{Lemma}~\ref{lem:fDivProperties}}{=}& D_{f - f(1)}(\mu,\nu) + f(1)|\nu|\\
	&\ge& D_{f - f(1)}(\mu\Gamma^{-1},\nu\Gamma^{-1}) + f(1)|\nu| \quad \text{ classical data processing} \\
	&=& D_{f - f(1)}(\mu\Gamma^{-1},\nu\Gamma^{-1}) + f(1)|\nu\Gamma^{-1}| \quad \text{ ($\Gamma$ preserves total mass)} \\
	&\overset{\text{Lemma}~\ref{lem:fDivProperties}}{=}& D_{f}(\mu\Gamma^{-1},\nu\Gamma^{-1}). 
	\end{eqnarray*}
	To generalize to arbitrary finite, positive measures, we note that the function $f(t;|\mu|,|\nu|)$ as defined in part $3$ of Lemma~\ref{lem:fDivProperties} is convex, and thus 
	\begin{eqnarray*}
	D_{f}(\mu,\nu) &\overset{\text{Lemma}~\ref{lem:fDivProperties}}{=}& D_{f(;|\mu|,\nu)}(\mu/|\mu|,\nu/|\nu|)\\
	&=& D_{f(;|\mu|,\nu)}(\frac{\mu}{|\mu|}\Gamma^{-1},\frac{\nu}{|\nu|}\Gamma^{-1}) \text{  uncentered data processing (above)} \\
	&\overset{\text{Lemma}~\ref{lem:fDivProperties}}{=}& D_{f}(\mu\Gamma^{-1},\nu\Gamma^{-1}).
	\end{eqnarray*}
	\end{proof}

\subsubsection{Proof of Lemma~\ref{lem:_phi_regularity} \label{sec:phi_reg}}
	The first point follows since, if $f$ is a convex function, the perspective map $(a,b) \mapsto bf(a/b)$ is convex (see \cite{boyd2004convex}). The second point follows from applying the second point of Lemma~\ref{lem:fDivProperties}, and noting that $1$-d convex functions are non-increasing to the left (resp.\ non-decreasing to the right) of their minimizers.

\subsubsection{Proof of Lemma~\ref{lem:fDivProperties}\label{Sec:fDivProp}}
The set $\{(\mu,\nu): \mu \ll \nu\}$ is convex, since if $\alpha \nu_1(A) + (1-\alpha)\nu_2(A) = 0$, then $\nu_1(A) = \nu_2(A) = 0$, and thus if $\mu_1 \ll \nu_1$ nad $\mu_2 \ll \nu_2$, then $\alpha \mu_1(A) + (1-\alpha)\mu_2(A) = 0$. Moreover, the perspective map $(x,y) \to yf(x/y)$ is jointly for convex $f$~\cite{boyd2004convex}, so that $\int f(\frac{\rmd\mu}{\rmd\nu})\rmd\nu$ is jointly convex in each argument. 

	For the second point , we see that that, by Jensen's inequality:
	\begin{eqnarray*}
	\int f(\frac{\rmd\mu}{\rmd\nu})\rmd\nu &=& |\nu|\int f(\frac{\rmd\mu}{\rmd\nu})\frac{\rmd\nu}{|\nu|} \ge |\nu|f( \int \frac{\rmd\mu}{\rmd\nu}\frac{\rmd\nu}{|\nu|} ) \\
	&=& |\nu|f( \frac{1}{|\nu|}\int \rmd\mu ) =  |\nu|f(\frac{|\mu|}{|\nu|}), 
	\end{eqnarray*}
	so the result holds as long as $f'(\infty) \ge 0$.

	Third, let $g(t) = f(t;p,q) = |\nu|f(t\frac{|\mu|}{|\nu|})$.Then $g'(\infty) = f'(\infty) \cdot |\mu|/|\nu| \cdot |\nu| = |\mu|f'(\infty)$. Thus,
	\begin{eqnarray*}
	D_{f}(\mu,\nu) &=& \int f(\frac{\rmd\mu}{\rmd\nu})d\nu + \mu(\{\rmd\nu = 0\})f'(\infty) \\
	&=& \int |\nu|f(\frac{|\mu|}{|\nu|}\cdot\frac{\rmd(\mu/|\mu|)}{\rmd(\nu/|\nu|)})\cdot \rmd\nu/|\nu| + (\frac{\mu}{|\mu|})(\{d\nu = 0\})\cdot |\mu|f'(\infty)\\
	&=& \int g(\frac{\rmd(\mu/|\mu|)}{\rmd(\nu/|\nu|)})\rmd\nu/|\nu| + \frac{\mu}{|\mu|}(\{\rmd\nu = 0\})g'(\infty)\\
	&=& D_{g}(\mu/|\mu|,\nu/|\nu|),
	\end{eqnarray*}
	as needed. For the fourth point point, note that for any constant $\alpha, (f+\alpha)'(\infty) = f(\infty)$. Thus,
	\begin{eqnarray*}
	D_{f+\alpha}(\mu,\nu) &=& \int \{f(\frac{\rmd\mu}{\rmd\nu})+\alpha\}d\nu + \mu(\{\rmd\nu = 0\})(f+\alpha)'(\infty)\\
	&=& \alpha |\nu| + \int f(\frac{\rmd\mu}{\rmd\nu})\rmd\nu + \mu(\{\rmd\nu = 0\})(f)'(\infty) = \alpha |\nu| + D_{f}(\mu,\nu).
	\end{eqnarray*}
	Similarly, since $(\beta f)'(\infty) = \beta f'(\infty)$, one has
	\begin{eqnarray*}
	D_{\beta f}(\mu,\nu) &=& \int \{\beta f(\frac{\rmd\mu}{\rmd\nu})+\alpha\}\rmd\nu + \mu(\{\rmd\nu = 0\})(\beta f)'(\infty)\\
	&=& \beta \int f(\frac{\rmd\mu}{\rmd\nu})\rmd\nu + \beta \mu(\{\rmd\nu = 0\})(f)'(\infty) = \beta D_{f}(\mu,\nu).
	\end{eqnarray*}

\part{Random Matrix Theory}
Before continuing, we introduce some additional notation. Given a map $\Psi: \R^{n_1 \times m_1} \to \R^{n_2 \times m_2}$, we let $\Lip(\Psi)$ denote its Lipschtitz constant as a map between Euclidean spaces endowed with the Euclidean (Frobenius) norm. In particular, if $\Psi: \R^{n_1 \times m_1} \to \R^{n_2 \times m_2}$, then $\|\Psi(X)-\Psi(Y)\|_{\F} \le \Lip(\Psi)\|X-Y\|_{\F}$.  We let $\fraki^2 = -1$, and given $a,b \in \R$, we let $\Re(a+b\fraki) = a$, $\Im(a+b\fraki) = b$, $|a+b\fraki| = \sqrt{a^2 + b^2}$, and $\overline{a+b\fraki} = a - b\fraki$. We say that  $\matX \sim \SG(d)$ if the entries $\{\matX_{ij}\}_{1 \le i,j \le d}$ are independent, $\calN(0,1)$ random variables. Observe that if $\matW \sim \GOE(d)$ and $\matX \sim \SG(d)$, then $\matW$ and $\frac{1}{\sqrt{2d}}(\matX + \matX^\top)$ have the same distribution. Finally, for a given $\lambda > 0$, we will let $\matM = \matW + \lambda \matU\matU^{\top}$, where $\matW \sim \GOE(d)$ and $\matU \sim \calO(d,k)$ are independent. 

%!TEX root = main.tex
\section{Concluding the Proof of Theorem~\ref{thm:main_spec_thm}}

\subsection{Proof of Theorem~\ref{thm:main_spec_thm}\label{sec:main_spec_thm_proof}}
We begin by establishing the following lemma, proved in Section~\ref{sec:stiel_lem_proof}, to lower bound $\stielt(\alow)$ and upper bound $\stielt(\aup)$. 
\begin{lem}\label{lem:stielt_lem} Suppose that $\epsilon \le \boldgap \cdot \min\{1/2, \frac{1}{\boldlam^2 - 1}\}$. Then, 
\begin{eqnarray*}
\stielt((\boldlam+\boldlam^{-1})(1-\epsilon )) \ge \frac{1}{\boldlam}\left(1 + \frac{\epsilon }{2\sqrt{2\boldgap}}\right) & \text{and} & \stielt((\boldlam + \boldlam)^{-1}(1 + \epsilon)) \le \frac{1}{\boldlam}\left(1 - \frac{\epsilon}{4\sqrt{\boldgap}})\right)
\end{eqnarray*}
\end{lem}
In light of the above lemma, we pick
	\begin{eqnarray}
	\alow = (\boldlam + 1/\boldlam)(1-\epsilon) & \text{and} & \aup  = (\boldlam + 1/\boldlam)(1-\epsilon)~,
	\end{eqnarray}
	where $\epsilon \le \boldgap \cdot \min\{1/2, \frac{1}{\boldlam^2 - 1}\}$. We then show that for the appropriate choice of $d$, we can combine Theorem~\ref{thm:stiel_thm} and Proposition~\ref{prop:Hanson_wright_proposition} to show that with high probability
\begin{eqnarray*}
\max_{a \in \{\alow,\aup\}}\|\matU^\top (aI - \matW)^{-1} \matU - \stielt(a)I\|_{\op} < \frac{\epsilon }{4\lambda\sqrt{\gap}}~ 
\end{eqnarray*}
Since then, for $a = \alow$, 
\begin{multline}\label{eq:specWTS}
 \matU^\top (\alow I - \matW)^{-1} \matU \succ  \stielt(\alow)I - \frac{\epsilon }{4\lambda \sqrt{\gap}}I \succeq I~\text{and}  \\ \matU^\top (\aup I - \matW)^{-1} \matU \prec \stielt(\alow)I + \frac{\epsilon }{4\lambda\sqrt{\gap}}I \preceq I
\end{multline}
and similarly for $a = \aup$. Hence, the events $\Eup(\aup)$ and $\Elow(\alow)$ will hold. To this end,  we set $1/e \ge \delta \ge p = e^{-d^{1/3}}$, so that the event $\calA(z^*)$ holds with probability at least $1 - p$, where $z_* = 23 d^{-1/3}\log^{2/3}(d)$. Define the un-normalized gap $\boldDelta = (\boldlam + \boldlam^{-1} - 2)$. We shall further assume that
	\begin{eqnarray}
	\boldDelta \ge \frac{23}{\kappa} d^{-1/3}\log^{2/3} d
	\end{eqnarray} 
	 In order to apply the machinery of Section~\ref{sec:Wig_Spec}, we begin with the following lemma, which lower bounds $(\alow - z^*)$:

	\begin{claim}\label{claim:spec1} Pick $\epsilon \le \boldgap/2$. Then $\alow - z^* \ge (1 - \frac{2\boldlam \epsilon}{\boldDelta})(1 - \kappa) (\lambda + \lambdainv)\boldgap \ge  \frac{1}{4}(\lambda + \lambdainv)\boldgap$.
	\end{claim}
	The above claim is proved in the subsection below.We now use the following claim, which combines Proposition~\ref{prop:Hanson_wright_proposition} and Theorem~\ref{thm:stiel_thm} to shows that if $\epsilon$ is chosen as in Claim~\ref{claim:spec1}, then we can bound $\|\matU^\top(aI - \matW)^{-1}\matU - \stielt(a)I\|_{\op}$ with high probability: 
	\begin{claim}\label{claim:spec2} Let $\epsilon$ be as in Claim~\ref{claim:spec1}, and define $\boldeps_0 := \frac{1}{d^{1/2}\boldDelta}$. Then, if there exists constants $K,K',C$ sufficiently large such that the conditions $\boldeps_0^2 < \frac{1}{K}\min\{1,\boldDelta\}$ and $d \ge K$, then with probability at least $1 - 9\delta$, 
	\begin{eqnarray*}
	\max_{a \in \{\alow,\aup\}} \|\matU^\top(aI - \matW)^{-1}\matU - \stielt(a)I\|_{\op} &\le& C(\sqrt{r + \log 1/\delta})\boldeps_0~,
	\end{eqnarray*}
	\end{claim}
	Again, the above claim is proved in the subsection below. We now conclude the proof of Theorem~\ref{thm:main_spec_thm}. Suppose that $d$ satisfies $\boldDelta \ge \frac{23}{\kappa} d^{-1/3}\log^{2/3} d$.

	Given our $\epsilon \le \boldgap/2$, we then see that as long as 
	\begin{eqnarray*}\label{eps0_condition}
	\boldeps_0 < \min\left\{\sqrt{\frac{1}{K}\min\{1,\boldDelta\}}, \frac{\epsilon }{ 4 \lambda C(\sqrt{r + \log \delta^{-1}} \cdot \sqrt{\boldgap} }\right\}~,
	\end{eqnarray*}
	then on the event of Claim~\ref{claim:spec2}, we see that Equation~\ref{eq:specWTS} holds. Noting that $\min\{1,\boldDelta\} \le \frac{\boldDelta}{\lambda^2 + 1} = \boldgap$, and that $\epsilon/\sqrt{\boldgap} \lesssim \boldgap$, Equation~\eqref{eps0_condition} is satisfies as long as, for some universal large constant $K'$
	\begin{eqnarray*}
	&& \frac{1}{d^{1/2}\boldDelta} = \boldeps_0 \le \frac{1}{K'}\frac{\epsilon }{ (\lambda + \lambdainv) (\sqrt{r + \log \delta^{-1}} \cdot \sqrt{\boldgap} } \quad \iff  \\
	&& \frac{\lambda + \lambdainv}{d^{1/2}\boldDelta}  \le \frac{1}{K'}\frac{\epsilon }{ (\sqrt{r + \log \delta^{-1}} \cdot \sqrt{\boldgap} } \quad \iff \\
	&&\frac{1}{d^{1/2}\boldgap} \le \frac{1}{K'}\frac{\epsilon }{ (\sqrt{r + \log \delta^{-1}} \cdot \sqrt{\boldgap}} \quad \iff
	d \ge (K')^2 (r + \log(\delta^{-1}))/\epsilon^2\boldgap
	\end{eqnarray*}
	 
	Finally, Squaring, and combining with the fact that it suffices $\boldDelta \ge \frac{23}{\kappa} d^{-1/3}\log^{2/3} d$, which holds as long as 
	\begin{eqnarray*}
	d \ge C''' (\kappa\boldDelta)^{-3} \log^2(1/\kappa\boldDelta) \ge  C''' (\kappa\boldgap)^{-3} \log^2(1/\kappa\boldgap)
	\end{eqnarray*}
	for a sufficiently large constant $C$~, we conclude that it suffices
	\begin{eqnarray}
	d \ge  C \left(\frac{(r+\log(1/\delta)) }{ \boldgap \epsilon^2} + (\kappa\boldgap)^{-3} \log(1/\kappa\boldgap) \right)~. 
	\end{eqnarray}
	For the second point of the theorem - that $\lambda_{r}(\matM) - \|\matW\|_{\op} \ge \frac{1}{4}(\boldlam + \lambdainv)$ - we recall that $\lambda_r(\matM) - \|\matW\|_{\op} \ge (\alow - z^*) \ge (1 - \frac{2\boldlam \epsilon}{\boldDelta})(1 - \kappa) \boldDelta \ge \boldDelta/4$, by Claim~\ref{claim:spec1}.

\subsubsection{Proof of Supporting Claims}

	\begin{proof}[Proof of Claim~\ref{claim:spec1}]
	Observe that if $\alow \ge 2 + \frac{23}{\kappa}d^{-1/3}\log^{2/3}d$, then
	\begin{eqnarray*}
	\frac{\alow - z^*}{\alow - 2} \ge \frac{1/\kappa - 1}{1/\kappa} = (1-\kappa)
	\end{eqnarray*}
	Moreover, we have
	\begin{eqnarray*}
	\frac{\alow - 2}{\boldDelta} = \frac{(\boldlam + \lambdainv)(1-\epsilon) - 2}{\boldlam + \lambdainv - 2} = 1 - \frac{\epsilon(\boldlam + \lambdainv)}{\boldlam + \lambdainv - 2} = 1 - \epsilon/\boldgap
	\end{eqnarray*}
	Combining, and using the fact that $\kappa \le 1/2$ and $\epsilon/\boldgap \le 1/2$, we have
	\begin{eqnarray*}
	\frac{\alow - z^*}{\Delta} \ge 1/4~, \quad \text{ as needed.}
	\end{eqnarray*}
	\end{proof}
	\begin{proof}[Proof of Claim~\ref{claim:spec2}]
	Recall the definition $\boldeps := (d(a-z^*)^2)^{-1/2}$ as in Theorem~\ref{thm:stiel_thm}. Since $\aup - z^* \ge \alow - z^* \ge \boldDelta/4$, we have that $\boldeps \le 4\boldeps_0$, and hence if $K$ is a sufficiently large constant, $\boldeps_0^2 < \frac{1}{K}\min\{1,\boldDelta\}$ will imply $\boldeps^2 < \min\{\frac{1}{16\sqrt{2}},\frac{a-2}{32}\} $. 

	Moreover, as $a \le d$, we must have that $\boldeps \ge d^{-3/2}$. Since $p = e^{-d^{1/3}}$, it holds that as long as $d$ is a sufficiently large constant, the condition  $p^{1/3} < \boldeps/8$ of Theorem~\ref{thm:stiel_thm}. Hence, we see that the events $\calE_S(\alow,\delta)\cap \calE_S(\aup,\delta)$ from Theorem~\ref{thm:stiel_thm} occurs with probability at least $1 - 2\delta$. Moreover, as long as $d$ is sufficiently large constant that $d^{3} \cdot 8d^{3/2}p_{z^*}^{1/6} = 8d^{9/2}e^{-d^{1/3}} < 1$, there exists a numerical constant constant $c_1$ such that, on $\calE_S(\alow,\delta)\cap \calE_S(\aup,\delta) \cap \calA(z^*)$,
	\begin{eqnarray*}
	\left|S_\matW(a) - \stielt(a)\right| &\le&  (4\sqrt{2} + 2\sqrt{\log(2/\delta)})\boldeps^2 + 8d^{3/2} p^{1/6}  \\
	&\le&  (4\sqrt{2} + 2\sqrt{\log(2/\delta)})\boldeps^2 + 8d^{3/2} p^{1/6}  \\
	&\le&  c_1(\log(1/\delta))\boldeps^2
	\end{eqnarray*}
	Next, we apply Proposition~\ref{prop:Hanson_wright_proposition} with $\delta \le e^{-d^{1/3}}$. Since we can choose the constant $C$ large enough that $r \le d/10$, then if $d$ is sufficiently large so $ \log(1/\delta) \le d^{1/3} \le d/4 - 2.2 r$, then we have $t := 2.2r + \log(1/\delta_2)  \le d/4$, and thus on $\calA(z^*)$, we have
	\begin{eqnarray*}
	3 \delta &\ge& \Pr\left[\left\|\matU^{\top}(aI - \matW)^{-1}\matU - S_{\matW}(a) I\right\|_{\op} > \frac{8}{d(1 - 2\sqrt{t/d})} \left(t^{1/2}\|(aI - \matW)^{-1}\|_{\F} + t\|(aI - \matW)^{-1}\|_{\op}\right)\right] \\
	 &\ge& \Pr\left[\left\|\matU^{\top}(aI - \matW)^{-1}\matU - S_{\matW}(a) I\right\|_{\op} > \frac{16}{d} \left((dt)^{1/2}\|(aI - \matW)^{-1}\|_{\op} + t\|(aI - \matW)^{-1}\|_{\op}\right)\right]\\
	&\ge& \Pr\left[\left\|U\matU^{\top}(aI - \matW)^{-1}\matU - S_{\matW}(a) I\right\|_{\op} > \frac{32 \sqrt{2.2r + \log (\delta)}}{d^{1/2}} \|(aI - \matW)^{-1}\|_{\op}\right]\\
	&\ge&  \Pr\left[\left\|\matU^{\top}(aI - \matW)^{-1}\matU - S_{\matW}(a) I\right\|_{\op} > \boldeps(c_2 (\sqrt{r + \log (\delta)})\right]~,
	\end{eqnarray*}
	for some constant $c_2$. Altogether, we have that there exists a constant $c_3 > 0$ such that, with probability $1 - 8\delta - e^{-d^{1/3}} \ge 1 - 9\delta$, 
	\begin{eqnarray*}
	\|\matU^\top(aI - \matW)^{-1}\matU - \stielt(a)I\|_{\op} &\le& |S_{\matW}(a) - \stielt(a)| + \left\|\matU^{\top}(aI - \matW)^{-1}\matU - S_{\matW}(a) I\right\|_{\op}\\
	&\le& c_1(1+\log(1/\delta))\boldeps^2 + \boldeps(c_2 (\sqrt{r \log (\delta)} )\\
	&\le& c_3(\sqrt{r + \log (\delta)})\boldeps\\
	&\le& c_3'(\sqrt{r \log  (\delta)})\boldeps_0~.
	\end{eqnarray*}
	\end{proof}
\subsection{Proof of Lemma~\ref{lem:stielt_lem}\label{sec:stiel_lem_proof}}

For the first point, we begin by establishing the lemma with a modified parameterization; we show that
\begin{eqnarray}\label{stielt:desired}
\stielt(\boldlam + \boldlam^{-1} -  (\boldlam - \boldlam^{-1}) \epsilon) \le \frac{1}{\boldlam} + \frac{\epsilon } {2\boldlam}  & \text{and} & \stielt(\boldlam + \boldlam^{-1} + (\boldlam - \boldlam^{-1}) \epsilon) \ge \frac{1}{\boldlam} - \frac{\epsilon}{2\sqrt{2}\boldlam}
\end{eqnarray}
for $\epsilon \le \frac{\min\{1,(\boldlam^2 - 1) /2\} }{\boldlam^2 + 1}$. Hence if we let $\epsilon' = \frac{\boldlam -  \boldlam^{-1}}{\boldlam +  \boldlam^{-1}} = \frac{\boldlam^2 - 1}{\boldlam^2 + 1} \epsilon $, then we have we have
\begin{eqnarray*}
\stielt((\boldlam+\boldlam^{-1})(1-\epsilon' )) \ge \frac{1}{\boldlam}(1 + \frac{\epsilon'(\boldlam^2 + 1)}{2(\boldlam^2 - 1)}) & \text{and} & \stielt((\boldlam + \boldlam)^{-1}(1 + \epsilon')) \le \frac{1}{\boldlam}(1 - \frac{\epsilon'(\boldlam^2 + 1)}{2\sqrt{2}(\boldlam^2 - 1)})
\end{eqnarray*}
provided that $\epsilon' \le \frac{\boldlam^2 - 1}{(\boldlam^2 + 1)(\boldlam + 1)^2} \cdot \min\{1,(\boldlam^2 - 1) /2\} \le \frac{\boldlam^2 - 1}{(\boldlam^2 + 1)^2} \cdot  $. Finally, we may simplify
\begin{eqnarray*}
\frac{\boldlam^2 - 1}{(\boldlam^2 + 1)(\boldlam + 1)^2}\min\{1,(\boldlam^2 - 1) /2\} = \frac{(\boldlam^2 - 1)^2}{(\boldlam^2 + 1)(\boldlam + 1)^2} \min\{1/2, \frac{1}{\boldlam^2 - 1}\} = \boldgap \cdot \min\{1/2, \frac{1}{\boldlam^2 - 1}\}
\end{eqnarray*}
and $\frac{\boldlam^2 + 1}{\boldlam^2 - 1} = \frac{\sqrt{\boldlam^2 + 1}}{\boldlam - 1} \cdot  \frac{\sqrt{\boldlam^2 + 1}}{\boldlam + 1} = \boldgap^{-1/2} \cdot \sqrt{1 - \frac{2\boldlam}{(\boldlam+1)^2}} \ge (\boldgap/2)^{-1/2}$.

Now to prove Equation~\eqref{stielt:desired}. For $\sigma \in \{-1,1\}$, let $a_\sigma = \boldlam + 1/\boldlam + \sigma (\boldlam - \boldlam^{-1}) \epsilon$. We then have that
\begin{eqnarray*}
(\boldlam + \boldlam^{-1} + \sigma (\boldlam - \boldlam^{-1})\epsilon)^2  - 4 &=& (\boldlam + \boldlam^{-1})^2 - 4 + \epsilon^2 (\boldlam - \boldlam^{-1}) +  2 \sigma \epsilon (\boldlam + \boldlam^{-1})(\boldlam - \boldlam^{-1}))\\
&=& (\boldlam - \boldlam^{-1})^2  + \epsilon^2 (\boldlam - \boldlam^{-1})^2 +  2 \sigma \epsilon (\boldlam + \boldlam^{-1})(\boldlam - \boldlam^{-1}))\\
&=& (\boldlam - \boldlam^{-1})^2\left( 1 + \epsilon^2 + 2\sigma \epsilon \frac{\boldlam + \boldlam^{-1}}{\boldlam - \boldlam^{-1}} \right)
\end{eqnarray*}
Hence, 
\begin{eqnarray*}
\stielt(a_\sigma) &=& \frac{\boldlam + \boldlam^{-1} + \epsilon \sigma (\boldlam - \boldlam^{-1}) - (\boldlam - \boldlam^{-1}) +  (\boldlam - \boldlam^{-1}) (1 - \sqrt{1 + \epsilon^2 + 2\sigma \epsilon \frac{\boldlam + \boldlam^{-1}}{\boldlam - \boldlam^{-1}}})}{2}\\
&=& \frac{1}{\lambda}  +  (\boldlam - \boldlam^{-1}) \frac{\sigma \epsilon + (1 - \sqrt{1 + \epsilon^2 + 2\sigma \epsilon \frac{\boldlam + \boldlam^{-1}}{\boldlam - \boldlam^{-1}}})}{2}
\end{eqnarray*}

For $\sigma = -1$, we have that, as long as $(*) ~2\epsilon (\boldlam + \boldlam^{-1})(\boldlam - \boldlam^{-1})^{-1}  \le 1$ and $(**) ~\epsilon (\boldlam - \boldlam^{-1})/2 \le \boldlam^{-1}$ then we can bound lower bound above using concavity of $x\mapsto \sqrt{1 -x}$ as 
\begin{eqnarray*}
\stielt(a_-)  &\overset{(*)}{\ge}& \frac{1}{\boldlam} + (\boldlam - \boldlam^{-1}) \cdot \frac{- \epsilon +   \epsilon (\boldlam + \boldlam^{-1})^{-1}(\boldlam - \boldlam^{-1}) - \epsilon^2/2  ) }{2} \\
&=& \frac{1}{\boldlam} + \epsilon \frac{  (\boldlam + \boldlam^{-1}) - (\boldlam - \boldlam^{-1}) -  (\boldlam - \boldlam^{-1})\epsilon/2 }{2} \\
&=& \frac{1}{\boldlam} + \epsilon \frac{ 2\boldlam^{-1} -  (\boldlam - \boldlam^{-1})\epsilon/2  ) }{2} \overset{(**)}{\ge} \frac{1}{\boldlam} + \frac{\epsilon } {2\boldlam} 
\end{eqnarray*}
 On the other hand, we note that by Taylor's theorem
\begin{eqnarray}
\sqrt{1 + x} \ge 1 + x \cdot \frac{1}{2\sqrt{1+x}}
\end{eqnarray}
Hence, if we set $ \epsilon (\boldlam + \boldlam^{-1})(\boldlam - \boldlam^{-1})^{-1} = t$, then we can upper bound
\begin{eqnarray*}
\stielt(a_+) &=&  \frac{1}{\boldlam} + (\boldlam - \boldlam^{-1}) \cdot \frac{ \epsilon - (1-\sqrt{1  + 2  \epsilon (\boldlam + \boldlam^{-1})(\boldlam - \boldlam^{-1})^{-1} } ) }{2} \\
&=&  \frac{1}{\boldlam} + (\boldlam - \boldlam^{-1}) \cdot \frac{ \epsilon - (1-\epsilon (\boldlam + \boldlam^{-1})(\boldlam - \boldlam^{-1})^{-1}/\sqrt{1+2t}  } ) {2}  \\
&=&  \frac{1}{\boldlam} - \frac{(\boldlam + \boldlam^{-1}) \epsilon -  (\boldlam - \boldlam^{-1}) \sqrt{1+2t}\epsilon   }{2\sqrt{1+t}}  \\
&=&  \frac{1}{\boldlam} - \epsilon \frac{2 \boldlam^{-1}  -  (\boldlam - \boldlam^{-1}) (\sqrt{1+2t} - 1)   }{2\sqrt{1+t}}  \\
&\le&  \frac{1}{\boldlam} - \epsilon \frac{2 \boldlam^{-1} -  t(\boldlam - \boldlam^{-1})   }{2\sqrt{1+t}}  \\
\end{eqnarray*}
Hence if we have $t =  \epsilon (\boldlam + \boldlam^{-1})(\boldlam - \boldlam^{-1})^{-1} \le 1/2 < 1$, and 
\begin{eqnarray}
(\boldlam - \boldlam^{-1})t  = \epsilon (\boldlam + \boldlam^{-1}) \le \boldlam^{-1} \iff \epsilon \le \frac{1}{\boldlam (\boldlam + \boldlam^{-1})}
\end{eqnarray}, then the previous display is at most $\frac{1}{\boldlam} - \frac{\epsilon}{2\sqrt{2}\boldlam}$. Collecting the conditions we needed, we required $\epsilon \le \frac{\boldlam - \boldlam^{-1}}{2(\boldlam + \boldlam^{-1})} = \frac{\boldlam^2 - 1}{2(\boldlam^2 + 1)}$ and $\epsilon \le  \frac{1}{\boldlam (\boldlam - \boldlam^{-1})} = \frac{1}{\boldlam^2 + 1}$.

%Proof of Convergence Results of the Stiejtes Transform
%!TEX root = main.tex
	\section{Concentration of $S_{\matW}(z)$}\label{sec:Stieltjes_sec}
	
	\subsection{Proof of Theorem~\ref{thm:stiel_thm}\label{sec:stiel_thm_proof}}
	For a given $a$ and $z^*$, set $b = 8p_{z^*}^{1/3} \ge \sqrt{p_{z^*}d(a-z^*)^2}$. First, we claim that the following hold:
	\begin{claim}
	\begin{eqnarray}
	\max\left\{\frac{1}{d(a-z^*)^{2}}, \frac{1}{d^2(a-z^*)^{5}} \right\} \le 1/16\sqrt{2}~ \quad \text{and} \quad b < (a^2 - 4 - \frac{32}{d(a-z^*)^2})^{1/2}
	\end{eqnarray}
	\end{claim}
	\begin{proof}
	By assumption, we have that $\frac{32}{d(a-z^*)^2} \le a-2$. 
	\begin{eqnarray}
	(a^2 - 4 - \frac{32}{d(a-z^*)^2})^{1/2} &=& ((a+2)(a-2) - \frac{32}{d(a-z^*)^2})^{1/2}\\
	&\ge& (2(a-2) - \frac{32}{d(a-z^*)^2})^{1/2}\\
	&\ge&  (\frac{32}{d(a-z^*)^2})^{1/2} \ge \sqrt{\frac{1}{d(a-z^*)^2}} ~,
	\end{eqnarray}
	the last expression which at least $b$ since $b = 8p_{z^*}^{1/3} < \sqrt{\frac{1}{d(a-z^*)^2}}$ by assumption. This verifies the condition after the ``and''. For the first condition, we Theorem~\ref{thm:stiel_thm} directly assume that. $\frac{1}{d(a-z^*)^{2}} < 1/16\sqrt{2}$, 
	\begin{eqnarray}\frac{1}{d^2(a-z^*)^{5}} &\le& \frac{1}{(16\sqrt{2}) d(a-z^*)^2} \cdot (a-z^*)^{-1}\\
	&\le& \frac{1}{32 \cdot 16\sqrt{2}} \cdot \frac{a-2}{a-z^*} \\
	&\le& \frac{30}{32 \cdot 16\sqrt{2}} < 1/16\sqrt{2} \\
	\end{eqnarray}
	\end{proof}

		Our goal will be to apply Proposition~\ref{prop:exp_comp}. In order to do so, we need to check first that $a^2 - 4 > b^2 + 4|\Re(\err)|$ and $|b| > |\Im(\err(z))|$. Observe that
	\begin{eqnarray}
	\max\left\{\frac{1}{d(a-z^*)^{2}}, \frac{1}{d^2(a-z^*)^{5}} \right\} \le 1/16\sqrt{2}~,
	\end{eqnarray}
	suffices for the first condition of Proposition~\ref{prop:exp_comp}, since then 
	\begin{eqnarray*}
	|\Im(\err(z))| &<& \frac{8\sqrt{2} b}{d}\max\left\{\frac{1}{d(a-z^*)^{2}}, \frac{1}{d^2(a-z^*)^{5}} \right\} + 4p_{z^*}/b^2 \\
	&\le& \frac{8\sqrt{2} b}{d}\max\left\{\frac{1}{d(a-z^*)^{2}}, \frac{1}{d^2(a-z^*)^{5}} \right\} + b/2 \le b/2 + b/2  = b
	\end{eqnarray*}
	For the second condition, it suffices that 
	\begin{eqnarray*}
	b < (a^2 - 4 - \frac{32}{d(a-z^*)^2})^{1/2}
	\end{eqnarray*}
	Since, using the bound $p_{z^*}/b^2 \le \frac{1}{d(a-z^*)^2}$, we have
	\begin{eqnarray*}
	|\Re(\err(z))| &\overset{b^2 \ge (a-z^*)^2d p_{z^*}}{\le}& 4 \left( \frac{1}{d(a-z^*)^2} + p_{z^*}/b^2\right) \le \frac{8}{d(a-z^*)^2}\\
	\implies b^2 + 4|\Re(\err(z))| &<& a^2 - 4 - \frac{32}{d(a-z^*)^2} + \frac{32}{d(a-z^*)^2} < a^2 - 4
	\end{eqnarray*}
	In summary, Proposition~\ref{prop:exp_comp} will hold for our choise of $b$ long as $\max\{(8p_{z^*})^{1/3},\sqrt{p_{z^*}d(a-z^*)^2}\}\} < (a^2 - 4 - \frac{32}{d(a-z^*)^2})^{1/2}$ and $\max\{d(a-z^*)^{-2}, d^2(a-z^*)^{-5}\} \le 1/16\sqrt{2}$. 
	Under these conditions, we have
	\begin{eqnarray*}
	\left|\Re(\Exp[S_W(z)]) - \frac{a - \sqrt{a^2 -4 }}{2}\right|
	 &\overset{(i)}{\le}&  \sqrt{ |b^2 + 4\Re(\err(z))| + |(2ab + \Im(\err(z))|}\\ 
	  &\overset{(ii)}{\le}&  \sqrt{ |b^2 + 4\Re(\err(z))| + |(2a+1)b|}\\ 
	&\overset{(iii)}{\le}&  |b| + 4\sqrt{2}d^{-1/2}\max\{1,(a-z^*)^{-2}\} + \sqrt{(2a+1)b}\\
	&\overset{(vi)}{\le}&  (1+\sqrt{2a+1})b^{1/2} + \frac{4\sqrt{2}}{d(a-z^*)^{2}}
	\end{eqnarray*}
	where $(i)$ is from Proposition~\ref{prop:exp_comp}, $(ii)$ is the fact that $|\Im(\err(z))| \le |b|$, and $(iii)$ combines concavity of $\sqrt{\cdot}$ with our estimate for $\Re(\err(z))$, and $(vi)$ uses the fact that $b < 1$. Hence, on $\calA(z^*)$, we have
	\begin{eqnarray*}
	&& \left|\Re(S_W(z)) - \frac{a - \sqrt{a^2 -4 }}{2}\right| \\
	&\overset{\text{triange ineq.}}{\le}& (1+\sqrt{2a+1})b^{1/2} + \frac{4\sqrt{2}}{d(a-z^*)^{2}}  + \left|\Re(\Exp[S_W(z)]) - \Re(S_W(z))\right|\\
	&\overset{\text{Equation~\ref{eq:expectation_distance}}}{\le}& (1+\sqrt{2a+1})b^{1/2} + \frac{4\sqrt{2}}{d(a-z^*)^{2}}  + p_{z^*}/b + \left|\Re(\Exp[S_{\widetilde{W}}(z)]) - \Re(S_{\widetilde{W}}(z))|\right|\\
	\end{eqnarray*}
	Finally, if we let $\mathcal{E}_S(\delta):= \{\left|\Re(\Exp[S_{\widetilde{W}}(z)]) - \Re(S_{\widetilde{W}}(z))\right| \le \frac{2\sqrt{\log(2/\delta)}}{d(a-z^*)^2}\}$, then by combining TIS Inequality in Lemma~\ref{lem:TIS} with the Lipschitz estimates in Lemma~\ref{lem:Lipschitz Lemma}, we have
	 then $\Pr[\widetilde{\mathcal{E}}_S(\delta)] \ge 1-\delta$. Finally, on $\widetilde{\mathcal{E}}(\delta)$, we have that $\left|\Re(S_W(z)) - \frac{a - \sqrt{a^2 -4 }}{2}\right|$ is at most
	\begin{eqnarray*}
	&& (1+\sqrt{2a+1})b^{1/2} + \frac{4\sqrt{2}}{d(a-z^*)^{2}}  + p_{z^*}/b + \frac{2\sqrt{\log(2/\delta)}}{d(a-z^*)^2} \\
	&=&\frac{4\sqrt{2} + 2\sqrt{\log(2/\delta)}}{d(a-z^*)^2} + (1+\sqrt{2a+1})b^{1/2} + p_{z^*}/b
	\end{eqnarray*}
	Finally, using the estimate $\left|\Re(S_W(ai+b)) - S_W(a)\right| \le  b^2/(a-z^*)^3$ from Lemma~\ref{lem:imag_closeness}, we conclude that $\left|S_W(a) - \frac{a - \sqrt{a^2 -4 }}{2}\right|$ is at most 
	\begin{eqnarray*}
	\frac{4\sqrt{2} + 2\sqrt{\log(2/\delta)}}{d(a-z^*)^2} + (1+\sqrt{2a+1})b^{1/2} + p_{z^*}/b + b^2/(a-z^*)^3
	\end{eqnarray*}
	Subtituting in $b = 8p_{z^*}^{1/3}$, we have that the above is at most
	\begin{eqnarray*}
	\frac{4\sqrt{2} + 2\sqrt{\log(2/\delta)}}{d(a-z^*)^2} + 2\sqrt{2}(1+\sqrt{2a+1})\cdot p_{z^*}^{1/6} + p_{z^*}^{2/3}\cdot \left(\frac{1}{8}  + 1/(a-z^*)^3\right)
	\end{eqnarray*}
	Further, substituting in $\boldeps = 1/(a-z^*)d^{1/2}$, and using that fact that $\boldeps \le 1/4$ by assumption, we have that the above is at most
	\begin{eqnarray*}
	&& (4\sqrt{2} + 2\sqrt{\log(2/\delta)})\boldeps^2 + 2\sqrt{2}(1+\sqrt{2a+1})\cdot p_{z^*}^{1/6} + p_{z^*}^{2/3}\cdot \left(\frac{1}{8}  + d^{3/2}\boldeps^3\right) \\
	&& (4\sqrt{2} + 2\sqrt{\log(2/\delta)})\boldeps^2 + (4 \sqrt{2} + 4a)\cdot p_{z^*}^{1/6} + \frac{d^{3/2}}{4}p_{z^*}^{2/3}
	\end{eqnarray*}
	Lastly, we can bound $(4 \sqrt{2} + 4a)\cdot p_{z^*}^{1/6} + \frac{d^{3/2}}{4}p_{z^*}^{2/3} \le 8d^{3/2}p_{z^*}^{1/6}$, as needed.  

	\subsection{Proof of Lemma~\ref{lem:err_control}}
			Define the random variable $\matZ:= (S_{\widetilde{\matW}}(z) - \Exp[S_{\widetilde{\matW}}(z)]$. Then, we have 
		\begin{eqnarray*}
		|\Im(\Exp[\matZ^2])| &=& |\Exp[\Im(\matZ^2)]| \overset{(i)}{\le} \Exp[|\Im(\matZ^2)|] = 2\Exp[|\Re(\matZ)||\Im(\matZ)|] \\
		&\overset{(ii)}{\le}&  2\sqrt{|\Exp[\Re(\matZ)|^2]\Exp[|\Im(\matZ)^2|]} \overset{(iii)}{\le} 2\sqrt{\frac{\sqrt{2}}{d^2(a-z^*)^4}\cdot\frac{4\sqrt{2}b^2}{d^2(a-z^*)^6}}  = \frac{4\sqrt{2}b}{d^2(a-z^*)^5} 
		\end{eqnarray*}
		where $(i)$ is Jensen's inequality, $(ii)$ is Cauchy Schwartz, and $(iii)$ uses the estimates from Lemma~\ref{lem:Lipschitz Lemma}. Hence, combing with Proposition~\ref{lemma:tildeW_err_bound}, we have
		\begin{eqnarray*}
		|\Im(\err(z))| &\le& |\Im(\Exp[\matZ^2])| + \frac{2}{d(a-z^*)^2} + 4p_{z^*}/b^2 \le \frac{b}{d}\cdot\left(\frac{4\sqrt{2}}{d(a-z^*)^5} + \frac{2}{(a-z^*)^2}\right) + 4p_{z^*}/b^2 \\
		&\le& \frac{8\sqrt{2} b}{d}\cdot \max\{(a-z^*)^{-2},(a-z^*)^{-5}/d\} + 4p_{z^*}/b^2
		\end{eqnarray*}
		By the same token, one has 
		\begin{eqnarray*}
		|\Re(\Exp[\matZ^2])| &\le& \Exp[|\Re(\matZ^2)|] = \Exp[(\Im(\matZ))^2 ] + \Exp[\Re(\matZ)^2] \\ 
		&\le& \frac{\sqrt{2}}{d^2(a-z^*)^4} + \frac{4\sqrt{2}b^2}{d^2(a-z^*)^6}\\
		\end{eqnarray*}
		whence, by Lemma~\ref{lemma:tildeW_err_bound}, we can conclude the following as long as  $b < (a-z^*)/2$ and $d \ge (a-z^*)^2$ that
		\begin{eqnarray*}
		|\Re(\err(z))| &\le& \frac{\sqrt{2}}{d^2(a-z^*)^4} + \frac{4\sqrt{2}b^2}{d^2(a-z^*)^6} + \frac{1}{d(a-z^*)^2} + 4p_{z^*}/b^2\\
		&\overset{d \ge (a-z^*)^2}{\le}& \frac{\sqrt{2} + 1 + 4\sqrt{2}b^2\cdot(a-z^*)^{-2}}{d(a-z^*)^2} + 4p_{z^*}/b^2\\
		&\overset{b < (a-z^*)/2}{\le}& (\sqrt{2}+1/\sqrt{2}+1)(a-z^*)^4 + 4p_{z^*}/b^2 \le 4 \left( \frac{1}{d(a-z^*)^2} + p_{z^*}/b^2\right)
		\end{eqnarray*}

	\subsection{Proof of Lemma~\ref{lemma:tildeW_err_bound}}
		Recall the definition $\err(z) := \Exp[S_\matW(z)^2] + \frac{1}{d}\Exp[\tr(zI - \matW)^2] - \Exp[S_\matW(z)]^2$ and $p_{z^*} := \Pr[\calA(z^*)]$. We start off by bounding the term $\frac{1}{d}\Exp[\tr(zI - \matW)^2]$. We begin be observing that
		\begin{eqnarray*}
		\Exp[\tr((zI - \matW)^2)] &=& \frac{1}{d}\sum_{i=1}^d\left(\frac{1}{z - \lambda_i(\matWtil)}\right)^2 \I(\matW = \matWtil) + \left(\frac{1}{z - \lambda_i(\matW)}\right)^2\I(\matW \ne \matWtil)
		\end{eqnarray*}
		Observe that since $|z-\lambda_i(\matW)| \ge |b|$ by assumption, we have 
		\begin{eqnarray}
		|\Exp[\frac{1}{(z - \lambda_i(\matW))}\I(\matW \ne \matWtil)]| \le \Pr[\matW \ne \matWtil]/b^2 =  p_{z^*}/b^2
		\end{eqnarray}
		Furthermore, we have
		\begin{eqnarray*}
		\frac{1}{d}\sum_{i=1}^d\left(\frac{1}{z - \lambda_i(\matWtil)}\right)^2 &=& \frac{1}{d}\sum_{i=1}^n \frac{1}{(a - \min\{\lambda_i(\matW),z^*\}) + b\fraki)^2}\\
		&=& \frac{1}{d}\sum_{i=1}^n \frac{1}{(a - \min\{\lambda_i(\matW),z^*\}) + b\fraki)^2}\\
		\end{eqnarray*}
		We can then upper bound
		\begin{eqnarray*}
		|\Re(\frac{1}{d}\sum_{i=1}^d\left(\frac{1}{z - \lambda_i(\matWtil)}\right)^2)| &\le& |\frac{1}{d}\sum_{i=1}^n \frac{1}{(a - \min\{\lambda_i(\matW),z^*\}) + b\fraki)^2}|\\
		&\le& \max_{i \in [d]}|\frac{1}{|(a - \min\{\lambda_i(\matW),z^*\}) + b\fraki|)^2}| \le \frac{1}{(a-z^*)^2}
		\end{eqnarray*}
		Altogether, we conclude
		\begin{eqnarray}
		\frac{1}{d}|\Re(\Exp[\tr((zI - \matW)^2)])| \le \frac{1}{d}(1/(a-z^*)^2 + p^*/b^2)
		\end{eqnarray}
		For the more precise estimate of the impaginary component of $\Exp[\tr((zI - \matW)^2)$, we have 
		\begin{eqnarray*}
		\left|\Im\left(\frac{1}{d}\sum_{i=1}^d\left(\frac{1}{z - \lambda_i(\widetilde{\matW})}\right)^2\right)\right| &\le&  \max_{i \in [d]}|\Im\left(\frac{1}{(a - \min\{\lambda_i(\matW),z^*\}) + b\fraki}\right)^2|\\
		&\le&  \max_{i \in [d]}\left|\frac{(a - \min\{\lambda_i(\matW),z^*\}) - b\fraki)^2}{|a - \min\{\lambda_i(\matW),z^*\})^2 + b\fraki|^4}\right|\\
		&\le&  \max_{i \in [n]}\frac{2|b||a - \min\{\lambda_i(\matW),z^*\})|}{|(a-z^*\min\{\lambda_i(\matW),z^*\})|^4} \le 2b/(a-z^*)^2
		\end{eqnarray*}
		Altogether, we have
		\begin{eqnarray}
		\frac{1}{d}(\left|\Im(\Exp[\tr((zI - \matW)^2)])\right| \le \frac{1}{d}(2b/(a-z^*)^2 + p_{z^*}/b^2)
		\end{eqnarray}
		We now argue that $\Exp[S_{\matW}(z)^2] - \Exp[S_{\matW}(z)^2]$ is close to $\Exp[S_{\widetilde{\matW}}(z)^2] - \Exp[S_{\widetilde{\matW}}(z)^2]$. Repeating the arguments from above, we have the bounds that
		\begin{eqnarray*}
		|\Exp[S_{\widetilde{\matW}}(z)^2] - \Exp[S_{\matW}(z)^2]| \le \frac{1}{b^2}\Pr(\widetilde{\matW} \ne \matW)) & \text{and} & |\Exp[S_{\widetilde{\matW}}(z)] - \Exp[S_{\matW}(z)^2]| \le \frac{1}{|b|}p_{z^*}
		\end{eqnarray*}
		The later implies that
		\begin{eqnarray}
		|\Exp[S_{W}(z)]| \le |\Exp[S_{\widetilde{W}}(z)]| + \frac{1}{|b|}\Pr(\widetilde{W} \ne W)) \le \frac{1}{a-z^*} + \frac{1}{|b|}p_{z^*}
		\end{eqnarray}
		which entails that 
		\begin{eqnarray*}
		\Exp[S_{\matW}(z)]^2 - \Exp[S_{\widetilde{\matW}}(z)]^2 &=& |\Exp[S_\matW(z)] + \Exp[S_{\widetilde{\matW}}(z)]||\Exp[S_\matW(z)] - \Exp[S_{\widetilde{\matW}}(z)]|\\
		&\le& \frac{1}{(a-z)b}p_{z^*} + p_{z^*}/b^2
		\end{eqnarray*}
		All in all one has, for $b < a-z$ that
		\begin{eqnarray*}
		|\Exp[S_{\widetilde{W}}(z)^2] - \Exp[S_{\widetilde{\matW}}(z)]^2 -  (\Exp[S_{\matW}(z)^2] - \Exp[S_{\matW}(z)]^2)| \le  \frac{1}{(a-z)|b|}p_{z^*} + p_{z^*}/b^2 +  \frac{1}{|b|}p_{z^*} \le \frac{3p_{z^*}}{|b|^2}
		\end{eqnarray*}
		Putting together this estimate with the ones for $\frac{1}{d}|\Re(\Exp[\tr((zI - \matW)^2)])|$ and $\frac{1}{d}|\Im(\Exp[\tr((zI - W)^2)])|$ conclude the proof of the first display. The proof of the second display follows from essentially the same argument that we used to bound $|\Exp[\tr((zI - \matW)^2)]-\Exp[\tr((zI - \widetilde{\matW})^2)]|$, namely, that
		\begin{eqnarray}
		\|\Exp[S_\matW(z)] - \Exp[S_{\widetilde{\matW}}(z)]\| \le p_{z^*}\|S_\matW(z) - S_{\widetilde{\matW}}(z)\|_{\infty} \le p_{z^*}/b
		\end{eqnarray}
		Since $\matW = \widetilde{\matW}$ on $\calA(z^*)$, we conclude that
		\begin{eqnarray*}
		|S_\matW(z) - \Exp[S_\matW(z)]| &\overset{(i)}{=}& |S_{\widetilde{\matW}}(z) - \Exp[S_\matW(z)]| \\
		&\le& |S_{\widetilde{\matW}}(z) - \Exp[S_{\widetilde{\matW}}(z)]| + |\Exp[S_\matW(z)]-\Exp[S_{\widetilde{\matW}}(z)]| \\
		&\overset{(ii)}{\le}& |S_{\widetilde{\matW}}(z) - \Exp[S_{\widetilde{\matW}}(z)]| + \frac{p_{z^*}}{b} 
		\end{eqnarray*}

	\subsection{Proof of Lemma~\ref{lem:Lipschitz Computation}}
		We can write $\Psi$ as a composition of maps $\Psi_4 \circ \Psi_3 \circ \Psi_2 \circ \Psi_1$, where $\Psi_1(X) = \frac{1}{\sqrt{2d}}X + X^\top$ maps the underlying entries of $\matX$ to $\matW$, $\Psi_2(W) := (\lambda_1(W),\dots,\lambda_d(W)$, $\Psi_3$ maps $(\lambda_1,\dots,\lambda_d)$ to $(\min\{z^*,\lambda_1\},  \dots, \min\{z^*,\lambda_d\})$, and $\Psi_4$ maps $(\lambda_1,\dots,\lambda_d) \to \frac{1}{d} \sum_{i=1}^d \frac{1}{z - \lambda_i}$. Observe then that
		\begin{eqnarray*}
		\Psi_4 \circ \Psi_3 \circ \Psi_2 \circ \Psi_1(\matX) = \frac{1}{d} \sum_{i=1}^d \frac{1}{z - \min\{z^*,\lambda_i\}} = S_{\widetilde{\matW}}(z)
		\end{eqnarray*}
		Recalling here that $\Lip$ denotes the Lipschitz constant as a map between vector spaces endowed with the Euclidean norm, and htat $\Psi_1,\Psi_2,\Psi_3$ are all maps between real vector spaces, we have $\Lip(\Re(\Psi)) = \Lip(\Re(\Psi_4 \circ \Psi_3 \circ \Psi_2 \circ \Psi_1)) \le \Lip(\Re(\Psi_4)) \cdot \Lip(\Psi_3) \cdot \Lip(\Psi_2) \cdot \Lip(\Psi_2)$,  and analogously for $\Lip(\Im(\Psi))$. For $\Psi_1$,
		\begin{eqnarray*}
		\|\Psi_1(X) - \Psi_1(X')\|_F^2 &=& \frac{1}{2d}\|X + X^{\top} - X' + X^{'\top}\|_F^2\\
		&\le& \frac{1}{2d}(2\|X - X'\|_F^2 + 2\|X - X^{'\top}\|_F^2) \\
		&=& \frac{2\|X - X'\|_F^2}{d}
		\end{eqnarray*}
		so $\Lip(\Psi_1)^2 \le 2/d$. By the Hoffman-Weilandt Theorem (Theorem 6.3.5 in Horn and Johnson~\cite{horn2012matrix})
		\begin{eqnarray*}
		\|\Psi_2(W) - \Psi_2(W')\|_2^2 = \sum_{i=1}^d (\lambda_i(W) - \lambda_i(W'))^2 \le \sum_{i=1}^d (\lambda_i(W-W'))^2 = \|W-W'\|_F^2
		\end{eqnarray*}
		so $\Lip(\Psi_2) \le 1$. $\Lip(\Psi_3) \le 1$ as well, since
		\begin{eqnarray*}
		\sum_{i=1}^d (\max\{z^*,\lambda_i(W)\} - \max\{z^*,\lambda_i(W')\})^2 \le \sum_{i=1}^d (\lambda_i(W) - \lambda_i(W'))^2 
		\end{eqnarray*}
		It remains to compute $\Lip(\Re(\Psi_4))$ and $\Lip(\Im(\Psi_4))$ on the domain $(-\infty,z^*]^{d} \supseteq \im(\Psi_3 \circ \Psi_2 \circ \Psi_1)$. Since $\Psi_4$ is smooth (in fact analytic) on this domain for any $z > z^*$, it suffices to bound $\|\Im(\nabla \Psi_4)\|_2$ and $\|\Re(\nabla \Psi 4)\|_2$. We compute
		\begin{eqnarray*}
		\nabla \Psi_4 &=& \frac{1}{d}\left(\frac{-1}{(a - \lambda_i + b\fraki)^2}\right)_{1 \le i \le d}\\
		&=& \frac{1}{d}\left(\frac{-(a - \lambda_i - b\fraki)^2}{|a - \lambda_i + b\fraki|^4}\right)_{1 \le i \le d}\\
		&=& \frac{1}{d}\left(\frac{\{-(a - \lambda_i)^2 + b^2\} +  \{2(a-\lambda_i) b\}\fraki}{|a - \lambda_i + b\fraki|^4}\right)_{1 \le i \le d}
		\end{eqnarray*}
		Note then $\lambda_i < z^*$,and when  $a,b: a - z^*  > b$, $-(a - \lambda_i)^2 + b^2 \le |a - \lambda_i + bi|^2$, and thus
		\begin{eqnarray*}
		\|\Re(\nabla \Psi_4)\|_2^2 &\le& \frac{1}{d^2}\sum_{i=1}^n\left|\frac{\{-(a - \lambda_i)^2 + b^2\}} {|a - \lambda_i + b\fraki|^4}\right|^2 \\
		&\le&\frac{1}{d^2}\sum_{i=1}^d \frac{1}{|a - \lambda_i + b\fraki|^4} \le \frac{1}{d(a- \lambda_i)^4}
		\end{eqnarray*}
		Similarly,
		\begin{eqnarray*}
		\|\Im(\nabla \Psi_4))\|_2^2 \le \frac{1}{d^2}\sum_{i=1}^n \frac{4^2b^2(a-\lambda_i)^2}{|a - \lambda_i + b\fraki|^8} \le \frac{b^2}{d(a-z^*)^6}
		\end{eqnarray*}

	\subsection{Proof of Propostion~\ref{prop:exp_comp}\label{sec:exp_comp}}
		From Equation (2.45) in Anderson et al.~\cite{anderson2010introduction}, we have that for any $z \in \C - \R$ that
		\begin{eqnarray}\label{ExpectRelation}
		\Exp[S_\matW(z)] =  \frac{1}{z}\left( 1 + \Exp[S_\matW(z)^2] + \frac{1}{d}\Exp[\tr(zI - \matW)^2]\right)
		\end{eqnarray}
		Define $\overline{S}_\matW(z) := \Exp[S_\matW(z)]$, and rearranging Equation~\ref{ExpectRelation2} with the definition of $\err(z)$, we have
		\begin{eqnarray*}\label{ExpectRelation2}
		\overline{S}_\matW(z)^2 - z \overline{S}_\matW(z) +1+ \err(z) &=& 0 
		\end{eqnarray*}

		It follows from the quadratic formula that that 
		\begin{eqnarray*}
		\overline{S}_\matW(z) = \frac{z + \sigma \sqrt{z^2 - 4 - 4\err(z)}}{2} \quad \text{for} \quad \sigma \in \{-1,1\}
		\end{eqnarray*}
		Our first goal is to determine the sign $\sigma$. Observe that $S_\matW(a+b\fraki) = \sum_{j=1}^n \frac{1}{(a+b\fraki) - \lambda_j(\matW)} = \frac{1}{n}\sum_{j=1}^n \frac{a- \lambda_j(\matW) - bi}{|a-\lambda_i(\matW) - b\fraki|^2}$, so that $\sign(\Im(S_\matW(a+b\fraki)) = -\sign(b)$. Thus, $\sign(\Im(\overline{S}_\matW(a+b\fraki)) = -\sign(b)$ as well, which means that for $z = a+bi$,
		\begin{eqnarray*}
		-\sign(b) = \sign(\Im(z + \sigma \sqrt{z^2 - 4 - 4\err(z)})) = \sign(b + \sigma\Im(\sqrt{z^2 - 4 - 4\err(z)})))~,
		\end{eqnarray*}
		which implies that $\sigma = - \sign(b)\cdot \sign\Im(\sqrt{z^2 - 4 - 4\err(z)})))$. From the definition of the complex square root, 
		\begin{eqnarray*}
		\sign\Im(\sqrt{z^2 - 4 - 4\err(z)})) &=& \sign(\Im(z^2 - 4 - 4\err(z))) =  \sign( 2ab - 4\Im(\err(z)))
		\end{eqnarray*}
		Hence, for $a > 2$, then as long as $b > |\Im(\err(z))|$, it holds that the above display has the same sign as $b$, whence $\sigma = - \sign(b)^2 = - 1$. Hence, we have established that 
		\begin{eqnarray}
		\overline{S}_W(z) = \frac{z -  \sqrt{z^2 - 4 - 4\err(z)}}{2}
		\end{eqnarray}
		The proposition now follows from the following lemma:
		\begin{lem}[Perturbation Bound for Quadratics]
		\begin{eqnarray}\label{Quadratic_Perturb}
		|\Re(\sqrt{z^2 - 4 - 4\err}) - \sqrt{a^2 - 4}| \le \sqrt{ |b^2 + 4\Re(\err)| + |(2ab + \Im(\err)|}
		\end{eqnarray}
		\end{lem}
		\begin{proof} 
		Adopting the shorthand $u = a^2 - b^2 - 4 - 4\Re(\err)$ and $w = (2ab + \Im(\err)$, then as long as $a^2 - 4 > b^2 + 4\Re(\err)$, we have
		\begin{eqnarray*}
		\Re(\sqrt{z^2 - 4 - 4\err}) &=& \Re(\sqrt{a^2 - b^2 - 4 + \Re(\err) + \fraki(2ab + \Im(\err))})\\
		&=& \Re(\sqrt{u + \fraki w}) = \frac{1}{\sqrt{2}}\sqrt{ \sqrt{u^2 + w^2} + u }
		\end{eqnarray*}
		Hence, we can bound $\Re(\sqrt{z^2 - 4 + -4 \err}) \ge  \frac{1}{\sqrt{2}}\sqrt{ \sqrt{u^2} + u } = \sqrt{u}$. Moreover, one has $\sqrt{u} \ge \sqrt{a^2 - 4} - \sqrt{|b^2 + 4\Re(\err)|}$. On the other hand, $\Re(\sqrt{z^2 - 4 + \err}) \le \frac{1}{\sqrt{2}}\sqrt{ 2u + w} = \sqrt{u + w/2}$, and one can bound $\sqrt{u + w/2} \le \sqrt{a^2 - 4} + \sqrt{ |b^2 + 4\Re(\err)| + |(2ab + \Im(\err)|} $. Putting things together proves equation
		\end{proof}
	\subsection{Proof of Lemma~\ref{lem:imag_closeness}\label{sec:imag_closeness_proof}}

	 On $\calA(z^*)$, one has that $|\lambda_i(\matW) - a| \ge |z^*-a|$. Hence
	\begin{eqnarray*}
	\left|\Re(S_{\matW}(z)) - S_{\matW}(a)\right| &=& \left|\frac{1}{d}\sum_{i=1}^d \Re(\frac{1}{\lambda_i(\matW) - a - b\fraki}) + \frac{1}{\lambda_i(\matW) - a}\right|\\
	&=& \left|\frac{1}{d}\sum_{i=1}^d \Re(\frac{\lambda_i(\matW)- a}{(\lambda_i(\matW) - a)^2 + b^2}) + \frac{1}{\lambda_i(\matW) - a}\right|\\
	&=& \left|\frac{1}{d}\sum_{i=1}^d (\lambda_i(\matW) - a) \frac{1}{(\lambda_i(\matW) - a)^2 + b^2} + \frac{1}{(\lambda_i(\matW) - a)^2}\right|\\
	&\le&  \max_i \left|(\lambda_i(\matW) - a)\right| \cdot \left|\frac{1}{(\lambda_i(\matW) - a)^2 + b^2} + \frac{1}{(\lambda_i(W) - a)^2}\right|\\
	&=&  \max_i\left|(\lambda_i(\matW) - a)\right| \cdot \left|\frac{b^2}{\left(\left(\lambda_i(\matW) - a\right)^2 + b^2\right)\cdot \left(\lambda_i(\matW) - a\right)^2} \right|\\
	&\le&\max_i b^2 /\max_i(\lambda_i(\matW) - a)^3  \le b^2/(a-z^*)^3
	\end{eqnarray*}

%Proof of the Uniqueness of the Top K Eigenvalues
%!TEX root = main.tex

\newpage

\section{Proof of Proposition~\ref{Distinct_Eig_Prop}\label{sec:distinct_eig_prop}}

Define the matrix $\matLambda(a):= I - \boldlam \matU^\top (aI-\matW)^{-1}\matU$. Observe then that the event $\Eup(\aup)\cap \Elow(\alow)$, we have that $\spec(\matLambda(\alow)) \subset (-\infty,0)$ and $\spec(\matLambda(\aup)) \subset (0,\infty)$. Further, since $aI - \matW$ is invertible for all $a \in [\alow,\aup]$ under $\calA(z^*)$, it follows that the functions $a \mapsto \lambda_i(\matLambda(z))$ for $i \in [k]$ are analytic on $[\alow,\aup]$. In follows that, under $\calA(z^*) \cap \Eup(\aup)\cap \Elow(\alow)$, there exists real numbers $\{a^{(i)}\}_{1\le i k}\subset [\alow,\aup]$ such that $\lambda_i(\matLambda(a^{(i)})) = 0$. We need to now show that, up to a null event, these $a^{(i)}$ are distinct. Specifically, we claim that the event $\calN$ defined below measure zero:
\begin{eqnarray*}
\calN:= \{\exists a \in [a_1,a_2], i < j \in [k]: \lambda_i(\matLambda(a)) = \lambda_{i+1}(\matLambda(a)) = 0\} \cap \calA(z^*)
\end{eqnarray*}

To do this, define  the map $\psi: a \mapsto (\lambda_1(\matLambda(a)),\dots,\lambda_k(\matLambda(a)) \in \R^{k}$, and define the subspaces $\calV_{i,j} := \{v \in \R^k: v_i = v_j = 0\}$. By a union bound, it suffices to show that, for all $i < j \in [k]$,
\begin{eqnarray*}
\Pr[\calA(z^*) \cap \{\exists a \in [a_1,a_2]: \psi(a) \in \calV_{i,j}] = 0
\end{eqnarray*}
To do so, we establish two regularity properties about $\psi(a)$. First, observe that, with probability $1$ under $\calA(z^*)$, the fact that $a \mapsto \psi(a)$ is analytic and $[a_1,a_2]$ is compact implies there exists some (random) Lipschitz constant $L = L(W,U)$ such that $\psi(a)$ is $L$-lipschitz on the interval $[a_1,a_2]$ (in fact, one can show that $\psi$ is uniformly Lipschitz, but we shall not need this). 

Next, we claim that, for all $a$, $\psi(a)$ has density which is absolutely continuous with respect to the Lebesgue measure for all $a \in [a_1,a_2]$. Indeed, we we that 
\begin{lem}\label{lem_m_z density} Condition on the event $\calA(z^*)$. Then for every $a > z^*$, the random matrix $\matLambda(a) := I - \boldlam \matU^\top (aI - \matW)^{-1} \matU$ has a density with respect to the Lebesgue measure on $\Sym^k$. 
\end{lem}
Note then that if $\matLambda(a) \in \Sym^k$ has a density with respect to the Lebesgue measure, then $\matLambda(a)$ has a density with respect to the Wigner law on $\Sym^k$, and thus by a change of variables, $\spec(\matLambda(a))$ has a density with respect to the law of the eigenvalues of a Wigner matrix on $\Sym^N$. It is well know that the later have a density with respect to the Lebesgue measure\cite{anderson2010introduction}, which implies that $\psi(a)$ has density which is absolutely continuous with respect to the Lebesgue measure, as needed. Hence, our desired result follows from the following, quite general lemma:
\begin{lem}\label{lem:Lipschitz Lemma} Let $I$ be a compact interval, $\calA$ an event, and let $\psi:I \to \R^k$ be a real valued random funtion such that a) for all $a \in I$, $\psi(a)$ has a density with respect to $\Leb(\R^k)$, and b) with probability $1$ under $\calA$, $a \mapsto \psi(a)$ is Lipschitz  for $a \in I$. Then for any $k-2$-dimensional subspace $\calW$, $\Pr[\calA \cap \{\exists a \in [a_1,a_2]: \psi(a) \in \calW\}] = 0$.
\end{lem}

\subsection{Proof of Lemmas~\ref{lem_m_z density} and~\ref{lem:Lipschitz Lemma}}

\begin{proof}[Proof of Lemma~\ref{lem:Lipschitz Lemma}]
By a change of basis, we may assume that $\calW = \{w \in \R^k: w_1 = w_2 = 0\}$. Let $\calA_C$ denote the event that  $\sup_{a \in I}\max_i |\psi_i(a)| \le C$ and $\psi_i$ is $C$-Lipschitz on $I$. Observe that since $I$ is compact, and $\psi$ is lipschitz for some constant $L$ with probability $1$ on $\calA$, then with probability $1$ (on $\calA$) there exists a $C$ for which $\calA_C$ holds. Next, define the event $\calB_{\epsilon}(a) := \{|\psi_1(a)| < \epsilon\} \cap \{\psi_2(a)\} < \epsilon\}$. Then, $\{\exists a \in [a_1,a_2]: \psi(a) \in \calW = \bigcap_{\epsilon > 0} \bigcup_{a \in I} \calB_{\epsilon}(a)$. Using the above inclusions together with continuity from above and below of probability measures, 
\begin{eqnarray*}
\Pr[\calA \cap \{\exists a \in [a_1,a_2]: \psi(a) \in \calW\}]&=& \lim_{C \to \infty} \Pr[\calA \cap  \calA_{C} \cap \{\exists a \in [a_1,a_2]: \psi(a) \in \calW\}]\\
&=&\lim_{C \to \infty} \Pr[\calA \cap  \calA_{C} \cap \{\exists a \in [a_1,a_2]: \psi(a) \in \calW\}] \\
&=& \lim_{C \to \infty} \lim_{\epsilon \to 0 } \Pr[\calA \cap \calA_{C} \cap \bigcup_{a \in I} \calB_{\epsilon}(a)] 
\end{eqnarray*}
Now, let $\calN(C,\epsilon)$ be an $C/\epsilon$-net of $I$. Then, on $\calA \cap \calA_{C} \cap \bigcup_{a \in I} \calB_{\epsilon}(a)$, there exists an $a' \in \calN(C,\epsilon)$ such that $|\psi_i(a')| \le \epsilon + C|a-a'| \le 2\epsilon$ for $i \in \{1,2\}$. Hence $\calA \cap \calA_{C} \cap \bigcup_{a \in I} \calB_{\epsilon}(a) \subset \calA \cap \calA_{C} \cap \bigcup_{a \in \calN(C,\epsilon)} \calB_{2\epsilon}(a)$, and thus, 
\begin{eqnarray*}
\Pr[\calA \cap \calA_{C} \cap \bigcup_{a \in I} \calB_{\epsilon}(a)] \le |\calN(C,\epsilon)| \cdot \sup_{a \in I} \Pr[\calA \cap \calA_C \cap \calB_{\epsilon}(a)]
\end{eqnarray*}
Finally, since $\psi(a)$ is absolutely continuous with respect to the Lebesgue measure, there exists a constant $C'$ such that $\Pr[\calA \cap \calA_C \cap \calB_{2\epsilon}(a)] \le C'\mathrm{vol}(\{w: |w_1|\le 2\epsilon, |w_2|\le 2\epsilon, \max_{j > 2}|w_j| \le C\}) = C'(2C)^{k-2}(4\epsilon^2)$. Moreover, one as that $|\calN(C,\epsilon)|  \le 1+ 2C\mathrm{vol}(I)/\epsilon$. Hence, $|\calN(C,\epsilon)| \cdot \sup_{a \in I} \Pr[\calA \cap \calA_C \cap \calB_{\epsilon}(a)] \le O(1/\epsilon) \cdot O(\epsilon^2) = O(\epsilon)$. and hence $\lim_{\epsilon \to 0} \Pr[\calA \cap \calA_{C} \cap \bigcup_{a \in I} \calB_{\epsilon}(a)] = 0$, as needed.
\end{proof}

\begin{proof}[Proof of Lemma~\ref{lem_m_z density}] Under $\calA(z^*)$, $\Law(\matW)$ has a density with respect to $\Leb(\Sym^d)$. This implies that $(zI - \matW)$ has a density with respect to $\Leb(\Sym^d)$. On $\calA(z^*)$, $zI - \matW$ is invertible, and thus $(zI-W)^{-1}$  has density with respect to $\Leb(\Sym^d)$. Now, observe that if $U$ is a full rank matrix, then the map $X \mapsto U^{\top} X U$ is a surjective linear transformation from $\Sym^d$ to $\Sym^k$. It therefore follows that $\boldlam \matU^\top (zI - \boldlam \matW)^{-1} \boldlam \matU$ has a density with respect to $\Leb(\Sym^k)$, and hence $\matLambda(z)$ also has a density with respect to $\Leb(\Sym^k)$. 
\end{proof}

%High probability upper bound on ||W||
%!TEX root = main.tex

\section{Proof of Proposition~\ref{prop:norm_upper_bound}\label{sec:norm_upper_bound_proof}}

We begin by stating a useful concentration bound for functions of Gaussian random variables:
\begin{lem}[Tsirelson-Ibgragimov-Sudakov, Theorem 5.5 in~\cite{boucheron2013concentration}\label{lem:TIS}] Let $f$ be a $L$-Lipschitz function and let $X$ be a standard Gaussian vector. Then,
\begin{eqnarray*}
\Pr[f(X) \ge \Exp[f(X)] + t] \vee \Pr[f(X) - \Exp[f(X)] \le - t] \le e^{-t^2/2L^2}.
\end{eqnarray*}
\end{lem}
As a consequence, we can establish the following concentration bound for $\|\matW\|$
\begin{lem}[Norm Concentration] \label{lem:norm_conc}
\begin{eqnarray*}
\Pr(\|\matW\| \le \Exp[\|\matW\|] + 2\sqrt{\log(1/\delta)/d}) \ge 1 - \delta
\end{eqnarray*}
\end{lem}
\begin{proof} 
	
	Recall that $\matW\sim \GOE(d)$ has the distribution $\frac{1}{\sqrt{2d}}(\matX + \matX^{\top})$, where $\matX \sim \SG(d)$. We now claim that the composition $\matX \mapsto \matW \mapsto \|\matW\|$ is $\sqrt{2/d}$-Lipschitz. Since the pointwise supremum of $L$-Lipschitz functions is $L$-Lipschitz, and since
	\begin{eqnarray*}
	\|\matW\| = \sup_{v,w}v^{\top}\frac{1}{\sqrt{2d}}(\matX + \matX^{\top})vw~,
	\end{eqnarray*}
	it suffices to show that $f_{v,w}(X) := v^{\top}\frac{1}{\sqrt{2d}}(X + X^{\top})w$ is $\sqrt{2/d}$-Lipschitz. Since $f_{v,w}(X)$ is linear in $X$, it suffices to bound the operator norm (from $\|\cdot\|_F \to |\cdot|$) of $f_{v,w}(\cdot)$ by $\sqrt{2/d}$. This follows since
	\begin{eqnarray*}
	|v^{\top}\frac{1}{\sqrt{2/d}}(X+X^{\top})w| = \sqrt{2/d}|v^{\top}Xw| \le \sqrt{2/d}\cdot\|X\|_{F}\|wv^{\top}\|_F = \sqrt{2/d}\|X\|_F.
	\end{eqnarray*}
	The bound now follows from putting $L = \sqrt{2/d}$ into Lemma~\ref{lem:TIS}. 

\end{proof}

Next, we compute an upper bound of $\Exp[\|\matW\|]$.  To the best of the author's knowledge, the only reasonably sharp, non-asymptotic guarantees on $\Exp[\|\matW\|_{op}]$ come from []. However, asymptotic bounds established in 

\begin{thm}[Specializtion of Theorem 1.1 in] Let $\matW$ be a standard Wigner matrix, $\sigma^2 = \max_{i}\sum_{j}\Exp[\matW_{ij}^2]$, and $\sigma_*^2 = \max_{i,j} \Exp[\matW_{ij}^2]$. Then, 
	\begin{eqnarray*}
	\Exp[\|\matW\|] \le \inf_{\epsilon \in (0,1/2)}(1+\epsilon)\{2\sigma + \sigma^*\frac{6\sqrt{\log d}}{\sqrt{\log(1+\epsilon)}} \}
\end{eqnarray*} 
\end{thm}
	We optimize the above bund in the following corollary:
\begin{cor}\label{cor:norm_cor} For all $d \ge 250$, 
	\begin{eqnarray}
	\sqrt{d}\Exp[\|\matW\|] \le 2\sqrt{d} + 21d^{1/6}\log^{2/3}(d)
	\end{eqnarray}
\end{cor}

\begin{proof}
	Using the estimate $\log(1+\epsilon) \ge \epsilon/2$  for $\epsilon (0,1/2)$, the estimate $\sqrt{d+1} = \sqrt{d}(\sqrt{1 + 1/d}) \le \sqrt{d}(1 + \frac{1}{2\sqrt{d}})$,  and plugging in $d\sigma^2 = d+1$ and $d\sigma_*^2 = 2$, one has that
	\begin{eqnarray*}
	\sqrt{d}\Exp[\|\matW\|] &\le& \inf_{\epsilon \in (0,1/2)}(1+\epsilon)(2\sqrt{d+1} +  12\log d\epsilon^{-1/2} \} \\
	&\le& 2\sqrt{d} + \frac{3}{2\sqrt{d}} + \inf_{\epsilon \in (0,1/2)}2\epsilon \sqrt{d} + 18 \log d\epsilon^{-1/2} \\
	\end{eqnarray*}
	Setting $\epsilon = d^{-1/3}\log d^{2/3}$, we have that for $d \ge 4$
	\begin{eqnarray}
	\sqrt{d}\Exp[\|\matW\|] \le 2\sqrt{d} + \frac{3}{2\sqrt{d}}  + 20d^{1/6}\log^{2/3} d \le 2\sqrt{d} + 21d^{1/6}\log^{2/3} d
	\end{eqnarray}
	Note that $\epsilon < 1/2$ as long as $\frac{d}{\log^2 d} \ge 8$, which holds as long as $d \ge 250$. 

\end{proof}
\begin{proof}[Proof of Proposition~\ref{prop:norm_upper_bound}]
Combining Lemmas~\ref{lem:norm_conc} and then Corollary~\ref{cor:norm_cor}, one has that with probability at least $1-\delta$,
\begin{eqnarray*}
\|\matW\| &\le& \Exp[\|\matW\|] + \frac{2\log^{1/2}\delta}{d^{1/2}} \le  d^{-1/2}(2\sqrt{d} + 21d^{1/6}\log^{2/3}(d)) + \frac{2\log^{1/2}\delta}{d^{1/2}} \\
&\le& 2 + 21^{-1/3}\log^{2/3}(d) + 2\sqrt{d\log (1/\delta)} 
\end{eqnarray*} 
\end{proof}

%Proof of Hanson-Wright style bound

%!TEX root = main.tex

\section{Proof of Proposition~\ref{prop:Hanson_wright_proposition}\label{sec:hanson_wright_proof}}

\begin{lem}[Gaussian Hanson-Wright]\label{Gaussian_Hanson_Wright} Let $\matZ  \sim \mathcal{N}(0,I_d)$ be an isotropic Gaussian vector. Then
\begin{eqnarray*}
\Pr\left[\left|\matZ^{\top}A\matZ - \mathrm{tr}(A)\right| > 2 \left(t^{1/2}\|A\|_F + t\|A\|_{2}\right) \right] \le 2e^{-t}
\end{eqnarray*}
\end{lem}
\begin{proof}
Since $Z$ has a rotation-invariant distribution, we may assume without loss of generality that $A = \mathrm{diag}(a)$ is a diagonal matrix where $a \in \R^d$. It then suffices to prove the following inequality, where $Y_i \overset{i.i.d}{\sim} \calN(0,1)$: 
\begin{eqnarray*}
\Pr[|\sum_{i=1}^n a_i(Y_i^2 - 1)| > 2 (t^{1/2}\|a\|_2 + t\|a\|_{\infty}) ] \le 2e^{-t}
\end{eqnarray*}
This statement is proved in Lemma 1 in Laurent and Massart '00 \cite{laurent2000adaptive} in the case $a_i \ge 0$. However, their proof goes through as-is for arbitrary $a_i$, with the modification that a sharper one-sided concentration bound they derive no longer holds in the more general case.
\end{proof}
\begin{cor}[Corollary on the sphere]\label{cor:sphere_hanson} Let $\matu \sim \calS^{d-1}$. Then
\begin{eqnarray*}
\Pr\left[|\matu^{\top}A\matu- \mathrm{tr}(A)/d| > \frac{4}{d(1 - 2\sqrt{t/d})} (t^{1/2}\|A\|_F + t\|A\|_{op})\right]  \le 3e^{-t}
\end{eqnarray*} 
\end{cor}
\begin{proof}
	Define the matrix $\widetilde{A} = A - \frac{1}{d}\mathrm{tr}(A)I$. Observe then that $\|\widetilde{A}\| \le \|A\|_2 + \|\frac{1}{d}\mathrm{tr}(A)I\|_2 \le 2\|A\|_2$ and that $\|\widetilde{A}\|_F \le \|A\|_F + \|\frac{1}{d}\mathrm{tr}(A)I\|_F  = \|A\|_F + |\mathrm{tr}(A)|/\sqrt{d} \le 2\|A\|_F$. Moreover, one has that if $U$ is distributed uniformly on the sphere, and $\mathbf{Z} = (Z_1,\dots,Z_d)$ is a standard normal vector then 
	\begin{eqnarray*}
	\matu^{\top}A\matu - \mathrm{tr}(A)/d = \matu^{\top}\widetilde{A}\matu = \frac{\mathbf{Z}^{\top}\widetilde{A}\mathbf{Z}}{\|Z\|_2^2}
	\end{eqnarray*}
	By Lemma~\ref{Gaussian_Hanson_Wright} and our estimates of $\widetilde{A}$, we have
	\begin{eqnarray*}
	\Pr[|\mathbf{Z}^{\top}\widetilde{A}\mathbf{Z}| > 4 (t^{1/2}\|A\|_F + t\|A\|_{\op}) ] \le 2e^{-t}
	\end{eqnarray*}
	Moreover, using the one-sided analogue of Lemma~\ref{Gaussian_Hanson_Wright} with $A = I_d$, one has $\Pr[\|\mathbf{Z}\|^2 \le d(1 - 2\sqrt{t/d})] \le e^{-t}$. Combining both estimates, we have
	\begin{eqnarray*}
	\Pr[|\matu^{\top}A\matu - \mathrm{tr}(A)/d| > \frac{4}{d(1 - 2\sqrt{t/d})} (t^{1/2}\|A\|_F + t\|A\|_{\op})  \le 3e^{-t}
	\end{eqnarray*}
\end{proof}

Our last step to conclude the proof is the following packing argument:
\begin{comment}
\begin{lem}\label{lem:matrix_packing} Let $A \in \R^{k\times k}$ be a symmetric symetric matrix, and let $\calN$ be a $1/4$-net of the unit ball in $\R^k$. Then, $\|A\|_2 \le 2\sup_{w \in \calN} |w^\top A w|$.  
\end{lem}
\begin{proof}
WLet $\calB_k$ denote the unit ball in $\R^k$. We will prove the claim that $|v^\top vA v - w^\top Aw | \le 2 \|A\|_2 \|v-w\|$. This implies that  if $v_* = \arg\max_{v \in \calB_k} |v^{\top} Av|$ and $w \in \calS^{d-1}$
\begin{eqnarray*}
\|A\|_2 = |v_*^\top A v_*| \le |w^{\top}Aw| + 2 \|A\|_2 \|v_*-w\|
\end{eqnarray*}
Hence, if $\calN$ is a $1/4$-net of $\R^d$, then $\|A\|_2 \le 2\sup_{w \in \calN} |w^\top A w|$. 
To prove our claim, let $\|\cdot\|_*$ denote the trace norm. Observe that that by the matrix Holder inequality, $v^\top A v - w^\top A w = \mathrm{tr}(A(vv^{\top}-ww^{\top})) \le \|A\|_\op \|vv^{\top}-ww^{\top}\|_*$. By the triangle inequality $\|vv^{\top}-ww^{\top}\|_* \le \|v(v-w)^{\top}\|_* + \|(w-v)w^\top\|_*$. Since the later two terms are rank one, the nuclear norms are equal to the spectral norms, which are precisely $\|w-v\|\|v\|$ and $\|w-v\|\|w\|$, respectively. As $\max\{\|w\|,\|v\|\} \le 1$, we have $\mathrm{tr}(A(vv^{\top}-ww^{\top})) \le 2\|A\|_\op\|w-v\|$, as needed. 
\end{proof}
\end{comment}
We are now in place to prove  Proposition~\ref{prop:Hanson_wright_proposition}:
\begin{proof}[Proof of Proposition~\ref{prop:Hanson_wright_proposition}]
For ease of notation, set $C(A,d,t) := \frac{4}{d(1 - 2\sqrt{t/d})} (t^{1/2}\|A\|_F + t\|A\|_{op})$ to be the error term in Lemma~\ref{cor:sphere_hanson}.  Observe that if $\matU \sim \Stief(d,r)$ for any fixed $v \in \calS^{r-1}$, $\matU v \sim \calS^{d-1}$,and hence by Corollary~\ref{cor:sphere_hanson},
\begin{eqnarray}\label{eq:pointwise_hanson}
\Pr\left[|v^{\top}(\matU^{\top}A\matU - I_r \cdot \mathrm{tr}(A)/d )v| > C(A,d,t)\right] \le 3e^{-t}
\end{eqnarray} 
Now let $\calN$ is an $1/4$-net of $\calS^{r-1}$. By Exercise 4.4.3 in~\cite{vershynin2016high}, any symmetric matrix $B \in \R^{r \times r}$ satisfies the inequality  $\|B\|_{\op} \le 2\sup_{v \in \calN} | v^\top B v|$. Hence, by Equation~\eqref{eq:pointwise_hanson} and a union bound,
\begin{eqnarray*}
&& \Pr\left[\left\|\matU^{\top}A\matU - I_r \cdot \mathrm{tr}(A)/d \right\| > 2C(A,d,t)\right]  \\
&\le& \Pr\left[\sup_{v \in \calN}\left|v^{\top}\left(\matU^{\top}A\matU - I_r \cdot \mathrm{tr}(A)/d\right)v\right| > C(A,d,t)\right]  \le 3|\calN|e^{-t}
\end{eqnarray*}
A standard covering number bound (e.g. Corollary 4.2.13 in~\cite{vershynin2016high}) lets us choose $|\calN| \le (1 + \frac{2}{1/4})^r= 9^r \le \exp(2.2 r)$, which concludes the proof.
\end{proof}

\end{document}